%% file: arxiv_out.tex
\documentclass{article}

\usepackage{lineno,hyperref}
\modulolinenumbers[5]

\newcommand{\aspmc}{\ensuremath{\mathsf{asp2sat}}\xspace}

\newcommand{\FIX}[1]{#1} %
\newcommand{\FIXR}[1]{#1} %
\newcommand{\longversion}[1]{}
\newcommand{\shortversion}[1]{#1}
\newcommand{\futuresketch}[1]{}
\usepackage{soul}
\usepackage{hyperref}
\usepackage[utf8]{inputenc}
\usepackage[small]{caption}
\usepackage{url}  %
\usepackage{multirow}
\usepackage{amssymb,amsfonts,amsmath,amsthm}%
\newcommand{\footnoteitext}[1]{\stepcounter{footnote}
  \footnotetext[\thefootnote]{#1}}
\usepackage{enumitem}
\usepackage{stackengine}
\usepackage{calc}
\newlength\shlength
\newcommand\xshlongvec[2][0]{\setlength\shlength{#1pt}%
  \stackengine{-5.6pt}{$#2$}{\smash{$\kern\shlength%
    \stackengine{7.55pt}{$\mathchar"017E$}%
      {\rule{\widthof{$#2$}}{.57pt}\kern.4pt}{O}{r}{F}{F}{L}\kern-\shlength$}}%
      {O}{c}{F}{T}{S}}

\usepackage{xspace}
\usepackage[ruled,vlined,linesnumbered]{algorithm2e}
\usepackage[table,x11names]{xcolor}
\usepackage{graphicx}  %
\usepackage{caption}

\SetKwInput{KwData}{In}
\SetKwInput{KwResult}{Out}
\SetAlFnt{\small}
\SetAlCapFnt{\small}
\SetAlCapNameFnt{\small}
\SetAlCapHSkip{0pt}
\SetEndCharOfAlgoLine{}

\usepackage{array}
\newcolumntype{H}{>{\setbox0=\hbox\bgroup}c<{\egroup}@{}}

\newcommand{\problemFont}[1]{\textsc{#1}}

\newcommand{\mtext}[1]{\ensuremath{\mathcal{#1}}}
\newcommand{\cntc}[0]{\ensuremath{\#\cdot}}

\DeclareMathOperator{\poly}{poly}
\DeclareMathOperator{\rank}{ord}

\DeclareMathOperator{\SP}{ProofStates}
\DeclareMathOperator{\checkord}{isLvl}
\DeclareMathOperator{\checkmod}{CheckMod}
\DeclareMathOperator{\gatherproof}{proven}

\DeclareMathOperator{\possord}{lvl}
\newcommand{\MAI}[2]{\ensuremath{#1^+_{#2}}}%
\usepackage{tikz}
\usetikzlibrary{shapes}
\usetikzlibrary{calc}

\newcommand{\citex}[1]{\citeauthor{#1}~\shortcite{#1}}
\newcommand{\citey}[1]{\citeauthor{#1},~\citeyear{#1}}

\renewcommand{\P}{\ensuremath{\textsc{P}}\xspace}
\newcommand{\NP}{\ensuremath{\textsc{NP}}\xspace}
\newcommand{\SIGMA}[2]{\ensuremath{\Sigma_{\textrm{#1}}^{\textrm{#2}}}}

\newcommand{\HCF}{\text{HCF}\xspace}

\usetikzlibrary{fit,positioning,arrows,shapes,snakes,automata,%
                petri,%
                topaths,calc,patterns} %

\tikzstyle{tdnode} = [draw,rounded corners,top color=vertexTopColor,bottom color=vertexBottomColor,minimum size=1.5em]
\tikzstyle{stdnode} = [tdnode, font=\scriptsize]
\tikzstyle{stdnodecompact} = [stdnode, inner sep = 1.5pt, outer sep = 0.1pt]
\tikzstyle{stdnodetable} = [stdnode, inner sep = 0.5pt, outer sep = 0]
\tikzstyle{stdnodenum} = [minimum size=1.5em, font=\scriptsize]
\tikzstyle{tdedge} = [-,draw,thick]
\tikzstyle{tdlabel} = [draw=none, rectangle, fill=none, inner sep=0pt, font=\scriptsize]
\colorlet{vertexTopColor}{white}
\colorlet{vertexBottomColor}{black!10}
\tikzstyle{squigarrow} = [->,line join=round,decorate, decoration={
    zigzag,segment length=4,amplitude=.9,post=lineto,post length=2pt}]

\tikzstyle{dashedarrow} = [->,dashed]

\usepackage{ctable}
\usepackage{ifthen}
\newif\iflong
\newcommand{\restrict}[2]{\ensuremath{#1\cap #2}}

\usepackage{float}
\newcommand{\SB}{\{}%
\newcommand{\SM}{\mid}%
\newcommand{\SE}{\}}%
\def\hy{\hbox{-}\nobreak\hskip0pt}

\newcommand{\mdpa}[1]{\ensuremath{\mathtt{PCNT}_{#1}}}
\newcommand{\ta}[1]{\ensuremath{2^{#1}}}
\newcommand{\Card}[1]{\left|#1\right|}
\newcommand{\CCard}[1]{\|#1\|}
\newcommand{\algo}[1]{\ensuremath{\mathbb{#1}}}
\DeclareMathOperator{\width}{width}
\DeclareMathOperator{\children}{chldr}
\DeclareMathOperator{\rootOf}{root}

\usepackage{mathtools}
\DeclarePairedDelimiter\ceil{\lceil}{\rceil}

\newcommand{\bvali}[3]{\ensuremath{[\![#1]\!]_{#2,#3}}}

\makeatletter
\newcommand{\algorithmfootnote}[2][\footnotesize]{
  \let\old@algocf@finish\@algocf@finish
  \def\@algocf@finish{\old@algocf@finish
    \leavevmode\rlap{\begin{minipage}{\linewidth}
    #1#2
    \end{minipage}}
  }
}
\makeatother

\DeclareMathOperator{\orig}{\algo{A}\hy origins}
\DeclareMathOperator{\origs}{\algo{A}\hy origins}
\newcommand{\origa}[1]{\operatorname{#1\hy origins}}
\newcommand{\origse}[1]{\operatorname{#1\hy origins}}
\DeclareMathOperator{\Ext}{Ext}
\DeclareMathOperator{\Exts}{Exts}
\DeclareMathOperator{\PExt}{SatExt}
\DeclareMathOperator{\pmc}{pasc}
\DeclareMathOperator{\ipmc}{ipasc}
\DeclareMathOperator{\bucket}{=_P}%
\DeclareMathOperator{\buckets}{buckets}
\DeclareMathOperator{\subbuckets}{sub\hy buckets}

\newcommand{\TTT}{\ensuremath{\mathcal{T}}}%
\newcommand{\WWW}{\ensuremath{\mathcal{W}}}%
\newcommand{\por}{\vee}

\newcommand{\eqdef}{\ensuremath{\,\mathrel{\mathop:}=}}
\newcommand{\hsep}{\leftarrow\,}
\newcommand{\SAT}{\textsc{SAT}\xspace}
\newcommand{\ASP}{\textsc{ASP}\xspace}

\newcommand{\PASP}{\textsc{\#PAs}\xspace}%
\newcommand{\PDASP}{\textsc{\#PDAs}\xspace}%
\newcommand{\AlgA}{\algo{A}}%
\newcommand{\AlgS}{\AlgA}%
\newcommand{\PROJ}{\algo{PROJ}\xspace}
\newcommand{\at}{\text{\normalfont at}}
\newcommand{\var}{\at}
\newcommand{\bigO}[1]{\ensuremath{{\mathcal O}(#1)}}
\newcommand{\CCC}{\ensuremath{\mathcal{C}}}%

\newcommand{\tuplecolor}[1]{\textcolor{#1}}
\newcommand{\inputPredColor}{orange!55!red}
\newcommand{\outputPredColor}{blue!45!black}
\newcommand{\statePredColor}{green!62!black}
\newcommand{\specialPredColor}{red!62!black}

\newcommand{\tabval}{\ensuremath{u}}
\newcommand{\tab}[1]{\ensuremath{\tau_{#1}}}

\newcommand{\att}[1]{\ensuremath{\at_{\hspace{-0.05em}\leq\hspace{-0.05em}#1}}}
\newcommand{\attneq}[1]{\ensuremath{\at_{\hspace{-0.05em}<\hspace{-0.05em}#1}}}

\newcommand{\prog}{\ensuremath{\Pi}}
\newcommand{\progt}[1]{\ensuremath{\prog_{\hspace{-0.05em}\leq\hspace{-0.05em}#1}}}
\newcommand{\progtneq}[1]{\ensuremath{\prog_{\hspace{-0.05em}<\hspace{-0.05em}#1}}}

\newcommand{\dpa}{\ensuremath{\mathtt{DP}}}
\newcommand{\Tab}[1]{\ensuremath{\text{C-Tabs}}}

\def\thyph{\text{-}\penalty0\hskip0pt\relax}

\newcommand{\ATab}[1]{\ensuremath{#1\thyph\text{Comp}}}

\newcommand{\tw}[1]{\mathit{tw}(#1)}

\newcommand{\Nat}{\mathbb{N}} %
\DeclareMathOperator{\type}{type}
\newcommand{\intr}{\textit{int}}
\newcommand{\leaf}{\textit{leaf}}
\newcommand{\rem}{\textit{forget}}
\newcommand{\join}{\textit{join}}

\newtheorem{example}{Example}
\newtheorem{conjecture}{Conjecture}
\newtheorem{proposition}{Proposition}
\newtheorem{hypothesis}{Hypothesis}

\newtheorem{observation}{Observation}
\newtheorem{theorem}{Theorem}
\newtheorem{lemma}{Lemma}
\newtheorem{definition}{Definition}
\newtheorem{corollary}{Corollary}

{%
  \addtocounter{observation}{-1}
  \endgroup
}%

{%
  \addtocounter{corollary}{-1}
  \endgroup
}%

\newenvironment{restateproposition}[1][\unskip]{%
  \begingroup
  \def\theproposition{\ref{#1}} 
}%
{%
  \addtocounter{proposition}{-1}
  \endgroup
}%

\newenvironment{restatetheorem}[1][\unskip]{%
  \begingroup
  \def\thetheorem{\ref{#1}} 
}%
{%
  \addtocounter{theorem}{-1}
  \endgroup
}%

\DeclareMathOperator{\pcnt}{pasc}
\DeclareMathOperator{\local}{local}
\DeclareMathOperator{\sipmc}{s-ipasc}
\DeclareMathOperator{\icnt}{ipasc}

\newcommand{\PRIM}{\ensuremath{{\algo{PHC}}}\xspace}
\newcommand{\INC}{\ensuremath{{\algo{INC}}}\xspace}

\begin{document}

\title{Treewidth-aware Reductions of Normal \ASP to \SAT -- Is Normal \ASP Harder than \SAT after All?}
\author{Markus Hecher\\
TU Wien, Vienna, Austria\\
\href{hecher@dbai.tuwien.ac.at}{hecher@dbai.tuwien.ac.at}}

\maketitle
\begin{abstract}
Answer Set Programming (ASP) is a paradigm for modeling and solving problems for knowledge representation and reasoning.
There are plenty of results dedicated to studying
the hardness of (fragments of) ASP.
So far, these studies resulted in characterizations in terms of computational complexity
as well as in fine-grained insights 
presented in form of dichotomy-style results, 
lower bounds when translating to other formalisms like propositional satisfiability (SAT),
and even detailed parameterized complexity landscapes.
A generic parameter in parameterized complexity originating from graph theory is 
the so-called \emph{treewidth}, which in a sense captures structural density of a program.
Recently, there was an increase in the number of treewidth-based solvers
related to SAT.
While there are translations from (normal) ASP to SAT, no reduction that preserves treewidth or at least keeps track of the treewidth increase is known.
In this paper we propose a novel reduction from normal ASP to SAT that is aware of the treewidth,
and guarantees that a slight increase of treewidth is indeed sufficient.
Further, we show a new result establishing that, when considering treewidth,
already the fragment of normal ASP is slightly harder than SAT (under reasonable assumptions in computational complexity).
This also confirms that our reduction probably cannot be significantly improved and that the slight increase of treewidth
is unavoidable.
\FIXR{Finally, we present an empirical study of our novel reduction from normal ASP to SAT, where we
	compare treewidth upper bounds that are obtained via known decomposition heuristics.
	Overall, our reduction works better with these heuristics than existing translations. }
\end{abstract}

\section{Introduction}

Answer Set Programming (ASP)~\cite{BrewkaEiterTruszczynski11} is an
active research area of knowledge representation and reasoning. %
ASP provides a declarative modeling language and problem solving
framework~\cite{GebserKaminskiKaufmannSchaub12} for hard computational
problems, which has been widely
applied~\cite{BalducciniGelfondNogueira06a,NiemelaSimonsSoininen99,NogueiraBalducciniGelfond01a,GuziolowskiEtAl13a,SchaubWoltran18,AbelsEtAl19}.
There are very efficient ASP solvers~\cite{GebserKaminskiKaufmannSchaub12,AlvianoEtAl19,CabalarEtAl20} as well as several recent (language) extensions~\cite{CabalarEtAl20,CabalarEtAl20c,CalimeriEtAl20,CabalarEtAl20b}.
In ASP, questions are encoded into rules and constraints that form a
program (over atoms), whose solutions %
are called answer~sets.

In terms of computational complexity, the \emph{consistency problem} of deciding the existence of
an answer set is well-studied, i.e., the problem is~$\Sigma_2^P$-complete~\cite{EiterGottlob95}.
Some fragments of ASP have lower complexity though. 
A prominent example
is the class of \emph{head-cycle-free (HCF)}
programs~\cite{Ben-EliyahuDechter94}, which is a generalization of 
the class of \emph{normal} programs
and requires the absence of
cycles in a dependency graph representation of the program. 
Deciding whether such a program has an answer set is \NP-complete.

There is also a wide range of more fine-grained studies~\cite{Truszczynski11} for ASP, also in parameterized
complexity~\cite{CyganEtAl15,Niedermeier06,DowneyFellows13,FlumGrohe06}, 
where certain (combinations of) parameters~\cite{FichteKroneggerWoltran19,LacknerPfandler12} %
are taken into account.
In parameterized complexity, the ``hardness'' of a problem 
is classified according to the impact 
of a \emph{parameter} for solving the problem.
There, one often distinguishes the runtime dependency of the parameter, e.g., levels of 
exponentiality~\cite{LokshtanovMarxSaurabh11,MarxMitsou16} in the parameter,
required for problem solving.
Concretely, under the reasonable \emph{Exponential Time Hypothesis (ETH)}~\cite{ImpagliazzoPaturiZane01}, 
\emph{propositional satisfiability (\SAT)} is single exponential in the structural parameter treewidth, whereas
evaluating \emph{Quantified Boolean formulas (QBFs)}
of quantifier depth two is~\cite{LampisMitsou17} 
double exponential\footnote{Double exponentiality refers to runtimes of the form~$2^{2^{\mathcal{O}(k)}} \cdot n$.} %
in the treewidth~$k$.

For ASP there is growing research on 
treewidth~\cite{JaklPichlerWoltran09,FichteEtAl17a,FichteHecher19},
which even involves grounding~\cite{BichlerMorakWoltran18,BliemEtAl20}.
Algorithms of these works exploit structural restrictions (in form of treewidth) of a given program,
and often run in polynomial time in the program size,
while being exponential only in the treewidth.
Intuitively, treewidth gives rise to a \emph{tree decomposition},
which allows %
solving numerous NP-hard problems in parts \FIXR{(via dynamic programming)} %
and indicates the maximum number of variables one has to investigate 
in such parts during evaluation.
There were also dedicated competitions~\cite{Dell17a} 
and notable progresses %
in \SAT~\cite{FichteHecherZisser19,CharwatWoltran19}
and other areas~\cite{BannachBerndt19}. %

Naturally, there are numerous reductions of ASP~\cite{Clark77,Ben-EliyahuDechter94,LinZhao03,Janhunen06,AlvianoDodaro16} and extensions thereof~\cite{BomansonJanhunen13,Bomanson17} to \SAT.
These reductions have been investigated
in the context of resulting formula size and number of auxiliary variables.
However, structural dependency in form of, e.g., treewidth, has not been considered yet.
Notably, there are existing reductions causing a sub-quadratic blow-up in the number of variables (auxiliary variables), which is unavoidable~\cite{LifschitzRazborov06} if the answer sets should be preserved (bijectively).
However, if one considers the structural dependency in form of treewidth, existing reductions
could cause quadratic or even unbounded overhead in the treewidth.

On the contrary,
we present a novel reduction %
for HCF programs that increases the treewidth~$k$ at most \emph{sub-quadratically} ($k\cdot\log(k)$).
This is indeed interesting as there is a close connection~\cite{AtseriasFichteThurley11}
between resolution-width %
and treewidth,
resulting in efficient \SAT solver runs on instances of small treewidth.
As a result, our reduction could be of use
for %
solving approaches based on \SAT solvers~\cite{LinZhao04,Janhunen06}.
Then, we establish  
lower bounds, under ETH, for exploiting treewidth for consistency of normal programs.
This renders normal \ASP ``harder'' than \SAT.
At the same time we prove that one cannot significantly improve the reduction, i.e., avoid the sub-quadratic increase of treewidth.

\smallskip\noindent\textbf{Contributions.} Concretely, we provide the following.

\begin{enumerate}%
	\item First, we present a novel reduction from HCF programs to \SAT, which only requires linearly many auxiliary variables
	plus a number of auxiliary variables that is linear in the instance size and slightly superexponential in the treewidth of the \SAT instance.
	This is achieved by guiding the whole reduction along a tree decomposition of the program.
	Thereby the reduction only increases the treewidth sub-quadratically, i.e., the treewidth of the resulting \SAT formula is slightly larger
	than the treewidth of the given program.
	\item \FIX{Then, we develop and discuss a slightly different reduction from HCF programs to \SAT, where we show a bijective correspondence between answer sets of the program and models of the propositional formula for a certain sub-class of programs.
    This reduction, while preserving bijectivity for certain programs, comes at the cost of a quadratic increase of treewidth.%
}
	\item We show that %
	avoiding a sub-quadratic increase in the treewidth is very unlikely.
	Concretely, we establish that under the widely believed \emph{Exponential Time Hypothesis (ETH)}, one cannot decide ASP in time $2^{o(k \cdot \log(k))} \cdot n$,
	with treewidth~$k$ and program size~$n$. %
	This is in contrast to the runtime for deciding \SAT:~$2^{\mathcal{O}(k)}\cdot n$ with treewidth~$k$ and size~$n$ of the formula. %
	As a result, this establishes that the consistency of normal \ASP programs is already harder than  \SAT
	using treewidth. 
	Note that this is surprising as  %
	both problems are of similar hardness according to classical complexity (\NP-complete).
	\futuresketch{Further, compared to known results that do not rely on assumptions like the ETH, but restrict to, e.g., modular reductions~\cite{Janhunen06}, or involve the need of auxiliary variables~\cite{LifschitzRazborov06}, this shows that the increase of treewidth is likely unavoidable already for the consistency problem.} %
	\item \FIXR{Finally, we present an empirical study of the first contribution, where we
	compare treewidth upper bounds that are obtained via existing decomposition heuristics.
	Interestingly, compared with existing translations, in both acyclic and cyclic scenarios,
	our reduction overall works better
	with these heuristics than existing translations.  
	}
\end{enumerate}

\noindent\FIXR{This is an extended version of a paper~\cite{Hecher20} %
at the 17th International Conference on Principles of Knowledge Representation and Reasoning (KR 2020). %
In addition to the conference version, this work contains additional examples and more detailed explanations. Further, we added the second contribution and worked out details in terms of bijective preservation of answer sets. We were also able to simplify the reduction of the third contribution and discuss relations to other works in detail.
Further, we added an empirical study of treewidth, where we compare treewidth upper bounds of our reduction and an exisiting translation via an efficient decomposer based on heuristics.
}

\smallskip
\noindent\textbf{Related Work.}
For disjunctive ASP and extensions thereof, 
algorithms have been proposed~\cite{JaklPichlerWoltran09,PichlerEtAl14,FichteEtAl17a} %
running in time linear in the instance size, but double exponential
in the treewidth. %
Under ETH, one cannot significantly improve this runtime,
using a result~\cite{LampisMitsou17} for QBFs with quantifier depth two 
and a standard reduction~\cite{EiterGottlob95} from this QBF fragment to disjunctive ASP.
Unsurprisingly, \SAT only requires single exponential runtime~\cite{SamerSzeider10b} 
in the treewidth.
However, for normal and HCF programs
only a slightly superexponential algorithm~\cite{FichteHecher19} for solving consistency is known so far.
Still, the question whether the slightly superexponentiality can
be avoided %
was left open. %
The proposed algorithm %
was used for %
counting answer sets involving projection~\cite{GebserKaufmannSchaub09a}, %
which is at least double exponential~\cite{FichteEtAl18} in the treewidth.%
\futuresketch{However, for plain counting (single exponential), it can overcount due to lacking unique level mappings (orderings). %
}

\noindent\FIXR{There are also further studies on certain classes of programs and their relationships in the form of whether there exist certain reductions between these classes, thereby bijectively preserving the answer sets. These studies
result in an expressive power hierarchy among program classes~\cite{Janhunen06}.}
\FIXR{Note that while existing results~\cite{LifschitzRazborov06} and the expressive power hierarchy weakly indicate that normal \ASP might be slightly harder than \SAT, these results mostly deal with bijectively preserving all answer sets. 
However, our work considers the plain consistency problem,
where \ASP and \SAT are \emph{both \NP-complete}.
}

\section{Preliminaries}
\futuresketch{
\paragraph{Basics and Combinatorics.}
For a set~$X$, let $\ta{X}$ be the \emph{power set of~$X$}
consisting of all subsets~$Y$ with $\emptyset \subseteq Y \subseteq X$.
Let $\vec s$ be a sequence of elements of~$X$. When we address the
$i$-th element of the sequence~$\vec s$ for a given positive
integer~$i$, we simply write $\vec s_{(i)}$. The sequence~$\vec s$
\emph{induces} an \emph{ordering~$<_{\vec s}$} on the elements in~$X$
by defining the
relation~$<_{\vec s} \eqdef \SB (\vec s_{(i)},\vec s_{(j)}) \SM 1 \leq
i < j \leq \Card{\vec s}\SE$.
Given some integer~$n$ and a family of finite subsets~$X_1$, $X_2$,
$\ldots$, $X_n$. Then, the generalized combinatorial
inclusion-exclusion principle~\cite{GrahamGrotschelLovasz95a} states
that the number of elements in the union over all subsets is
$\Card{\bigcup^n_{j = 1} X_j} = \sum_{I \subseteq \{1, \ldots, n\}, I
  \neq \emptyset} (-1)^{\Card{I}-1} \Card{\bigcap_{i \in I} X_i}$.

\paragraph{Computational Complexity.}
We assume familiarity with standard notions in computational
complexity~\cite{Papadimitriou94}
and use counting complexity classes as defined
by~\citex{DurandHermannKolaitis05}.
For parameterized complexity, we refer to standard
texts~\cite{CyganEtAl15}. 
We recall some basic notions.
Let $\Sigma$ and $\Sigma'$ be some finite alphabets.  We call
$I \in \Sigma^*$ an \emph{instance} and $\CCard{I}$ denotes the size
of~$I$.  %
Let $L \subseteq \Sigma^* \times \Nat$ and
$L' \subseteq {\Sigma'}^*\times \Nat$ be two parameterized problems. An
\emph{fpt-reduction} $r$ from $L$ to $L'$ is a many-to-one reduction
from $\Sigma^*\times \Nat$ to ${\Sigma'}^*\times \Nat$ such that for all
$I \in \Sigma^*$ we have $(I,k) \in L$ if and only if
$r(I,k)=(I',k')\in L'$ such that $k' \leq g(k)$ for a fixed computable
function $g: \Nat \rightarrow \Nat$, and there is a computable function
$f$ and a constant $c$ such that $r$ is computable in time
$O(f(k)\CCard{I}^c)$. If additionally~$g$ is a linear function,
then~$r$ is referred to as~\emph{fpl-reduction}.
A \emph{witness function} is a
function~$\mathcal{W}\colon \Sigma^* \rightarrow 2^{{\Sigma'}^*}$ that
maps an instance~$I \in \Sigma^*$ to a finite subset
of~${\Sigma'}^*$. We call the set~$\WWW(I)$ the \emph{witnesses}. A
\emph{parameterized counting
  problem}~$L: \Sigma^* \times \Nat \rightarrow \Nat_0$ is a
function that maps a given instance~$I \in \Sigma^*$ and an
integer~$k \in \Nat$ to the cardinality of its
witnesses~$\Card{\WWW(I)}$.
Let $\mtext{C}$ be a decision complexity class,~e.g., \P. Then,
$\cntc\mtext{C}$ denotes the class of all counting problems whose
witness function~$\WWW$ satisfies (i)~there is a
function~$f: \Nat_0 \rightarrow \Nat_0$ such that for every
instance~$I \in \Sigma^*$ and every $W \in \WWW(I)$ we have
$\Card{W} \leq f(\CCard{I})$ and $f$ is computable in
time~$\bigO{\CCard{I}^c}$ for some constant~$c$ and (ii)~for every
instance~$I \in \Sigma^*$ the decision problem~$\WWW(I)$ belongs to
the complexity class~$\mtext{C}$.
Then, $\cntc\P$ is the complexity class consisting of all counting
problems associated with decision problems in \NP.
Let $L$ and $L'$ be counting problems with witness functions~$\WWW$
and $\WWW'$. A \emph{parsimonious reduction} from~$L$ to $L'$ is a
polynomial-time reduction~$r: \Sigma^* \rightarrow \Sigma'^*$ such
that for all~$I \in \Sigma^*$, we
have~$\Card{\WWW(I)}=\Card{\WWW'(r(I))}$. It is easy to see that the
counting complexity classes~$\cntc\mtext{C}$ defined above are closed
under parsimonious reductions. It is clear for counting problems~$L$
and $L'$ that if $L \in \cntc\mtext{C}$ and there is a parsimonious
reduction from~$L'$ to $L$, then $L' \in \cntc\mtext{C}$.

}

\noindent Before we discuss our reductions, 
we briefly recall some basics.

\smallskip
\noindent\textbf{Answer Set Programming (ASP).}
We assume familiarity with propositional satisfiability (\SAT)~\cite{BiereHeuleMaarenWalsh09,KleineBuningLettman99}, 
and follow standard definitions of propositional ASP~\cite{BrewkaEiterTruszczynski11,JanhunenNiemela16a}.
Let $\ell$, $m$, $n$ be non-negative integers such that
$\ell \leq m \leq n$, $a_1$, $\ldots$, $a_n$ be distinct propositional
atoms. Moreover, we refer by \emph{literal} to a \emph{propositional variable (atom)} or the negation
thereof.
A \emph{program}~$\prog$ is a set of \emph{rules} of the form
\begin{align*}
a_1\por \cdots \por a_\ell \hsep a_{\ell+1}, \ldots, a_{m}, \neg
a_{m+1}, \ldots, \neg a_n.
\end{align*}
For a rule~$r$, we let $H_r \eqdef \{a_1, \ldots, a_\ell\}$,
$B^+_r \eqdef \{a_{\ell+1}, \ldots, a_{m}\}$, and
$B^-_r \eqdef \{a_{m+1}, \ldots, a_n\}$.
We denote the sets of \emph{atoms} occurring in a rule~$r$ or in a
program~$\prog$ by $\at(r) \eqdef H_r \cup B^+_r \cup B^-_r$ and
$\at(\prog)\eqdef \bigcup_{r\in\prog} \at(r)$.
\FIX{A rule~$r$ is \emph{normal} if $\Card{H_r} \leq 1$ 
 and \emph{unary} if~$\Card{B_r^+}\leq 1$.
 Then, a program~$\prog$ is \emph{normal} or \emph{unary} if all its rules~$r\in\prog$ are normal or unary, respectively.}
The \emph{positive dependency digraph}~$D_\prog$ of $\prog$ is the
directed graph defined on the set of atoms
from~$\bigcup_{r\in \prog}H_r \cup B^+_r$, where for every
rule~$r \in \prog$ two atoms $a\in B^+_r$ and~$b\in H_r$ are joined by
an edge~$(a,b)$.
A \emph{head-cycle} of~$D_\prog$ is an $\{a, b\}$-cycle\footnote{Let
  $G=(V,E)$ be a digraph and $W \subseteq V$. Then, a (directed) cycle in~$G$ is
  a $W$-cycle if it contains all vertices from~$W$.} for two distinct
atoms~$a$, $b \in H_r$ for some rule $r \in \prog$. 
Program~$\prog$ is
\emph{head-cycle-free (HCF)} if $D_\prog$ contains no
head-cycle~\cite{Ben-EliyahuDechter94}
\FIX{and we say~$\prog$ is \emph{tight} if there is no directed cycle in~$D_\prog$~\cite{Fages94}.}

An \emph{interpretation} $I$ is a set of atoms. $I$ \emph{satisfies} a
rule~$r$ if $(H_r\,\cup\, B^-_r) \,\cap\, I \neq \emptyset$ or
$B^+_r \setminus I \neq \emptyset$.  $I$ is a \emph{model} of $\prog$
if it satisfies all rules of~$\prog$, in symbols $I \models \prog$. %
For brevity, we view propositional formulas 
as sets of formulas (e.g., clauses) that need to be satisfied, and
use the notion of interpretations, models, and satisfiability analogously. %
The \emph{Gelfond-Lifschitz
  (GL) reduct} of~$\prog$ under~$I$ is the program~$\prog^I$ obtained
from $\prog$ by first removing all rules~$r$ with
$B^-_r\cap I\neq \emptyset$ and then removing all~$\neg z$ where
$z \in B^-_r$ from the remaining
rules~$r$~\cite{GelfondLifschitz91}. %
$I$ is an \emph{answer set} of a program~$\prog$ if $I$ is a minimal
model of~$\prog^I$. %
The problem of deciding whether an \ASP program has an answer set is called
\emph{consistency}, which is \SIGMA{2}{P}-complete~\cite{EiterGottlob95}. 
If the input is restricted to normal programs, the complexity drops to
\NP-complete~\cite{BidoitFroidevaux91,MarekTruszczynski91}.
A head-cycle-free program~$\prog$ %
can be translated into a normal program in polynomial
time~\cite{Ben-EliyahuDechter94}.
\FIX{Further, the answer sets of a tight program can 
be represented by means of the models of a propositional formula, obtainable in linear time via, e.g., Clark's completion~\cite{Clark77}.}
The following characterization of answer sets is often
applied when considering normal programs~\cite{LinZhao03}.
For a given set~$A\subseteq\at(\Pi)$ of atoms, a function~$\varphi: A \rightarrow \{0,\ldots,\Card{A}-1\}$ is an \emph{ordering} over~$A$.
Let~$I$ be a model of a normal program~$\prog$, and~$\varphi$ be an ordering over~$I$.
\FIX{We say a rule~$r\in \Pi$ is \emph{suitable for proving~$a\in I$} if (i) $a\in H_r$, (ii) $B_r^+\subseteq I$, (iii) $I\cap B_r^- = \emptyset$, as well as (iv) $I\cap (H_r\setminus\{a\}])=\emptyset$.}
 An atom~$a\in I$ is \emph{proven} 
if there is a rule~$r\in\prog$ \emph{proving~$a$}, which is the case if $r$ is suitable for proving~$a$ and~$\varphi(b) < \varphi(a)$ for every~$b\in B_r^+$. Then, $I$ is an
\emph{answer set} of~$\prog$ if (i)~$I$ is a model of~$\prog$, and
(ii) \emph{$I$ is proven}, i.e., every~$a \in I$ is proven.
\FIX{This characterization vacuously holds also for head-cycle-free
programs~\cite{Ben-EliyahuDechter94}, since in HCF programs, vaguely speaking, all but one atom of the head of any rule can be ``shifted'' to the negative body~\cite{DixGottlobMarek96}.}
\futuresketch{Given a program~$\prog$. It is folklore that an atom~$a$ of any answer
set of $\prog$ has to occur in some head of a rule
of~$\prog$~\cite[Ch~2]{GebserKaminskiKaufmannSchaub12}, which we 
hence assume in the following. %
}
\longversion{%
} 
\longversion{%
}%
\longversion{%
}%
\FIX{
\begin{figure}[t]%
  \centering
  \shortversion{ %
    \input{graph0/depgraph}%
    \vspace{-.4em}
    \caption{\FIX{Dependency graph~$D_\prog$ of~$\prog$ (cf., Example~\ref{ex:running1}).}}
\label{fig:depgraph}
  }%
\end{figure}
}
\begin{example}%
\label{ex:running1}\label{ex:running}
Consider the given program
\vspace{-0.1em}
$\prog\eqdef$
$\SB\overbrace{ a \lor b \hsep}^{r_1};\, %
\overbrace{c \lor e \hsep d}^{r_2};\, %
\overbrace{d \hsep b, \neg e}^{r_3};\,\allowbreak %
\overbrace{e \hsep b, \neg d}^{r_4};\,
\overbrace{b \hsep e, \neg d}^{r_5};\, %
\overbrace{d \hsep \neg b}^{r_6} %
\SE$.
\FIX{Observe that $\prog$ is not tight, since the dependency graph~$D_\prog$ of Figure~\ref{fig:depgraph} contains the cycle~$b,d,e$.
However, the program~$\prog$ is head-cycle-free since there is neither an~$\{a,b\}$-cycle, nor a~$\{c,e\}$-cycle in~$D_\prog$.
Therefore, rule~$r_1$ allows shifting~\cite{DixGottlobMarek96} and actually corresponds to the two rules $a\leftarrow \neg b$ and~$b\leftarrow \neg a$. Analogously, rule~$r_2$ can be seen as the rules~$c\leftarrow d,\neg e$ and~$e\leftarrow d,\neg c$.}
Then, $I\eqdef\{b, c, d\}$ is an answer set of~$\prog$,
since~$I\models\Pi$, and we can prove with ordering
$\varphi \eqdef\{b\mapsto 0, d\mapsto 1, c\mapsto 2\}$
atom~$b$ by rule~$r_1$, 
atom~$d$ by rule~$r_3$, and
atom~$c$ by rule~$r_2$.
Further answer sets of~$\Pi$ are $\{b,e\}$,
$\{a,c,d\}$, and~$\{a,d,e\}$.

\FIX{
The characterization above already fails for simple programs that are not HCF.
Consider for example program~$\prog'\eqdef \{ a \vee b \leftarrow; a \leftarrow b; b \leftarrow a\}$, which has only one answer set~$I'=\{a,b\}$.
However, $I'$ cannot be proven. If the first rule~$a \vee b \leftarrow$ shall prove~$a$, we require~$b\notin I'$ (and vice versa). Then, the remaining two rules of~$\Pi'$ can only prove either~$a$ or~$b$, but fail to prove~$I'$, since both rules proving~$I'$ (together) prohibit every ordering due to the cyclic dependency.
}
\end{example}%

\smallskip
\noindent\textbf{Tree Decompositions (TDs).} %
We assume familiarity with graph terminology, cf.,~\cite{Diestel12}.
A \emph{tree decomposition (TD)}~\cite{RobertsonSeymour86} %
of a given graph~$G{=}(V,E)$ is a pair
$\TTT{=}(T,\chi)$ where $T$ is a tree rooted at~$\rootOf(T)$ and $\chi$ %
assigns to each node $t$ of~$T$ a set~$\chi(t)\subseteq V$,
called \emph{bag}, such that (i) $V=\bigcup_{t\text{ of }T}\chi(t)$, (ii)
$E\subseteq\SB \{u,v\} \SM t\text{ in } T, \{u,v\}\subseteq \chi(t)\SE$,
and (iii) for each $r, s, t\text{ of } T$, such that $s$ lies on the path
from~$r$ to $t$, we have $\chi(r) \cap \chi(t) \subseteq \chi(s)$.
For every node~$t$ of~$T$, we denote by $\children(t)$ the \emph{set of child nodes of~$t$} in~$T$.
The \emph{bags~$\chi_{\leq t}$ below~$t$} consists of the union of all bags of nodes below~$t$ in~$T$, including~$t$. %
We
let $\width(\TTT) {\eqdef} \max_{t\text{ of } T}\Card{\chi(t)}-1$.
The
\emph{treewidth} $\tw{G}$ of $G$ is the minimum $\width({\TTT})$ over
all TDs $\TTT$ of $G$. For a graph~$G$, one can compute a TD of~$G$, whose width is at most~$5\cdot\tw{G}$, in \emph{single exponential time}~\cite{BodlaenderEtAl13} in the treewidth (5-approximation). %
\FIXR{The so-called \emph{pathwidth} of~$G$ refers to the minimum~$\width(\TTT)$
over all TDs $\TTT$ of~$G$, whose trees are just paths.}
For a node~$t \text{ of } T$, we say that $\type(t)$ is $\leaf$ if $t$ has
no children and~$\chi(t)=\emptyset$; $\join$ if $t$ has children~$t'$ and $t''$ with
$t'\neq t''$ and $\chi(t) = \chi(t') = \chi(t'')$; $\intr$
(``introduce'') if $t$ has a single child~$t'$,
$\chi(t') \subseteq \chi(t)$ and $|\chi(t)| = |\chi(t')| + 1$; $\rem$
if $t$ has a single child~$t'$,
$\chi(t') \supseteq \chi(t)$ and $|\chi(t')| = |\chi(t)| + 1$. If for
every node $t\text{ of } T$, %
$\type(t) \in \{ \leaf, \join, \intr, \rem\}$, the TD is called \emph{nice}.
A TD can be turned into a nice TD~\cite{Kloks94a}[Lem.\ 13.1.3] \emph{without increasing the width} in linear~time.

\begin{example}\label{ex:td}
  Figure~\ref{fig:graph-td} illustrates a graph~$G$ and a TD~$\mathcal{T}$ of~$G$ of width~$2$, which is also the treewidth of~$G$,
  since~$G$ contains a complete graph on vertices $e$,$b$,$d$. In general, if a graph contains a complete graph among~$k+1$ vertices,
  the treewith of the graph is at least~$k$, cf.,~\cite{Kloks94a}.
\end{example}

\smallskip
\noindent\textbf{Dynamic Programming on TDs.}
Solvers based on \emph{dynamic programming (DP)} %
evaluate a given input instance~$\mathcal{I}$ in parts along a TD of a graph representation~$G$ of the instance.
Thereby, for each node~$t$ of the TD, intermediate results are %
stored in a \emph{table}~$\tab{t}$. %
This is achieved by running a \emph{table algorithm}, %
which is designed for~$G$, 
and stores in~$\tab{t}$ results of problem parts of~$\mathcal{I}$,
thereby considering tables~$\tab{t'}$ for child nodes~$t'$ of~$t$. %
DP works for \emph{many problems} as follows. %
\vspace{-.2em}\begin{enumerate}%
\item Construct a \emph{graph representation}~$G$ of~$\mathcal{I}$.\vspace{-.2em}
\item Compute a TD~$\TTT=(T,\chi)$ of~$G$, which is obtainable via heuristics, e.g.,~\cite{AbseherMusliuWoltran17}.
\vspace{-.2em}
\item\label{step:dp} Traverse the nodes of~$T$ in
  post-order (bottom-up tree traversal of~$T$).
  At every node~$t$ of $T$ during post-order traversal, execute a table algorithm
  that takes as input a bag $\chi(t)$, a certain \emph{bag instance}~$\mathcal{I}_t$ depending on the problem, as well as previously computed child tables of~$t$. Then, the results of this execution is stored in table~$\tab{t}$.\vspace{-.2em}
\item Finally, interpret table~$\tab{n}$ for the root node~$n$ of~$T$ in order to \emph{output the solution} to the problem for instance~$\mathcal{I}$.
\end{enumerate}

\begin{figure}[t]%
  \centering
  \shortversion{ %
    \input{graph0/graph}%
    \includegraphics{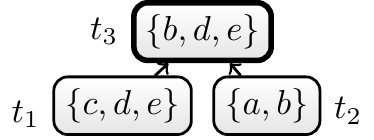}
    \vspace{-.4em}
    \caption{Graph~$G$ (left) and a tree decomposition~$\mathcal{T}$ of~$G$ (right).}
  }%
  \longversion{%
    \begin{subfigure}[c]{0.47\textwidth}
      \centering%
      \input{graph0/graph}%
      \input{graph0/td}%
      \caption{Graph~$G_1$ and a tree decomposition of~$G_1$.}
      \label{fig:graph-td}
    \end{subfigure}
    \begin{subfigure}[c]{0.5\textwidth}
      \centering \input{graph0/graph_inc}%
      \input{graph0/td_inc}%
      \caption{Graph~$G_2$ and a tree decomposition of~$G_2$.}
      \label{fig:graph-td2}%
    \end{subfqigure}
    \caption{Graphs~$G_1, G_2$ and two corresponding tree
      decompositions.}
  }%
  \label{fig:graph-td}%
\end{figure}

\noindent In order to use TDs for ASP, we need
dedicated graph representations of programs~\cite{JaklPichlerWoltran09}.
The \emph{primal graph\footnote{Analogously, the primal graph~$\mathcal{G}_F$ of a propositional Formula~$F$ (in CNF) %
uses variables of~$F$ as vertices and adjoins two vertices~$a,b$ by an edge, if there is a clause in~$F$ %
containing~$a,b$.%
}}~$\mathcal{G}_\prog$
of program~$\prog$ has the atoms of~$\prog$ as vertices and an
edge~$\{a,b\}$ if there exists a rule~$r \in \prog$ and $a,b \in \at(r)$.
\longversion{The \emph{incidence graph}~$I_\prog$ of $\prog$ is the bipartite
graph that has the atoms and rules of~$\prog$ as vertices and an
edge~$a\, r$ if $a \in \at(r)$ for some rule~$r \in \prog$.}
\noindent Let ${\cal T} = (T, \chi)$ be a TD of primal
graph~$\mathcal{G}_\prog$ of a program $\prog$, and let~$t$ be a node of~$T$. %
The \emph{bag program}~$\Pi_t$ contains rules entirely covered by
the bag~$\chi(t)$. Formally, $\prog_t \eqdef \SB r \SM r \in \prog, \at(r) \subseteq \chi(t)\}$.
\begin{example}
  Recall program~$\prog$ from Example~\ref{ex:running1}. Observe
  that graph~$G$ of Figure~\ref{fig:graph-td} is the primal graph
  of~$\Pi$.
  Further, we have $\Pi_{t_1}=\{r_2\}$, $\Pi_{t_2}=\{r_1\}$,
  and $\Pi_{t_3}=\{r_3, r_4, r_5, r_6\}$. Note that in general a rule might appear in several bag programs.
  \longversion{%
    Further, graph~$G_2$ of Figure~\ref{fig:graph-td2} is the
    incidence graph of~$\prog$.
  }
\end{example}

Now, the missing ingredient for solving problems via dynamic programming along a given TD, is a suitable table algorithm.
Such table algorithms have been already presented for~\SAT~\cite{SamerSzeider10b} and ASP~\cite{JaklPichlerWoltran09,FichteEtAl17a,FichteHecher19}. 
We only briefly sketch the ideas of a table algorithm using the primal graph that computes models %
of a given program~$\Pi$. %
Each table~$\tab{t}$ consists of rows storing interpretations
over atoms in the bag~$\chi(t)$.
Then, the table~$\tab{t}$ for a leaf node~$t$ %
consists of the empty interpretation.
For a node~$t$ with introduced variable $a\in\chi(t)$, 
we store in~$\tab{t}$ 
interpretations of the child table, but %
for each such interpretation %
we decide whether~$a$ is in the interpretation or not, 
and ensure that $\Pi_t$ is satisfied.
When an atom~$b$ is forgotten in a forget node~$t$, we store interpretations of the child table, but projected to~$\chi(t)$.
By the properties of a TD, it is then guaranteed that all rules containing~$b$ 
have been processed so far.
For a join node~$t$, we store in~$\tab{t}$ interpretations 
that are in both child tables of~$t$.

\futuresketch{
\subsection{Reduction of SAT to Normal ASP}

Note that~$w!/w^w=e^{-\mathcal{O}(k)}$ via Sterling's formula~\cite{LokshtanovMarxSaurabh11}.

Given a SAT formula~$F$ and a tree decomposition~$\mathcal{T}=(T,\chi)$ of the primal graph of~$F$ such that~$T=(N,A,n)$.
We reduce~$F$ to a normal program~$\Pi$ such that
$F$ is satisfiable if and only if~$\Pi$ has an answer set.
The reduction itself is guided by~$\mathcal{T}$ and the width of~$\mathcal{T}'$ is 
$\bigO{w/\log(w)}=\bigO{w/\log(w\cdot\log(w))}=\bigO{w/[\log(w)+\log(\log(w))]}$, 
where~$w$ is the width of~$\mathcal{T}$.
In total, with ASP, we can assign up to~$2^{\bigO{s\cdot\log(s)}}$ many states
for each bag~$\chi(t')$ of~$\mathcal{T}'$ of cardinality at most~$s$.
Concretely, we have exactly~$\Sigma_{I \subseteq \chi(t')} \Card{I}!$ many states.
We will use these states to simulate SAT, but thereby ``save'' in the treewidth.

Given any two bags~$\chi(t_1),\chi(t_2)$ of~$\mathcal{T}$ such that~$t_1$ is a children of~$t_2$, we construct the following ASP rules in~$\Pi$.
We assume an arbitrary total ordering~$\prec_{t_1}$ of the assignments~$I\in 2^{\chi(t_1)}$
in bag~$t_1$.
Further, we assume an arbitrary total ordering~$\prec_{t'_1}$ among $2^{\chi(t'_1)} \times ord(\chi(t'_1))$.
We use a mapping~$m: 2^{\chi(t_1)} \rightarrow [2^{\chi(t'_1)} \times ord(\chi(t'_1))]$,
which is a bijection between an assignment~$I$ in~$\chi(t_1)$ and an assignment~$I'$ over~$\chi(t'_1)$ together with
an ordering among those atoms in~$I'$.
The mapping~$m$ is naturally defined by assigning the elements in~$\prec_{t_1}$ to~$\prec_{t'_1}$ according to
the ordinal of the element.
Then, we add for each element~$I\in2^{\chi(t_1)}$ the rules~$0 { a_I } 1.$, but ensure an ``at most one'' behavior.
This can be done later on top of the TD. Further, we add~$v_j \leftarrow v_i, a_I.$
For each element~$I\in2^{\chi(t_1)}$ we add a parent~$t'_{1.i}$ to~$t'_1$.

\subsection{Reduction of normal ASP to PASP using trick of exponential compression as used in the hierarchy}

TBD: Just reduce normal ASP to InvPASP problem, using the same ideas as in the hierarchy, exponential compression and such, then easily reduce InvPASP to PASP.
We have now a lower bound for PASP using normal ASP.
}

\section{Treewidth-Aware Reductions to \SAT}

Having the basic concept of dynamic programming in mind,
we use this idea to design a reduction of an HCF program~$\Pi$ to a \SAT formula~$F$,
which is treewidth-aware.
The reduction is inspired by ideas of a DP algorithm for consistency of \HCF programs~\cite{FichteHecher19} and the idea of level mappings~\cite{Janhunen06}.
Intuitively, \emph{global} orderings can cause already huge blow-up in the treewidth,
e.g., reductions, where all atoms are ordered at once, often cause long rules, whose number of atoms vastly exceeds treewidth. %
\FIXR{This is indeed not surprising, since for a given program with~$n$ many atoms, we have that 
global level mappings use~$\mathcal{O}(\log(n))$ additional auxiliary atoms. Consequently,
already one rule that actually utilizes these mappings causes a treewidth blow-up factor~$\mathcal{O}(\log(n))$,
which obviously is unbounded by the treewidth of the given program.}

As a result, we apply these orderings only locally within the bags of a TD.
\FIX{Note that while this approach might look similar to existing techniques
that are applied on a component-by-component basis~\cite{Janhunen06,GebserJanhunenRitanen14,BomansonEtAl16},
the approach is slightly different, as different components of the positive dependency graph~$D_\Pi$ might be spread among different bags of a TD and a bag might only contain parts of components.
If components are required to be spread across any TD of primal graph~$\mathcal{G}_\prog$  whose width coincides with the treewidth, only parts of cycles of~$D_\prog$ can be analyzed by a table algorithm in a bag.
This is the reason why already the consistency problem of \ASP remains slightly harder than the decision problem \SAT (under ETH).
Consequently, instead of global orderings or orderings per component, we use local orderings, which only order within bags. However, these orderings need to be ``synchronized'' between bags and come at the price of loosing bijective preservation of answer sets in general. 
The reason for that will be discussed in Section~\ref{sec:bijective}.
}

\begin{figure}%
\centering
\begin{tikzpicture}[node distance=1mm, scale=0.16]%
\def\nodedist{0.7em}
\tikzset{every path/.style=thick}
\node (ableft) [tdnode,label={[yshift=-0.25em,xshift=0.25em] left:$t_3$}] {$\textcolor{blue}{\chi(t_3)}$};
\node (leaf1) [below=\nodedist of ableft,xshift=-2em, tdnode,label={[yshift=-0.25em,xshift=0.1em]left:$t_1$}] {$\textcolor{blue}{\chi(t_1)}$};
\node (leaf12) [below=\nodedist of ableft,xshift=1.5em, tdnode,label={[yshift=-0.25em,xshift=-0.1em]right:$t_2$}] {$\textcolor{blue}{\chi(t_2)}$};
\node (leaf2) [tdnode,label={[xshift=-1.0em, yshift=-0.15em]above right:$t_4$}, right = 0.5em of ableft]  {$\textcolor{blue}{\chi(t_4)}$};
\coordinate (middle) at ($ (ableft.north east)!.5!(leaf2.north west) $);
\node (root) [tdnode,ultra thick,label={[]left:$t_5$}, above = \nodedist of middle] {$\textcolor{blue}{\chi(t_5)}$};
\node (llabel) [left=of root,xshift=-1.5em] {$\mathcal{T}$:};
\coordinate (top) at ($ (root.north east)+(3.5em,0) $);
\coordinate (bot) at ($ (top)+(0,-4em) $);
\draw [stealth'-] (leaf1) to (ableft);
\draw [stealth'-] ($(leaf12.north)+(0.2em,-0.1em)$) to ($(ableft.south)+(-0.25em,0.0em)$);
\draw [-stealth'] (root) to (ableft);
\draw [-stealth'] (root) to (leaf2);
\node (rleaf1) [right=5em of ableft,tdnode,label={[yshift=-0.25em,xshift=0.25em] left:$ $}] {$\qquad\qquad\quad\qquad\qquad\qquad\quad$};
\node (rleaf1p) [right=5em of ableft,xshift=0.35em,inner sep=0.5] {$f(t_3,\textcolor{blue}{\chi(t_3)}, \textcolor{red}{\{\chi'(t_1), \chi'(t_2)\}})$};
\node (rem1) [below=1em of rleaf1,xshift=-2em, tdnode,label={[yshift=-0.25em,xshift=0.3em]left:$ $}] {$\qquad\qquad\qquad$};
\node (rem1p) [below=\nodedist of rleaf1,yshift=-0.5em,inner sep=0.5,xshift=-2.05em] {$f(t_1,\textcolor{blue}{\chi(t_1),\textcolor{red}{\emptyset}})$};
\node (remab) [below=\nodedist of rleaf1,yshift=-0.3em,xshift=5.5em, tdnode,label={[yshift=-0.25em,xshift=-0.1em]right:$ $}] {$\quad\qquad\qquad\quad$};
\node (remabp) [below=\nodedist of rleaf1,yshift=-0.5em,inner sep=0.5,xshift=5.5em] {$f(t_2,\textcolor{blue}{\chi(t_2)},\textcolor{red}{\emptyset})$};
\node (rleaf2) [tdnode,label={[xshift=-0.0em, yshift=-0.15em]above right:$ $}, right = 0.5em of rleaf1]  {$\qquad\qquad\qquad$};
\node (rleaf2p) [right=0.5em of rleaf1,yshift=-0.4em,inner sep=0.5,xshift=.5em,yshift=0.4em,] {$f(t_4,\textcolor{blue}{\chi(t_4)},\textcolor{red}{\emptyset})$};
\coordinate (middle) at ($ (rleaf1.north east)!.5!(rleaf2.north west) $);
\node (join) [tdnode,ultra thick,label={[xshift=-0.3em]right:$ $}, above  = \nodedist of middle] {\qquad\qquad\qquad\qquad\qquad\qquad\qquad};
\node (llabel) [left=of join,xshift= -3em] {$\mathcal{T}'$:};
\node (joinp) [above = \nodedist of middle,yshift=0.25em,xshift=-.05em,inner sep=0.5] {$f(t_5,\textcolor{blue}{\chi(t_5)}, \textcolor{red}{\{\chi'(t_3), \chi'(t_4)\}})$};
\coordinate (top) at ($ (join.north east)+(3.5em,0) $);
\coordinate (bot) at ($ (top)+(0,-4em) $);
\draw [stealth'-] (rem1) to (rleaf1);
\draw [stealth'-] ($(remab.north)+(0.2em,-0.1em)$) to ( $(rleaf1.south)+(-0.0em,0.0em)$);
\draw [-stealth'] (join) to (rleaf1);
\draw [-stealth'] ($(join.south)+(0em,0em)$) to (rleaf2);
\draw[dashedarrow,out=-170,in=-50,blue] (rem1) to (leaf1);
\draw[dashedarrow,out=-168,in=-40,blue] (remab) to (leaf12);
\draw[dashedarrow,out=-165,in=-27,blue] (rleaf1) to ($(ableft.south east)+(-2.0em,0.0em)$);
\draw[dashedarrow,out=-191,in=14,blue] (rleaf2) to (leaf2);
\draw[dashedarrow,out=-180,in=30,blue] (join) to ($(root.north)$);
\end{tikzpicture}%
\caption{\FIXR{High-level illustration of the treewidth-awareness of our reduction from \ASP to \SAT. We assume a given program~$\Pi$ and a tree decomposition~$\mathcal{T}=(T,\chi)$ of~$G_\Pi$. Then, our reduction is constructed for each node~$t$ of~$T$ and it immediately gives rise to a tree decomposition~$\mathcal{T}'=(T,\chi')$ of the resulting \SAT formula. %
Thereby, each resulting bag~$\chi'(t)$ %
functionally depends on~$\chi(t)$ (as well as on child or parent bags).}}\label{fig:decompguided2}
\end{figure}
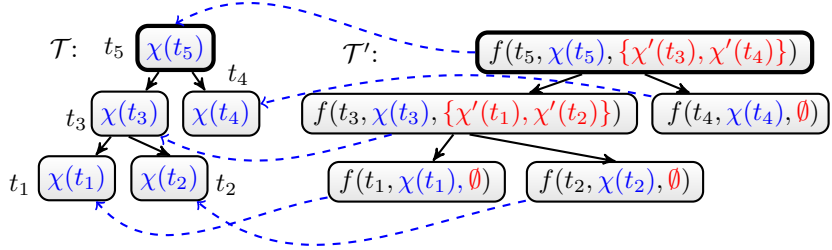

\FIXR{More concretely, our reduction is \emph{guided} by a TD~$\mathcal{T}=(T,\chi)$ of primal graph~$\mathcal{G}_\Pi$
and uses core ideas of dynamic programming along TD~$\mathcal{T}$
to ensure only a slight increase in treewidth of the resulting \SAT formula.
This increase of treewidth is by construction, since our reduction essentially yields a tree decomposition of the primal graph of the resulting \SAT formula, which functionally depends on~$\mathcal{T}$, as sketched in Figure~\ref{fig:decompguided2}.}
Intuitively, thereby the aforementioned reduction takes care to keep the increase of width
local, i.e., the increase of width happens \emph{within} the bags of~$\mathcal{T}$.
Concretely, if~$\width(\mathcal{T})$ is bounded by some value~$\mathcal{O}(k)$,
the treewidth of the resulting formula~$F$ is at most~$\mathcal{O}(k\cdot\log(k))$.

For encoding orderings along a TD, we need the following notation.
Let us consider a TD~$\mathcal{T}=(T,\chi)$ of~$\mathcal{G}_\Pi$, and a node~$t$ of~$T$.
We refer to an ordering over~$\chi(t)$ by \emph{$t$-local ordering}. 
\begin{definition}
A \emph{$\mathcal{T}$-local ordering} is a set containing one~$t$-local ordering~$\varphi_t$ for every~$t$ of~$T$
such that there is an interpretation~$I$ with 
(1) \emph{satisfiability}: $I\models\Pi_t$ for every node~$t$ of~$T$, 
(2) \emph{provability}: for every $a\in I$, there is a node~$t$ of~$T$ and a rule~$r\in\Pi_{t}$ proving~$a$, and
(3) \emph{compatibility}: for every nodes~$t,t'$ of~$T$ and every~$a,b\in\chi(t)\cap\chi(t')$, %
whenever $\varphi_t(a) < \varphi_t(b)$ then~$\varphi_{t'}(a) < \varphi_{t'}(b)$. %
\end{definition}
For an ordering~$\varphi$, we use the \emph{canonical $t$-local ordering~$\hat\varphi_t$} for each~$t$ of~$T$ as follows. 
Intuitively, atoms~$a\in\chi(t)$ with smallest ordering position $\varphi(a)$
among all atoms in~$\chi(t)$ get~$\hat\varphi_t(a)=0$, second-smallest get value~$1$, and so on.
Formally, we define $\hat\varphi_t(a) \eqdef %
\rank_{t}(a,\varphi) - 1 \text{ for each } a\in \chi(t)$, %
where~$\rank_t(a,\varphi)$ is the ordinal number (rank) of $a$ according to smallest ordering position~$\varphi(a)$ among~$\chi(t)$.

\begin{example}\label{ex:localorderings}
Consider program~$\Pi$,
answer set $I=\{b, c, d\}$, and ordering
$\varphi =\{b\mapsto 0, d\mapsto 1, c\mapsto 2\}$
 of Example~\ref{ex:running}.
Ordering~$\varphi$ can easily be extended to ordering
$\varphi' \eqdef\{a\mapsto 0, e\mapsto 0,b\mapsto 0, d\mapsto 1, c\mapsto 2\}$
over~$\at(\Pi)$. 
Then, using TD~$\mathcal{T}$ of~${\mathcal{G}}_{\Pi}$,
we can construct~$\mathcal{T}$-local ordering $\mathcal{M}\eqdef\{\hat\varphi_{t_1}, \hat\varphi_{t_2}, \hat\varphi_{t_3}\}$ of~$\varphi'$, where $\hat\varphi_{t_1}= \{e\mapsto 0,d\mapsto 1, c\mapsto 2\}$, $\hat\varphi_{t_2}= \{a\mapsto 0, b\mapsto 0\}$, and $\hat\varphi_{t_3}= \{e\mapsto 0,b\mapsto 0, d\mapsto 1\}$.
Consider a TD~$\mathcal{T}'$ of~$\mathcal{G}_{\Pi}$,
which is similar to~$\mathcal{T}$, but $t_1$ has a child node~$t'$, whose bag is $\{c,e\}$. Then, $\mathcal{M}\cup\{\hat\varphi_{t'}\}$ with $\hat\varphi_{t'}=\{e\mapsto 0,c\mapsto 1\}$ is a $\mathcal{T}'$-local ordering.
\end{example}

In our reduction, we use the following propositional variables.
For each atom~$x\in\at(\Pi)$, we use~$x$ also as propositional variable.
For each atom~$x\in\chi(t)$ of each node~$t$ of~$T$, we use~$\ceil{\log(\Card{\chi(t)})}$
many variables of the form~$b^i_{x_t}$ forming the $i$-th bit of the $t$-local ordering position (in binary) of~$x$.
By the shortcut notation~$\bvali{x}{t}{j}$, we refer to the \emph{conjunction of literals over bits~$b^i_{x_t}$} for~$1\leq i\leq \ceil{\log(\Card{\chi(t)})}$ according to the representation of the number~$j$ in binary.
For atoms~$x,x'\in\chi(t)$ of node~$t$ of~$T$, we use
the following notation to indicate that atom~$x$ is ordered before atom~$x'$:
\vspace{-.25em}
$$x \prec_t x' \eqdef \hspace{-1.6em}\bigvee_{1\leq i \leq \ceil{\log(\Card{\chi(t)})}}\hspace{-1em}(b_{x'_t}^i \wedge \neg b_{x_t}^i \wedge \bigwedge_{i < j \leq \ceil{\log(\Card{\chi(t)})}}\hspace{-1.6em} (b_{x_t}^j \longrightarrow b_{x'_t}^j)).$$

\begin{example}
Consider Example~\ref{ex:localorderings}
and the~$\mathcal{T}$-local ordering~$\mathcal{M}=\{\varphi_{t_1}, \varphi_{t_2}, \varphi_{t_3}\}$.
One could encode ordering position~$\varphi_{t_1}(e)=0$ using two bit variables $b^1_{e_{t_1}},b^2_{e_{t_1}}$ and forcing it to false.
This results in formula~$\bvali{e}{t_1}{0}=\neg b^1_{e_{t_1}} \wedge \neg b^0_{e_{t_1}}$.
Then, we formulate~$\varphi_{t_1}(d)=1$ by~$\bvali{d}{t_1}{1}=\neg b^1_{d_{t_1}} \wedge b^0_{d_{t_1}}$, and~$\varphi_{t_1}(c)=2$ by~$\bvali{c}{t_1}{2}=b^1_{c_{t_1}} \wedge \neg b^0_{c_{t_1}}$. 
For the whole resulting formula, $(e\prec_{t_1}d)$, $(d\prec_{t_1}c)$ as well as $(e\prec_{t_1}c)$ hold.
\end{example}

\subsection{TD-guided Reduction to \SAT}\label{sec:tdguided}

\noindent 
For solving consistency, we require to construct the following Formulas~(\ref{red:checkrules})--(\ref{red:checkfirst}) below \emph{for each TD node~$t$} of~$T$
having child nodes~$\children(t)=\{t_1, \ldots, t_\ell\}$.
Thereby, these formulas aim at constructing $\mathcal{T}$-local orderings along the TD~$\mathcal{T}$, where Formulas~(\ref{red:checkrules}) ensure satisfiability,
Formulas~(\ref{red:prop}) take care of compatibility along the TD,
and Formulas~(\ref{red:checkfirst}) enforce provability within a node, which is then guided along the TD by Formulas~(\ref{red:checkremove})--(\ref{red:check}).

\noindent Concretely, Formulas~(\ref{red:checkrules}) ensure that
the variables of the constructed \SAT formula~$F$ are such that
all (bag) rules are satisfied.
Then, whenever in node~$t$ an atom~$x$ has a smaller ordering position than an atom~$y$ (using $\prec_{t}$), 
this must hold also for the parent node of~$t$ and vice versa,
cf., Formulas~(\ref{red:prop}).
Formulas~(\ref{red:checkremove}) guarantee, for nodes~$t$
removing bag atom~$x$, i.e., $x\in\chi(t)\setminus\chi(t')$, 
that~$x$ is proven if~$x$ is set to true.
Similarly, this is required for atoms~$x\in\chi(n)$ that are in the root node~$n=\rootOf(T)$
and therefore never forgotten, cf., Formulas~(\ref{red:checkremove2}).
At the same time we ensure by Formulas~(\ref{red:check}) that an atom~$x$ is proven up to node~$t$
if and only if it is proven up to some child node of~$t$ or freshly proven in node~$t$.
Finally, Formulas~(\ref{red:checkfirst}) make sure that an atom~$x$ is freshly proven in node~$t$
if and only if there is at least one rule~$r\in \Pi_t$ proving~$x$.
{
\begin{flalign}
	\label{red:checkrules}&\bigvee_{b\in B_r^+} \neg b \vee \bigvee_{a\in B_r^- \cup H_r} a &&{\text{for each } r\in \Pi_t}\\
	&(x \prec_{t'} y) \longleftrightarrow (x \prec_t y)&&{\text{for each }t'\in\children(t)\text{ and } %
		x,y\in\chi(t)\cap\chi(t')}\notag\\[-.8em]%
	\label{red:prop}&&& \text{with }x\neq y\\
	\label{red:checkremove}&x\longrightarrow p^{x}_{<t'}&&{\text{for each }{t'\in\children(t)}\text{ and } %
		x\in\chi(t')\setminus\chi(t)}\\
	\label{red:checkremove2}&x \longrightarrow p^{x}_{<n}&&{\text{for each }x\in\chi(n)\text{ with }%
		n=\rootOf(T)}\\ %
	\label{red:check}&p^x_{<t} \longleftrightarrow p^x_t \vee (\hspace{-1em}\bigvee_{t' \in \children(t), x\in\chi(t')} \hspace{-2em} p^x_{<t'})&&{\text{for each } x\in \chi(t)}\\
	\label{red:checkfirst}&p_{t}^x \longleftrightarrow \bigvee_{r\in\Pi_t, x \in H_r} \hspace{-0.05em}(\hspace{-0.15em}\bigwedge_{b\in B_r^+}\hspace{-0.5em}b \wedge x \wedge &&{\text{for each } x\in \chi(t)}\\[-.45em]
	&\qquad\qquad (b \prec_t x) \wedge\hspace{-1.5em}\bigwedge_{a\in B_r^- \cup (H_r \setminus \{x\})}\hspace{-1.5em} \neg a)\notag%
\end{flalign}
}

\begin{example}\label{ex:reduction}
Recall program~$\Pi$ from Example~\ref{ex:running1}, and TD~$\mathcal{T}$ of $\mathcal{G}_\Pi$ given in Figure~\ref{fig:graph-td}.
We briefly show Formula~$F$ for node~$t_3$. 

\smallskip
\noindent\begin{tabular}{@{\hspace{0.15em}}l@{\hspace{0.15em}}|@{\hspace{0.15em}}l@{\hspace{0.0em}}}
Formulas & Formula $F$\\
\hline
(\ref{red:checkrules})& $\neg b \vee e \vee d$; $\neg b \vee d \vee e$; $\neg e \vee d \vee b$; $b \vee d$\\
(\ref{red:prop})& $(d \prec_{t_1} e) \leftrightarrow (d \prec_{t_3} e)$; $(e \prec_{t_1} d) \leftrightarrow (e \prec_{t_3} d)$\\
(\ref{red:checkremove})& $c \rightarrow p^c_{<t_1}$; $a \rightarrow p^a_{<t_2}$\\
(\ref{red:checkremove2}) & $b \rightarrow p^b_{<t_3}$; $d \rightarrow p^d_{<t_3}$; $e \rightarrow p^e_{<t_3}$\\ %
(\ref{red:check}) & $p^b_{<t_3} \leftrightarrow (p^b_{t_3} \vee p^b_{<t_2})$; %
 		  $p^d_{<t_3} \leftrightarrow (p^d_{t_3} \vee p^d_{<t_1})$; $p^e_{<t_3} \leftrightarrow (p^e_{t_3} \vee p^e_{<t_1})$\\ %
(\ref{red:checkfirst})& $p^b_{t_3} \leftrightarrow [e \wedge b \wedge (e \prec_{t_3} b) \wedge \neg d]$; %
		      $p^e_{t_3} \leftrightarrow [b \wedge e \wedge (b \prec_{t_3} e) \wedge \neg d]$; \\
			 & $p^d_{t_3} \leftrightarrow [(b \wedge d \wedge (b \prec_{t_3} d) \wedge \neg e) \vee (d \wedge \neg b)]$ %
\end{tabular}%
\end{example}

Next, we show that the reduction is indeed aware of the treewidth and
that the treewidth is slightly increased.

\begin{theorem}[Treewidth-awareness]\label{thm:runtime1}
The reduction from an HCF program~$\Pi$ and a nice TD~$\mathcal{T}=(T,\chi)$ of~$\mathcal{G}_\Pi$ to SAT formula~$F$ consisting of Formulas~(\ref{red:checkrules})--(\ref{red:checkfirst}) ensures that %
if~$k$ is the width of $\mathcal{T}$,
then the treewidth of~$\mathcal{G}_F$ is at most~$\mathcal{O}(k\cdot\log(k))$.
\end{theorem}

\begin{proof}
We construct a TD~$\mathcal{T}'=(T,\chi')$ of~$\mathcal{G}_F$ to show that the width of~$\mathcal{T}'$ increases from~$k$ to~$\mathcal{O}(k\cdot\log(k))$.
To this end, let~$t$ be a node of~$T$ with~$\children(t)=\langle t_1, \ldots, t_\ell \rangle$
and let~$\hat t$ be the parent of~$t$ (if exists).
We define~$B(t,x) \eqdef \{b_{x_{t}}^j \mid x\in\chi(t), 1\leq j \leq \ceil{\log(\Card{\chi(t)})}\}$.
We inductively define~$\chi'(t)\eqdef \chi(t) \cup (\bigcup_{x\in\chi(t)} B(t,x) \cup B(\hat t,x)) \cup \{p^y_{<{t'}}, p^x_{{t}} \mid t'\in\{t,t_1,\ldots,t_\ell\} , x\in\chi(t), y\in\chi(t')\}$.
Observe that indeed~$\mathcal{T}'$ is a TD of~$\mathcal{G}_F$.
Further, $\Card{\chi'(t)}\leq k + k \cdot \ceil{\log({k})} \cdot 2 + k \cdot (\ell + 2)$. Thus, the width of TD~$\mathcal{T}'$ is in~$\mathcal{O}(k \cdot \log(k))$,
since for nice TDs we have~$\ell=2$.
\end{proof}

\FIXR{Note that the result above can be also lifted to non-nice TDs, as long as the number of child nodes is limited by~$\mathcal{O}(\log(k))$.}

Later we will see the lower bound for consistency of normal \ASP,
which indicates that one cannot expect to significantly improve this increase of treewidth. %
Next, we present consequences for auxiliary variables and runtime.

\begin{corollary}[Runtime]\label{thm:runtime}
The reduction from an HCF program~$\Pi$ and a nice TD~$\mathcal{T}$ 
of~$\mathcal{G}_\Pi$ to SAT formula~$F$ 
consisting of Formulas~(\ref{red:checkrules})--(\ref{red:checkfirst}) 
uses at most~$\mathcal{O}(k\cdot\log(k)\cdot h)$ many variables
and runs in time~$\mathcal{O}(k^2\cdot \log(k)\cdot h + \Card{\Pi})$,
where~$k$ and $h$ are the width and the number of nodes of~$\mathcal{T}$, respectively.
\end{corollary}
\begin{proof}
The result follows from Theorem~\ref{thm:runtime1}.
Linear time in the size of~$\Pi$ can be obtained 
by slightly modifying Formulas~(\ref{red:checkrules}) and~(\ref{red:checkfirst})
such that each rule~$r\in \Pi$ is used in only one node~$t$,
where~$r\in\Pi_{t'}$, but~$r\notin\Pi_t$, for some~$t'\in\children(t)$.
Runtime~$\mathcal{O}(k^2\log(k))$ in~$k$ is due to the definition of~$(x\prec_t y)$ as 
used in Formulas~(\ref{red:prop}) and~(\ref{red:checkfirst}).
\end{proof}
Note that a nice TD of~$\mathcal{G}_\Pi$ of width~$k=\tw{\mathcal{G}_\Pi}$, 
having only~$h=\mathcal{O}(\Card{\at(\Pi)})$ many nodes~\cite{Kloks94a}[Lem.\ 13.1.2] always exists. 
\FIXR{%
Further, since $k\cdot\log(k)$ might be much smaller than~$\log(\Card{\at(\Pi)})$,
for some programs this reduction could be a benefit %
compared to global or component-based
orderings used in tools like lp2sat~\cite{Janhunen06,%
GebserJanhunenRitanen14,BomansonEtAl16}.
An empirical study addressing this is given in Section~\ref{sec:emp}.}

\smallskip
\noindent\textbf{%
Correctness of the Reduction.}
\futuresketch{
Before we show correctness, we need to formally define the concepts of level mappings over a set of atoms.
Given a program~$\Pi$, and a set~$A\subseteq\at(\Pi)$ of atoms. Then, a function~$\varphi: A \rightarrow \{0,\ldots,\Card{A}\}$ is an \emph{ordering} for~$\Pi$ over~$A$.
\begin{definition}\label{def:provinglvmapping}
Given a level mapping~$\varphi$ for~$\Pi$ over~$\at(A)$.
Then, $\varphi$ is \emph{proving}
if there is an answer set~$M$ of~$\Pi$, called \emph{proven} answer set, such that $a\in M$ if and only if
there is a rule~$r\in \Pi$, called \emph{justifying rule for~$a$ of~$\varphi$}, with: (1) $a\in H_r$, (2) $M\cap (H_r\setminus\{a\}\cup B_r^-)=\emptyset$, and (3) for each~$b\in B_r^+$ we have~$\varphi(b)<\varphi(a)$.
\end{definition}

For each answer set of a program, there exists a level mapping.

\begin{proposition}[cf.,~\cite{Janhunen06}]\label{prop:mapping}
Given a program~$\Pi$. Then, for each answer set~$M$ of~$\Pi$ there is a level mapping~$\varphi$ over~$\at(\Pi)$ proving~$M$.
\end{proposition}}
Now, we discuss the correctness of our reduction, which 
establishes that~$\mathcal{T}$-local orderings
encoded by Formulas~(\ref{red:checkrules})--(\ref{red:checkfirst})   follow ideas of the characterization of answer sets for HCF programs.

\begin{theorem}[Correctness]\label{thm:corr1}
The reduction from an HCF program~$\Pi$ and a TD~$\mathcal{T}=(T,\chi)$ of~$\mathcal{G}_\Pi$ to SAT formula~$F$ consisting of Formulas~(\ref{red:checkrules})--(\ref{red:checkfirst}) is correct.
Precisely, for each answer set of~$\Pi$ there is a model of~$F$ and vice versa.
\end{theorem}
\begin{proof}
The proof is given in~\ref{sec:appendix1}.
\end{proof}

\FIXR{The statement above can be strengthened to derive the following corollary, which concludes that the reduction above only weakly preserves answer sets (and not bijectively).}

\begin{corollary}[Preservation of Answer Sets]\FIX{
The reduction from an HCF program~$\Pi$ and a TD~$\mathcal{T}=(T,\chi)$ of~$\mathcal{G}_\Pi$ to SAT formula~$F$ consisting of Formulas~(\ref{red:checkrules})--(\ref{red:checkfirst}) preserves answer sets with respect to~$\at(\Pi)$.
Concretely, for each answer set~$M$ of~$\Pi$ there is exactly one model of~$F$ that when restricted to the variables in~$\at(\Pi)$, coincides with~$M$. Conversely, for each model of~$F$ there is exactly one answer set of~$\Pi$.}
\end{corollary}

However, in general we have that for an answer set of~$\Pi$, 
there might be several models of the propositional formula obtained by the reduction above.
\FIXR{Overall, one can strenghten the reduction in order to remove some models, which is
presented in~\ref{sec:bijective}.}

\FIXR{
In the next subsection, we present another reduction from HCF \ASP to \SAT that 
bijectively preserves all the answer sets at least for uniquely provable programs,
at the cost of a higher increase of the treewidth from~$k$ to~$k^2$.
The difference to the reduction above is that the increase from~$k$ to~$k^2$
explicitly allows us to verify whether the (up to~$k^2$ many) relations per tree decomposition node in terms of provability 
are applicable when proving an answer set.%
}
\futuresketch{
However, if for each answer set~$M$ of~$\Pi$, and every~$a\in M$, there can be only one rule~$r\in \Pi$, where~$a\in H_r$ and~$M\cap (B_r^- \cup [H_r\setminus\{a\}])=\emptyset$,
then there is a bijective correspondence between answer sets of~$\Pi$ and models of Formulas~(\ref{red:checkrules})--(\ref{red:cnt:checkfirst}). 

\begin{theorem}[Condition for bijective reduction]
\FIX{For any given HCF program~$\Pi$, where for each answer set~$M$ of~$\Pi$
and~$a\in M$, it is guaranteed that there is only one rule~$r\in\Pi$ with~$a\in H_r$,
$B_r^+\subseteq M$, and~$M\cap (B_r^- \cup [H_r\setminus\{a\}])=\emptyset$.
Then, for the reduction from~$\Pi$ and a TD~$\mathcal{T}=(T,\chi)$ of~$\mathcal{G}_\Pi$ to SAT formula~$F$ consisting of Formulas~(\ref{red:checkrules})--(\ref{red:checkfirst}), the following holds: For each answer set of~$\Pi$ there is exactly one unique model of~$F$ and vice versa.
}
\end{theorem}
\begin{proof}

\end{proof}

\FIX{One example of such programs, is the program that is constructed in the reduction of the next section.}}

\FIX{\subsection{Bijective and Treewidth-Aware Reduction to \SAT}\label{sec:ext}

This section deals with a different approach that is inspired by an early attempt~\cite{LinZhao03}, combined with the ideas of Clark's completion~\cite{Clark77} that is guided along a tree decomposition.
To this end, consider an HCF program~$\prog$ and a TD~$\mathcal{T}=(T,\chi)$ of~$\mathcal{G}_\prog$.  
Then, instead of orderings, we use fresh auxiliary variables of the form~$(x\prec y)$ (and~$(y\prec x$)) to indicate that an atom~$x$ has to precede an other atom~$y$ (and vice versa). %
This is done locally for each tree decomposition bag and consequently results in a quadratic number of additional auxiliary variables per bag and causes a quadratic increase of the treewidth in the worst-case.
Further, we also require auxiliary variables~$p^x_t$, $p^{x\prec y}_t$, and $p^{y\prec x}_t$ for a node~$t$ of~$T$ and $x,y\in\chi(t)$ in order to indicate whether there is a rule proving atom~$x$, auxiliary variable~$(x\prec y)$, and auxiliary variable~$(y\prec x)$, respectively.
Finally, this approach also uses auxiliary variable~$p^x_{t,r}$ to indicate that an atom~$x\in\chi(t)$ is proven by a rule~$r\in\Pi_t$ in a node~$t$ of~$T$.

The reduction constructs a formula~$F'$ that consists of Formulas~(\ref{red2:checkrules}), (\ref{red2:checkremove})--(\ref{red2:check}) as well as Formulas~(\ref{red2:prove})--(\ref{red2:check2}), as given below, for each node~$t$ of~$T$. %
As before, Formulas~(\ref{red2:checkrules}) ensure that each rule~$r\in\prog$ is satisfied.
Then, Formulas~(\ref{red2:prove}) guarantee that we have provability~$p^x_{t,r}$ for an atom~$x\in\chi(t)$ using a rule~$r\in\prog_t$, if~$r$ proves~$x$ with any ordering, i.e., if~$x$ does not precede any atom of positive body~$B_r^+$.
Formulas~(\ref{red2:propsmaller}) make sure that if we have~$p^x_{t,r}$, indeed all positive body atoms~$b\in B_r^+$ precede~$x$,
i.e., whenever a rule is suitable for proving an atom, it actually has to be applied (greedy application).
We also ensure by Formulas~(\ref{red2:proptrans}) that the precedence is transitive.
Formulas~(\ref{red2:exclusion}) take care of not allowing cycles over~$\prec$, i.e., we cannot have both~$(x\prec y)$ and~$(y\prec x)$ at the same time.
}

\FIX{{
\vspace{-1.5em}
\begin{flalign}
	\label{red2:checkrules}&\bigvee_{b\in B_r^+} \neg b \vee \bigvee_{a\in B_r^- \cup H_r} a &&{\text{for each } r\in \Pi_t}\tag{\ref{red:checkrules}}\\
\label{red2:prove}&p_{t,r}^{x} \longleftrightarrow (\hspace{-0.15em}\bigwedge_{b\in B_r^+}\hspace{-0.5em}b \wedge x \wedge \neg (x \prec b) &&{\text{for each } r\in\Pi_t\text{ and }x\in H_r}\\[-.65em]
	&\qquad\qquad \wedge\hspace{-1.5em}\bigwedge_{a\in B_r^- \cup (H_r \setminus \{x\})}\hspace{-1.5em} \neg a)\notag\\
	\label{red2:propsmaller} &p_{t,r}^x \longrightarrow (b\prec x) && \text{for each }r\in\Pi_t, x\in H_r,\text{ and }b\in B_r^+\end{flalign}\begin{flalign}%
	 &(x \prec y) \wedge (y\prec z) \longrightarrow (x\prec z) && \text{for each }x,y,z\in\chi(t)\text{ with }\notag\\[-.8em]
\label{red2:proptrans}&&& x\neq y, x\neq z,\text{ and } y\neq z\\
\label{red2:exclusion}& \neg(x\prec y) \vee \neg(y\prec x) && \text{for each }x,y\in\chi(t)\text{ with }x\neq y\\
\label{red2:checkfirst}&p_{t}^x \longleftrightarrow \bigvee_{r\in\Pi_t, x \in H_r} p^x_{t,r} &&{\text{for each } x\in \chi(t)}\\
		&p_{t}^{y \prec x} \longleftrightarrow \hspace{-2.5em}\bigvee_{r\in\Pi_t, x\in H_r,y \in B_r^+} \hspace{-1.5em}p^x_{t,r}\quad \vee  &&{\text{for each } x,y\in \chi(t)\text{ with }}x\neq y\notag\\[-.5em]%
		\label{red2:checkfirst2}&\qquad\bigvee_{z\in\chi(t),z\neq x,z\neq y}\hspace{-1em}[(y\prec z) \wedge (z\prec x)]\hspace{-5em}\\
	\label{red2:checkremove}&x\longrightarrow p^{x}_{<t'}&&{\text{for each }{t'\in\children(t)},x\in\chi(t')\setminus\chi(t)}\tag{\ref{red:checkremove}}\\
	\label{red2:checkremove2}&x \longrightarrow p^{x}_{<n}&&{\text{for each }x\in\chi(n)\text{ with }n=\rootOf(T)}\tag{\ref{red:checkremove2}}\\
	\label{red2:check}&p^x_{<t} \longleftrightarrow p^x_t \vee (\hspace{-1em}\bigvee_{t' \in \children(t), x\in\chi(t')} \hspace{-2em} p^x_{<t'})&&{\text{for each } x\in \chi(t)}\tag{\ref{red:check}}\\
	&(x\prec y)\longrightarrow p^{x\prec y}_{<t'}&&{\text{for each }{t'\in\children(t)}\text{ and }x,y\in\chi(t')}\text{ with}\notag\\[-.8em]
	\label{red2:checkremove3}&&&x\neq y\text{ and }\{x,y\} \cap (\chi(t')\setminus\chi(t))\neq \emptyset\\
	&(x\prec y) \longrightarrow p^{x\prec y}_{<n}&&{\text{for each }x,y\in\chi(n)}\text{ with }x\neq y\text{ and }\notag\\[-.8em]
\label{red2:checkremove5}&&&n = \rootOf(T)\\
	\label{red2:check2}&p^{x\prec y}_{<t} \longleftrightarrow p^{x\prec y}_t \vee (\hspace{-2em}\bigvee_{t' \in \children(t), x,y\in\chi(t')} \hspace{-2em} p^{x\prec y}_{<t'})&&{\text{for each } x,y\in \chi(t)\text{ with }x\neq y} %
\end{flalign}
}
} 

\FIX{
\noindent We define~$p^x_t$ by Formulas~(\ref{red2:checkfirst}), similarly to Formulas~(\ref{red:checkfirst}), but here we are able to use auxiliary variables~$p^x_{t,r}$ for every~$r\in\prog_t$.
Further, we also define provability~$p^{y\prec x}_t$ for auxiliary variables~$(y\prec x)$ in node~$t$, which is the case if we either derive~$(y\prec x)$ due to Formulas~(\ref{red2:propsmaller}), where we also need~$p^x_{t,r}$, or we obtain~$(y\prec x)$ due to Formulas~(\ref{red2:proptrans}), where we require~$(y\prec z)$  and~$(z\prec x)$ for some~$z\in\chi(t)$.
For ensuring provability for an atom~$x\in\chi(t)$ and guiding it along the TD, we use Formulas~(\ref{red2:checkremove}), (\ref{red2:checkremove2}), and~(\ref{red2:check}) as before.
Analogously, we require Formulas~(\ref{red2:checkremove3}), (\ref{red2:checkremove5}), and~(\ref{red2:check2}) for guiding provability of an auxiliary variable~$(x\prec y)$ along the TD.

\begin{example}
Recall program~$\Pi$ from Example~\ref{ex:running1}, TD~$\mathcal{T}=(T,\chi)$ of $\mathcal{G}_\Pi$ given in Figure~\ref{fig:graph-td},
as well as Formulas~(\ref{red:checkrules}) and~(\ref{red2:checkremove})--(\ref{red2:check}) from Example~\ref{ex:reduction}.
Next, we briefly show Formulas~(\ref{red2:prove})--(\ref{red2:check2}) for node~$t_3$ of~$T$. 

\smallskip
\noindent\begin{tabular}{@{\hspace{0.15em}}l@{\hspace{0.15em}}|@{\hspace{0.15em}}l@{\hspace{0.0em}}}
Formulas & Formula $F'$\\
\hline
(\ref{red2:prove})& $p_{t_3, r_3}^d \leftrightarrow b \wedge \neg (d \prec b) \wedge \neg e$; $p_{t_3, r_4}^e \leftrightarrow b \wedge \neg (e \prec b) \wedge \neg d$; \\
& $p_{t_3, r_5}^b \leftrightarrow e \wedge \neg (b \prec e) \wedge \neg d$; $p_{t_3, r_6}^d \leftrightarrow \neg b$\\
(\ref{red2:propsmaller})& $p_{t_3, r_3}^d \rightarrow (b\prec d)$; $p_{t_3, r_4}^e \rightarrow (b\prec e)$; $p_{t_3, r_5}^b \rightarrow (e\prec b)$\\
(\ref{red2:proptrans})& $(e\prec b) \wedge (b\prec d) \rightarrow (e\prec d)$; $(d\prec b) \wedge (b\prec e) \rightarrow (d\prec e)$;\\
& $(d\prec e) \wedge (e\prec b) \rightarrow (d\prec b)$; $(b\prec d) \wedge (d\prec e) \rightarrow (b\prec e)$\\
(\ref{red2:exclusion})& $\neg(b\prec d) \vee \neg(d\prec b)$; $\neg(b\prec e) \vee \neg(e\prec b)$; $\neg(d\prec e) \vee \neg(e\prec d)$\\
(\ref{red2:checkfirst}) & $p^d_{t_3} \leftrightarrow p^d_{t_3,r_3} \vee p^d_{t_3,r_6}$; $p^e_{t_3} \leftrightarrow p^e_{t_3,r_4}$; $p^b_{t_3} \leftrightarrow p^b_{t_3,r_5}$\\
(\ref{red2:checkfirst2}) & $p^{b\prec d}_{t_3} \leftrightarrow p^d_{t_3,r_3}$; $p^{d\prec b}_{t_3} \leftrightarrow p^b_{t_3,r_5} \vee  [(d \prec e) \wedge (e\prec b)]$; \\
& $p^{e\prec b}_{t_3} \leftrightarrow p^b_{t_3,r_5}$;  $p^{b\prec e}_{t_3} \leftrightarrow p^e_{t_3,r_4} \vee [(b\prec d) \wedge (d\prec e)]$;\\
& $p^{d\prec e}_{t_3} \leftrightarrow p^e_{t_3,r_4} \vee [(d \prec b) \wedge (b\prec e)]$; $p^{e\prec d}_{t_3} \leftrightarrow p^d_{t_3,r_3} \vee [(e \prec b) \wedge (b \prec d)]$\\
(\ref{red2:checkremove3}) & $(c\prec d) \rightarrow p^{c\prec d}_{<t_1}$; $(d\prec c) \rightarrow p^{d\prec c}_{<t_1}$; $(c\prec e) \rightarrow p^{c\prec e}_{<t_1}$; $(e\prec c) \rightarrow p^{e\prec c}_{<t_1}$;\\ %
& $(a\prec b) \rightarrow p^{a\prec b}_{<t_2}$; $(b\prec a) \rightarrow p^{b\prec a}_{<t_2}$\\
(\ref{red2:checkremove5})& $(b\prec d) \rightarrow p^{b\prec d}_{<t_3}$; $(d\prec b) \rightarrow p^{d\prec b}_{<t_3}$; $(b\prec e) \rightarrow p^{b\prec e}_{<t_3}$; $(e\prec b) \rightarrow p^{e\prec b}_{<t_3}$;\\ %
& $(d\prec e) \rightarrow p^{d\prec e}_{<t_3}$; $(e\prec d) \rightarrow p^{e\prec d}_{<t_3}$ \\ %
(\ref{red2:check2})& $p_{<t_3}^{b\prec d} \leftrightarrow p_{t_3}^{b\prec d}$; $p_{<t_3}^{d\prec b} \leftrightarrow p_{t_3}^{d\prec b}$; $p_{<t_3}^{b\prec e} \leftrightarrow p_{t_3}^{b\prec e}$; $p_{<t_3}^{e\prec b} \leftrightarrow p_{t_3}^{e\prec b}$\\
& $p_{<t_3}^{d\prec e} \leftrightarrow p_{t_3}^{d\prec e} \vee  p_{<t_1}^{d\prec e}$; $p_{<t_3}^{e\prec d} \leftrightarrow p_{t_3}^{e\prec d} \vee  p_{<t_1}^{e\prec d}$\\
\end{tabular}%

\smallskip
\noindent Note that for practical implementations there is, of course, potential for optimizations.
To demonstrate this, the table above does not contain every useless instance of Formulas~(\ref{red2:proptrans}). As an example, $(b\prec e) \wedge (e\prec d) \rightarrow (b\prec d)$ is not needed since the only way to prove~$(e\prec d)$ is via the first instance of Formulas~(\ref{red2:proptrans}) in the table, which requires~$(e\prec b)$.
However, having both~$(b\prec e)$ and~$(e\prec b)$ is not possible anyway due to Formulas~(\ref{red2:exclusion}).
\end{example}}

\smallskip
\noindent\FIX{\textbf{%
Treewidth-Awareness and Runtime.}}
\FIX{Next, we discuss consequences of the reduction consisting of Formulas~(\ref{red2:checkrules}), (\ref{red2:checkremove})--(\ref{red2:check}) as well as Formulas~(\ref{red2:prove})--(\ref{red2:check2}) for each node~$t$ of~$T$.
Thereby, we show results for treewidth-awareness and runtime. %

\begin{theorem}[Treewidth-awareness]\label{thm:runtime2}
The reduction from an HCF program~$\Pi$ and a nice TD~$\mathcal{T}=(T,\chi)$ of~$\mathcal{G}_\Pi$ to \SAT formula~$F'$ consisting of Formulas~(\ref{red2:checkrules}), (\ref{red2:checkremove})--(\ref{red2:check}) as well as Formulas~(\ref{red2:prove})--(\ref{red2:check2}) ensure that %
if~$k$ is the width of $\mathcal{T}$,
then the treewidth of~$\mathcal{G}_{F'}$ is at most~$\mathcal{O}(k^2)$.
\end{theorem}

\begin{proof}
We construct a TD~$\mathcal{T}'=(T,\chi')$ of~$\mathcal{G}_{F'}$ to show that the 
width of~$\mathcal{T}'$ increases from~$k$ to~$\mathcal{O}(k^2)$. %
To this end, let~$t$ be a node of~$T$ with~$\children(t)=\langle t_1, \ldots, t_\ell \rangle$
and let~$\hat t$ be the parent node of~$t$ (if it exists).
Note that the number~$\Card{\prog_t}$ of rules might be larger than~$\mathcal{O}(k^2)$.
However, one can easily modify TD~$\mathcal{T}$ %
by adding intermediate nodes~$t^1, \ldots, t^o$ between~$t$ and~$\hat t$, where for each 
node~$t^i$ with~$1\leq i\leq o$, we have~$\chi(t_i)=\chi(t)$.
Then, instead of the actual bag program~$\prog_{t^i}=\prog_t$ we only apply 
a small subset~$\prog'_{t^i}\subseteq\prog_{t}$ of~$t^i$ such that~$\bigcup_{1\leq i\leq o}\prog'_{t^i}=\prog_t$.
So, intuitively instead of applying~$\prog_t$ in one node~$t$, this allows us to partition~$\prog_t$ and apply the parts separately by using intermediate nodes. 
Therefore, in the following we assume for the ease of notation and without loss of generality that~$\Card{\prog_t}$ is bounded by a constant~$c$,
i.e., we sloppily refer to~$\prog'_t$ by~$\prog_t$.
We inductively define~$\chi'(t)\eqdef \chi(t) \cup \{(x\prec y) \mid \{x,y\}\subseteq\chi(t), x\neq y\}\cup \{p^y_{<{t'}}, p^x_{{t}}\mid t'\in\{t,t_1,\ldots,t_\ell\} , x\in\chi(t), y\in\chi(t')\} \cup \{p^x_{t,r}\mid x\in\chi(t), r\in\prog_t\} \cup  \{p^{y\prec y'}_{<{t'}}, p^{x\prec x'}_{{t}}  
\mid t'\in\{t,t_1,\ldots,t_\ell\}, \{x,x'\}\subseteq\chi(t), \{y,y'\}\subseteq\chi(t'), x\neq x', y\neq y'\}$.
Observe that indeed~$\mathcal{T}'$ is a TD of~$\mathcal{G}_{F'}$.
Further, $\Card{\chi'(t)}\leq k + k^2 + 2k \cdot (\ell + 1) + ck + 2k^2\cdot(\ell+1)$. 
Thus, since~$\ell\leq 2$ for nice TDs, the width of~$\mathcal{T}'$ is in~$\mathcal{O}(k^2)$.
\end{proof}

This increase of treewidth is also reflected in the runtime of the reduction.

\begin{corollary}[Runtime]
The reduction from an HCF program~$\Pi$ and a nice TD~$\mathcal{T}$ 
of~$\mathcal{G}_\Pi$ to SAT formula~$F'$ consisting of Formulas~(\ref{red2:checkrules}), (\ref{red2:checkremove})--(\ref{red2:check}) as well as Formulas~(\ref{red2:prove})--(\ref{red2:check2}) 
uses at most~$\mathcal{O}(k^2\cdot h+\Card{\Pi})$ many variables
and runs in time~$\mathcal{O}(k^3\cdot h+\Card{\Pi})$,
where~$k$ and $h$ are the width and the number of nodes of~$\mathcal{T}$, respectively.
\end{corollary}
\begin{proof}
The result follows from Theorem~\ref{thm:runtime2}.
Linear time in the size of~$\Pi$ can be obtained 
by slightly modifying Formulas~(\ref{red2:checkrules}), (\ref{red2:prove}), (\ref{red2:propsmaller}),  (\ref{red2:checkfirst}), and (\ref{red2:checkfirst2})
such that each rule~$r\in \Pi$ is used in only one node~$t$,
where~$r\in\Pi_{t'}$, but~$r\notin\Pi_t$, for some~$t'\in\children(t)$.
The cubic runtime in~$k$ is due to transitivity by Formulas~(\ref{red2:proptrans}) and~(\ref{red2:checkfirst2}).
\end{proof}}

\FIXR{For detailed discussions on correctness and further consequences, we refer to~\ref{sec:appendix}.}

\medskip
\FIX{
Compared to the reduction of Section~\ref{sec:tdguided}, the worst-case increase of treewidth from~$k$ to~$k^2$
explicitly allows us to verify whether every single $\prec$-relation per tree decomposition node
is indeed applicable when proving an answer set.
However, this reduction does not bijectively preserve answer sets for HCF programs in general,
and we believe that a different approach is needed in order to design a reduction that is both
 treewidth-aware and bijective at the same time.

Next, we show that indeed already deciding the consistency of a normal (uniquely provable) program is expected to be slightly harder than deciding the satisfiability of a propositional formula.}

\futuresketch{
\begin{definition}
Given a level mapping~$\varphi$ for~$\Pi$ over~$A$.
Then, $\varphi$ is \emph{proving}
if there is an answer set~$M$, called \emph{proving} answer set, of~$\Pi$ such that $a\in M$ if and only if
there is a rule~$r\in \Pi$, called \emph{justifying rule for~$a$ of~$\varphi$}, with: (1) $a\in H_r$, (2) $M\cap (H_r\setminus\{a\}\cup B_r^-)=\emptyset$, and (3) for each~$b\in B_r^+$ we have~$\varphi(b)<\varphi(a)$.
\end{definition}

\begin{definition}
Given a level mapping~$\varphi$ for~$\Pi$ over~$A$.
Then, $\varphi$ is \emph{canonical} if there exists an answer set~$M$ of~$\Pi$ such that
(1) $a\in M$ if and only if $\varphi(a)>0$, (2) for every~$a\in A$ with~$\varphi(a) > 0$ there exists~$a'\in A$ with $\varphi(a') - \varphi(a) = 1$, and (3) for every~$a\in M$, and every justifying rule~$r\in\Pi$ for~$a$ of~$\varphi$, there is some~$b\in B_r^+$ with~$\varphi(b) - \varphi(a) = 1$.
\end{definition}

\begin{proposition}\label{prop:mappingunique}
Given a program~$\Pi$. Then, for each proving, canonical level mapping~$\varphi$ for~$\Pi$ over~$\at(\Pi)$ there is one unique proving answer set~$M$ of~$\Pi$ for~$\varphi$ and vice versa.
\end{proposition}
\begin{proof}
$\Rightarrow$: Assume towards a contradiction that there are two different answer sets~$M,M'$ of~$\Pi$ for~$\varphi$. Since~$M, M'$ are both subset-minimal,
there is some~$a\in M$, such that $a\notin M'$. 
Then, by canonicity, $\varphi(a)=0$ since $a\notin M'$, but~$\varphi(a)\neq 0$ due to~$M$.

$\Leftarrow$: Assume that for an answer set~$M$ there are two canonical level mappings~$\varphi,\varphi'$ such that~$M$ is proving for~$\varphi,\varphi'$.
By (1) of canonicity, for~$a\in M$, we have $\varphi(a)>0$.
Assume towards a contradiction that there is some~$a\in M$ with~$\varphi(a) < \varphi'(a)$.
But then $\varphi'$ cannot be canonical, as it violates (3) of the canonicity definition, as at least one justifying rule exists since~$\varphi$ is proving.
\end{proof}

Given a TD~$\mathcal{T}=(T,\chi)$ of~$\mathcal{G}_\Pi$, and a node~$t$ of~$T$.
We refer to a canonical level mapping of~$\Pi_t$ over~$\chi(t)$ by \emph{$t$-local level mapping~$\varphi_t$}. 
Then, we refer by \emph{$\mathcal{T}$-local level mappings} to the set~$\mathcal{M}$ consisting of one~$t$-local level mapping~$\varphi_t$ for each~$t$ of~$T$
such that we have \emph{compatibility} as follows: For each nodes~$t,t'$ of~$T$ and every~$a,b\in\chi(t)\cap\chi(t')$, we have (1) $\varphi_t(a) = 0$ if and only if $\varphi_{t'}(a)=0$, (2) whenever $\varphi_t(a) < \varphi_t(b)$  then~$\varphi_{t'}(a) < \varphi_{t'}(b)$, and (3) for each justifying rule~$r\in\Pi_{t}$ for~$a$ respecting~$\varphi_{t}$, we require that~$\varphi_t$ gives the maximum level for~$a$ in~$T$, i.e., $\varphi_t(a)=\max_{t''\text{ of } T, a\in\chi(t'')}(\varphi_{t''}(a))$.

Further, $\mathcal{M}$ requires \emph{provability}, where for each $a\in\at(\Pi)$, there exists a node~$t$ of~$T$ s.t.\ either~$\varphi_t(a)=0$, or there is a justifying rule~$r\in\Pi_{t}$ for~$a$ respecting~$\varphi_{t}$.
\begin{lemma}[Equivalence of level mappings]
Given a program~$\Pi$, and a TD~$\mathcal{T}=(T,\chi)$ of~$\mathcal{G}_\Pi$.
Then, %
there is a bijective correspondence between a set~$\mathcal{M}$ of~$
\mathcal{T}$-local level mappings of~$\varphi$ and a proving, minimal level mapping~$\varphi$.
Concretely, for each such set~$\mathcal{M}$ there is a unique mapping~$\varphi$ (and vice versa) assuming that for each~$a\in\at(\Pi)$, $\varphi(a) = 0$ if and only if there is $\varphi_t\in\mathcal{M}$ with~$\varphi_t(a) = 0$.
\end{lemma}
\begin{proof}
$\Rightarrow$: Given the set~$\mathcal{M}$ of~$\mathcal{T}$-local level mappings. 
Then, we construct a proving, canonical level mapping~$\varphi$ as follows.
We set~$\varphi(a)\eqdef 0$ for each~$a\in \at(\Pi)$, where
there exists a node~$t$ of~$T$ with~$\varphi_t(a)=0$.
Then, we set~$\varphi(a)\eqdef 1$ for each~$a\in \at(\Pi)$,
where there is no node~$t$ of~$T$ with~$\varphi_t(b) < \varphi_t(a)$ for some~$b$
with~$\varphi(b)=0$, and so on.
In turn, we construct~$\varphi$ in rounds, where each round assigns an increasing value of~$2,3,\ldots$.
Observe that $\varphi$ is well-defined, i.e., each atom~$a\in\at(\Pi)$ gets a unique value since~$\mathcal{M}$ only contains canonical mappings and by compatibility of~$\mathcal{M}$.

Next, we show that~$\varphi$ is a proving, canonical level mapping of~$\Pi$.
Indeed, $\varphi$ is canonical by canonicity of each~$t$-local level mapping within~$\mathcal{M}$ and compatibility of~$\mathcal{M}$.
In particular, by (3) of compatibility of~$\mathcal{M}$,
for each atom~$a$, every nodes~$t',t''$ of~$T$,
and every justifying rule for~$a$ respecting~$\varphi_{t'}, \varphi_{t''}$, respectively, we have~$\varphi_{t'}(a)=\varphi_{t''}(a)$.
Further, $\varphi$ is proving by provability of~$\mathcal{M}$ and since the properties of TDs ensure that for TD~$\mathcal{T}$ for each rule~$r\in\Pi$ there is a node~$t$ of~$T$ with~$r\in\Pi_t$.

Towards a contradiction assume that there is a proving, canonical level mapping~$\varphi'$ 
for~$\Pi$ that is incomparable to~$\varphi$ with~$\varphi(a)=0$ if and only if~$\varphi'(a)=0$ for each~$a\in\Pi_t$.
However, both~$\varphi,\varphi'$ lead by construction to the same proving answer set~$M$. 
Consequently, by Proposition~\ref{prop:mappingunique}, $\varphi=\varphi'$, which leads to a contradiction.

$\Leftarrow$:
Assuming proving, canonical level mapping $\varphi$, and the proving answer set~$M$ of~$\varphi$.
Then, we define the \emph{canonical $t$-local level mapping~$\hat\varphi_t$} for each~$t$ of~$T$ as follows.
\noindent$\hat\varphi_t(a) {\eqdef} \begin{cases}0 & \text{for each } a\in \chi(t) \text{ with } \varphi(a) = 0,\\
\rank_{\chi(t)}(a,\varphi) &\text{for each } a\in \chi(t) \text{ with } \varphi(a) \neq 0,\end{cases}$
where~$\rank_A(a,\varphi)$ is the ordinal number of $a$ according to~$\varphi(a)$ among all elements in~$A$.
The resulting set~$\mathcal{M}$ of~$\mathcal{T}$-local level mappings consisting of~$\hat\varphi_{t}$ for each~$t$ of~$T$ is well-defined.
Observe that any~$\hat\varphi_t$ is canonical, as for~$M$, (1) and (3) of canonicity remains satisfied since the relation between level numbers remains as in~$\varphi$ (by construction of~$\hat\varphi_t$). Also Condition (2) is satisfied by construction of~$\hat\varphi_t$, since the absolute numbers are within~$1,\ldots, \Card{\chi(t)}$ and~$\varphi$ is canonical. Further, compatibility of~$\mathcal{M}$ holds, as well as provability, as by properties of a TD any rule~$r\in\Pi$ is in some node~$t$, i.e., $r\in\Pi_t$.

Next, we show that the set~$\mathcal{M}$ of~$\mathcal{T}$-local level mappings consisting of~$\hat\varphi_{t}$ for each~$t$ of~$T$ is unique.
Assume towards a contradiction that there is an alternative set~$\mathcal{M}'$ of~$\mathcal{T}$-local level mappings~$\psi_t$ for each~$t$ of~$T$. 
Then, by compatibility of~$\mathcal{M}'$ and assumptions of this lemma, $\psi_t(a)=0$ if and only if~$\hat\varphi_t(a)=0$.
Consequently, there must be an atom~$a\in\at(\Pi_t)$ with~$\psi_t(a) \neq \varphi_t(a)$ such that~$\psi_t(a) > 0$ and~$\varphi_t(a) > 0$.
We distinguish the cases (a) $\psi_t(a) > \varphi_t(a)$ and (b) $\varphi_t(a) > \psi_t(a)$. 

\noindent Case (a):~$\psi_t(a) > \varphi_t(a)$. if~$\psi_t(a) > \max_{t'\text{ of }T, a\in\chi(t')}\psi_{t'}(a)$,
then (3) of compatibility of~$\mathcal{M}'$ is dissatisfied.
Then, there is by canonicity~$b\in\chi_t$ where~$\psi(a)>\psi(b)$, but~$\varphi_t(a)\leq\varphi_t(b)$.
There cannot be a node~$t'$ of~$T$, with~$t\neq t'$ where~$b\in\chi_{t'}$,
otherwise~$\mathcal{M}'$ does not have compatibility. Consequently, since~$b$ was arbitrarily chosen, this has to hold for every such atom~$b$.
Therefore

\noindent Case (b):~$\varphi_t(a) > \psi_t(a)$. Then, there cannot be a rule~$r\in \Pi$
\end{proof}

\begin{theorem}
The reduction~$R$ from an HCF program~$\Pi$ and a TD~$\mathcal{T}=(T,\chi)$ of~$\mathcal{G}_\Pi$ to SAT formula~$F$ consisting of Formulas~(\ref{red:checkrules})--(\ref{red:cnt:checkfirst}) is correct.
Concretely, for each answer set of~$\Pi$ there exists exactly one model of~$F$ and vice versa. TODO: make it precise!
\end{theorem}

\begin{proof}
Assume any answer set~$M$ of~$\Pi$. 
Then, by Proposition~\ref{prop:mappingunique} there has to be a unique, minimal level mapping~$\varphi: \at(\Pi) \rightarrow \Nat$, proven by~$M$.
We construct a model~$I$ of~$F$ as follows.
For each~$x\in\at(\Pi)$ we let $I(x)\eqdef 1$ if $x\in M$ and~$I(x)\eqdef 0$ otherwise.
For each node~$t$ of~$T$, %
and~$x\in\chi(t)$ we assign the following variables:
(1) For every~$l\in \bvali{x}{t}{i}$ where~$i=\rank_{\chi(t)}(x,\varphi)$, we set $I(b)\eqdef 1$ if~$l$ is an atom and~$I(b)\eqdef 0$ if~$l$ is not an atom. 
(2) If there is a justifying rule~$r$ for~$x$ respecting~$\varphi$, we set~$I(p_{x_t})\eqdef I(pf_{x_t})\eqdef 1$.
(3) If~$I(p_{x_{t'}})=1$ for~$t'\in\children(t)$, then we set~$I(p_{x_t})\eqdef 1$.
\end{proof}

\subsubsection{Parameterized Algorithm for Counting and Enumeration.}
The reduction presented in the previous section immediately gives rise to an algorithm for counting and enumerating answer sets of an \ASP program.

\begin{proposition}
Runtime for counting and enumerating with linear delay.
\end{proposition}

\begin{corollary}
Counting by reduction
\end{corollary}

\begin{corollary}
Enumeration by reduction, mention linear delay
\end{corollary}
}

\section{Why \ASP Consistency is Harder than \SAT}\label{sec:hardness}

This section concerns the hardness of ASP consistency when
considering treewidth.
The high-level reason for \ASP being harder than \SAT when
assuming bounded treewidth, lies in the issue
that a TD, while capturing the structural dependencies
of a program, might force an evaluation that is
completely different from the orderings proving answer sets. %
Consequently, during dynamic programming for \ASP, one needs to store in each table~$\tab{t}$ for each node~$t$ during post-order traversal, in addition to an interpretation (candidate answer set), also an ordering among the atoms in those interpretations.
We show %
that under reasonable assumptions in complexity theory,
this worst-case cannot be avoided. Then, the resulting runtime consequences cause \ASP to be slightly harder than \SAT, where 
in contrast to \ASP
storing a table~$\tab{t}$ of only assignments for each node~$t$~suffices.

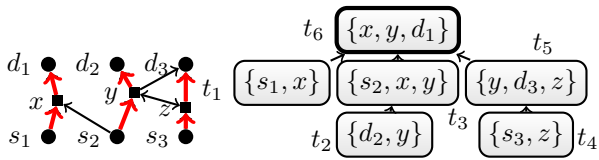
\begin{figure}
	\centering%
	\begin{tikzpicture}[node distance=7mm,every node/.style={fill,circle,inner sep=2pt}]%
		\node (s1) [label={[text height=1.5ex,yshift=0.0cm,xshift=0.05cm]left:$s_1$}] {};
		\node (s2) [right=of s1,label={[text height=1.5ex,yshift=0.0cm,xshift=0.05cm]left:$s_2$}] {};
		\node (s3) [right=of s2,label={[text height=1.5ex,yshift=0.0cm,xshift=0.05cm]left:$s_3$}] {};
		\node (d1) [above of=s1,yshift=.75em,label={[text height=1.5ex,xshift=0.05cm]left:$d_1$}] {};
		\node (d2) [above of=s2,yshift=.75em,label={[text height=1.5ex,xshift=0.05cm]left:$d_2$}] {};
		\node (d3) [above of=s3,yshift=.75em,label={[text height=1.5ex,xshift=0.05cm]left:$d_3$}] {};
		\node (x1) [rectangle,above=of s1,yshift=-1.15em,xshift=.35em,label={[text height=1.5ex,yshift=0.0cm,xshift=0.05cm]left:$x$}] {};
		\node (x2) [rectangle,above=of s2,yshift=-.85em,xshift=.65em,label={[text height=1.5ex,yshift=0.0cm,xshift=0.05cm]left:$y$}] {};
		\node (x4) [rectangle,above=of s3,yshift=-1.4em,label={[text height=1.5ex,yshift=0.0cm,xshift=0.05cm]left:$z$}] {};
		\draw [->,thick] (s1) to (x1);
		\draw [->,ultra thick, red] (s1) to (x1);
		\draw [->,thick] (s2) to (x1);
		\draw [->,thick] (x1) to (d1);
		\draw [->,ultra thick,red] (x1) to (d1);
		\draw [->,thick] (x2) to (d2);
		\draw [->,ultra thick, red] (x2) to (d2);
		\draw [->,thick] (s2) to (x2);
		\draw [->,ultra thick, red] (s2) to (x2);
		\draw [->,thick] (s3) to (x4);
		\draw [->,ultra thick, red] (s3) to (x4);
		\draw [->,thick] (x4) to (d3);
		\draw [->,ultra thick, red] (x4) to (d3);
		\draw [->,thick] (x2) to (d3);
		\draw [<->,thick] (x4) to (x2);
	\end{tikzpicture}%
	\begin{tikzpicture}[node distance=0.75mm]
\tikzset{every path/.style=thick}

\node (leaf1) [tdnode,label={[yshift=-0.25em,xshift=0.0em]left:$t_1$}] {$\{s_1,x\}$};
\node (leaf1b) [tdnode, right=of leaf1,label={[yshift=0.25em,xshift=0.65em]below right:$t_3$}] {$\{s_2,x,y\}$};
\node (leaf1c) [tdnode,below=of leaf1b,xshift=-.5em,label={[yshift=-0.25em,xshift=0.25em]left:$t_2$}] {$\{d_2,y\}$};
\node (leaf2) [tdnode,label={[xshift=-1.0em, yshift=-0.15em]above right:$t_5$}, right = of leaf1b]  {$\{y,d_3,z\}$};
\node (leaf2b) [tdnode,label={[xshift=-.25em, yshift=-0.15em] right:$t_4$}, below = of leaf2]  {$\{s_3,z\}$};
\coordinate (middle) at ($ (leaf1.north east)!.5!(leaf2.north west) $);
\node (join) [tdnode,ultra thick,label={[]left:$t_6$}, above  = .75mm of middle] {$\{x,y,d_1\}$};

\draw [<-] (join) to (leaf1);
\draw [<-] (join) to (leaf1b);
\draw [<-] (join) to (leaf2);
\draw [<-] (leaf1b) to (leaf1c);
\draw [<-] (leaf2) to (leaf2b);
\end{tikzpicture}
	\vspace{-.45em}
	\caption{An instance $I=(G,P)$ (left) of the \problemFont{Disjoint Paths Problem} and a TD of~$G$ (right).}
	\label{fig:disjpaths}
\end{figure}

We show our novel hardness result by reducing from the \problemFont{(directed) Disjoint Paths Problem}, which is a graph problem defined as follows.
Let us consider a directed graph~$G=(V,E)$, and a set~$P\subseteq V\times V$ 
of disjoint pairs of the form $(s_i, d_i)$ consisting of 
\emph{source}~$s_i$ and \emph{destination}~$d_i$,
where~$s_i, d_i\in V$ such that each vertex occurs at most once in~$P$, 
i.e., $\Card{\bigcup_{(s_i,d_i)\in P}\{s_i,d_i\}}=2\cdot\Card{P}$.
Then, $(G,P)$ %
is an instance of the \problemFont{Disjoint Paths Problem},
asking whether there exist $\Card{P}$ many (vertex-disjoint) paths
from~$s_i$ to~$d_i$ for~$1\leq i\leq \Card{P}$.
Concretely, each vertex of~$G$ is allowed to appear in at most one of these paths.
For the ease of presentation, we assume without loss of generality~\cite{LokshtanovMarxSaurabh11}
that sources~$s_i$ have no incoming edge $(x,s_i)$,
and destinations~$d_i$ have no outgoing edge~$(d_i,x)$.

\begin{example}
Figure~\ref{fig:disjpaths} (left) shows an instance~$I=(G,P)$ of the \problemFont{Disjoint Paths Problem},
where $P$ consists of pairs of the form $(s_i,d_i)$.
The only solution to~$I$ is both emphasized and colored in red.
Figure~\ref{fig:disjpaths} (right) depicts a TD of~$G$.
\end{example}

While under ETH, \SAT cannot be solved in time~$2^{o(k)}\cdot\poly(\Card{\var(F)})$,
where~$k$ is the treewidth of the primal graph of a given propositional formula~$F$,
the  \problemFont{Disjoint Paths Problem} is considered to be even harder.
Concretely, the problem has been shown to be slightly superexponential
as stated in the following proposition.
\begin{proposition}[\cite{LokshtanovMarxSaurabh11}]\label{prop:slightlysuper}
Under ETH, there is no algorithm solving the~\problemFont{Disjoint Paths Problem} %
in time~$2^{o(k\cdot\log(k))}\cdot \poly({\Card{V}})$, where $(G,P)$ is any instance with $k=\tw{G}$.
\end{proposition}

It turns out that the %
\problemFont{Disjoint Paths Problem}
is a suitable problem candidate for showing the hardness of \ASP.
In our reduction, we use the following notation of open pairs, %
which leads to the result below. %
Let~$(G,P)$ be an instance of the~\problemFont{Disjoint Paths Problem}, $\mathcal{T}=(T,\chi)$ be a TD of~$G$, and~$t$ be a node of~$T$.
Then, a pair~$(s,d)\in P$ is \emph{open in node~$t$}, if either~$s\in \chi_{\leq t}$ (``\emph{open due to source $s$'}') or~$d\in\chi_{\leq t}$ (``\emph{open due to destination $d$}''), but not both.

\begin{proposition}[\cite{Scheffler94}]\label{prop:earlyout}
An instance~$(G,P)$ of the~\problemFont{Disjoint Paths Problem} does not have a solution if there is a TD~$\mathcal{T}=(T,\chi)$ of~$G$ and a bag~$\chi(t)$ with more than~$\Card{\chi(t)}$ many pairs in~$P$ that are open in a node~$t$ of~$T$.
\end{proposition}
\begin{proof}
The result, cf.,~\cite{Scheffler94}, boils down to the fact that each bag~$\chi(t)$, when removed from~$G$, results in a disconnected graph consisting of two components.
Between these components can be at most~$\Card{\chi(t)}$ different paths. %
\end{proof}

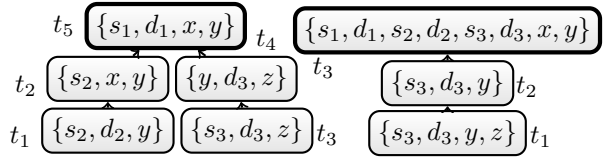
\begin{figure}
\centering%
	\begin{tikzpicture}[node distance=0.75mm]
\tikzset{every path/.style=thick}

\node (leaf1b) [tdnode,label={[yshift=-0.25em,xshift=0.0em]left:$t_2$}] {$\{s_2,x,y\}$};
\node (leaf1c) [tdnode,below=of leaf1b,label={[yshift=-0.25em,xshift=0.0em]left:$t_1$}] {$\{s_2,d_2,y\}$};
\node (leaf2) [tdnode,label={[xshift=-.5em, yshift=-0.15em]above right:$t_4$}, right = of leaf1b]  {$\{y,d_3,z\}$};
\node (leaf2b) [tdnode,xshift=.4em,label={[xshift=-.25em, yshift=-0.15em] right:$t_3$}, below = of leaf2]  {$\{s_3,d_3,z\}$};
\coordinate (middle) at ($ (leaf1.north east)!.5!(leaf2.north west) $);
\node (join) [tdnode,ultra thick,label={[]left:$t_5$}, above  = .75mm of middle] {$\{s_1,d_1,x,y\}$};

\draw [<-] (join) to (leaf1);
\draw [<-] (join) to (leaf2);
\draw [<-] (leaf2) to (leaf2b);
\draw [<-] (leaf1b) to (leaf1c);
\end{tikzpicture}%
\hspace{-4em}\begin{tikzpicture}[node distance=0.5mm]%
\tikzset{every path/.style=thick}

\node (abstand) [white] {};
\node (leaf2) [tdnode,label={[xshift=-.25em, yshift=-0.15em] right:$t_1$}, right = 4.1em of abstand]  {$\{s_3,d_3,y,z\}$};
\node (leaf3) [tdnode,above=of leaf2, label={[xshift=-.30em, yshift=-0.15em] right:$t_2$}]  {$\{s_3,d_3,y\}$};
\coordinate (middle) at ($ (leaf3.north east)!.5!(leaf3.north west) $);
\node (join) [tdnode,ultra thick,label={[xshift=-3em,yshift=.25em]below left:$t_3$}, above  = .75mm of middle] {$\{s_1,d_1,s_2,d_2,s_3,d_3,x,y\}$};

\draw [<-] (join) to (leaf3);
\draw [<-] (leaf3) to (leaf2);
\end{tikzpicture}
	\vspace{-.45em}
	\caption{A pair-respecting TD (left), and a pair-connected TD~$\mathcal{T}$ (right) of $(G,P)$ of Figure~\ref{fig:disjpaths}.}
	\label{fig:tds}
\end{figure}

\noindent%
\textbf{Preparing pair-connected TDs.}
Before we present the actual reduction, we need to define a \emph{pair-respecting} tree decomposition of an instance~$(G,P)$ of the~\problemFont{Disjoint Paths Problem}.
Intuitively, such a TD of~$G$ additionally ensures 
that each pair in~$P$ is encountered together in some TD bag.

\begin{definition}
A TD~$\mathcal{T}=(T,\chi)$ of~$G$ is a
\emph{pair-respecting TD} of~$(G,P)$ if for every pair~$p=(s,d)$ with~$p\in P$, (1) whenever~$p$
is open in a node~$t$ due to~$s$, or due to~$d$, then~$s\in\chi(t)$, or $d\in\chi(t)$, respectively.
Further, (2) whenever $p$ is open in a node~$t$, 
but not open in the parent~$t'$ of $t$ (``$p$ is \emph{closed in~$t'$}''), both~$s,d\in\chi(t')$. %
\end{definition}

We observe that such a pair-respecting TD can be computed with
only a linear increase in the (tree)width in the worst-case.
Concretely, we can turn any TD~$\mathcal{T}=(T,\chi)$ of~$G$ into
a pair-respecting TD~$\mathcal{T}'=(T,\chi')$ of~$(G,P)$.
Thereby, the tree~$T$ is traversed for each~$t$ of~$T$ in post-order, and vertices of~$P$ are added to~$\chi(t)$
accordingly, resulting in~$\chi'(t)$, such that conditions (1) and (2) of pair-respecting TDs are met.
Observe that this %
doubles the sizes of the bags in the worst-case,
since by Proposition~\ref{prop:earlyout}
there can be at most bag-size many open pairs.

\begin{example}
Figure~\ref{fig:tds} (left) shows a pair-respecting TD of~$(G,P)$ of Figure~\ref{fig:disjpaths},
which can be obtained by transforming the TD of Figure~\ref{fig:disjpaths} (right), followed by simplifications.
\end{example}

For a given sequence~$\sigma$ %
of
pairs of~$P$ in the order of closure
with respect to the post-order of~$T$,
we refer to~$\sigma$ by the \emph{closure sequence} of~$\mathcal{T}$.
We denote by~$p\in_i\sigma$ that pair~$p$ is the \emph{pair %
closed $i$-th} in the order of~$\sigma$.
Intuitively, e.g., the first pair~$p \in_1\sigma$
indicates that pair $p\in P$ %
is the first 
to be closed %
when traversing~${T}$ in post-order.

\begin{definition}
A \emph{pair-connected TD}~$\mathcal{T}{=}(T,\chi)$ of $(G,P)$ is a 
pair-respecting TD of $(G,P)$, if,
whenever a pair~$p\in_i\sigma$ with~$i{>}1$
is closed in a node~$t$ of~$T$,
also for the pair $(s,d)\in_{i-1}\sigma$
closed directly %
before $p$ in $\sigma$, both~$s,d\in\chi(t)$.
\end{definition}

We can turn any pair-respecting, \emph{nice} TD~$\mathcal{T}'{=}(T,\chi')$ of width~$k$ into a pair-connected TD~$\mathcal{T}''{=}(T,\chi'')$ with constant increase in the width.
Let therefore pair~$p\in_i\sigma$ be closed ($i{>}1$) in a node~$t$,
and pair~$(s,d)\in_{i-1}$ be closed before~$p$ in node~$t'$.
Intuitively, we need to add~$s,d$ to all bags~$\chi'(t'), \ldots, \chi'(t)$ of nodes encountered
after node~$t'$ and before node~$t$ of
the post-order tree traversal, resulting in~$\chi''$.
However, %
the width of~$\mathcal{T}''$ %
is at most~$k+3\cdot \Card{\{s,d\}} = k+6$,
since in the tree traversal each node of~$T$ is passed at most $3$ times, 
namely when traversing down, when going from the left branch to the right branch, 
and then also when going upwards.
Indeed, to ensure~$\mathcal{T}''$ is a TD (connectedness condition),
we add at most $6$ additional atoms to every bag.

\begin{example}
Figure~\ref{fig:tds} (right) depicts a pair-connected TD of~$(G,P)$ of Figure~\ref{fig:disjpaths},
obtainable by transforming the pair-respecting TD of Figure~\ref{fig:tds} (left), followed by simplifications.
\end{example}

\subsection*{Reducing from \problemFont{Disjoint Paths} to \ASP}

In this section, we show the main reduction~$R$ of this work,
assuming any instance~$I=(G,P)$ of the \problemFont{Disjoint Paths Problem}.
Before we construct our program~$\Pi$, we require a  nice,
pair-connected TD~$\mathcal{T}=(T,\chi)$ of~$G$,
whose width is~$k$ and a corresponding closure sequence~$\sigma$.
By Proposition~\ref{prop:earlyout}, for each node~$t$ of~$\mathcal{T}$,
there can be at most~$k$ many open pairs of~$P$, which we assume in the following.
If this was not the case, we can immediately output, %
e.g., $\{a\leftarrow \neg a\}$.

Then, we use the following atoms in our reduction.
Atoms $e_{u,v}$, or $ne_{u,v}$ indicate that edge~$(u,v)\in E$ is used, or unused,
respectively.
Further, atom $r_u$ for any vertex~$u\in V$ indicates that~$u$ is reached via used edges. %
Finally, we also need atom~$f^u_t$ for a node~$t$ of~$T$, and vertex~$u\in\chi(t)$,
to indicate that vertex~$u$ is already finished in node~$t$,
i.e., $u$ has one used, outgoing edge.
The presence of this atom~$f^u_t$ in an answer set prohibits to take additional 
edges of~$u$ in parent nodes of~$t$, which is needed due to the need of disjoint paths
of the~\problemFont{Disjoint Paths Problem}.

The instance~$\Pi=R(I,\mathcal{T})$ constructed by reduction~$R$ consists of three program parts,
namely \emph{reachability}~$\Pi_\mathcal{R}$, %
\emph{linking}~$\Pi_\mathcal{L}$ of two pairs in~$P$, as well as \emph{checking}~$\Pi_\mathcal{C}$ of disjointness of constructed paths.
Consequently, $\Pi=\Pi_\mathcal{R}\cup\Pi_\mathcal{L}\cup\Pi_\mathcal{C}$.
All three programs~$\Pi_{\mathcal{R}}$, $\Pi_\mathcal{L}$, and~$\Pi_\mathcal{C}$ are guided along TD~$\mathcal{T}$,
which ensures that the width of~$\Pi$ is only linearly increased.
Note that this has to be carried out carefully, since, e.g., the number of atoms of the form~$e_{u,v}$ using 
only vertices~$u,v$ that appear in one bag, can be already quadratic in the bag size.
The goal of this reduction, however, admits only a linear overhead in the bag size.
Consequently, we are, e.g., not allowed to construct rules in~$\Pi$ that require more than~$\mathcal{O}(k)$ edges in one bag of a TD of~$\mathcal{G}_\Pi$.

To this end, %
let the \emph{ready edges~$E^{\text{re}}_{t}$  in node~$t$}
be the set of edges~$(u,v)\in E$  %
not present in~$t$ anymore, i.e., 
$\{u,v\}\subseteq \chi(t')\setminus \chi(t)$ 
for any child node~$t'\in\children(t)$. %
Further, let $E^{\text{re}}_{n}$ for the root node~$n=\rootOf(T)$ %
additionally contain also all edges of~$n$, i.e., $E\cap (\chi(n) \times \chi(n))$.
Intuitively, ready edges for~$t$ will be processed in node~$t$.
Note that each edge occurs in exactly one set of ready edges.
Further, for nice TDs~$\mathcal{T}$, we always have~$\Card{E^{\text{re}}_{t}}\leq k$, i.e., 
ready edges are linear in~$k$.
\begin{example}
Recall instance~$I{=}(G,P)$ with $G{=}(V,E)$ of Figure~\ref{fig:disjpaths}, and pair-connected TD~$\mathcal{T}{=}(T,\chi)$ of~$I$ of Figure~\ref{fig:tds} (right).
Then, $E_{t_1}^{\text{re}}{=}\emptyset$,
$E_{t_2}^{\text{re}}{=}\{(y,z), (z,\allowbreak y), (z, d_3), (s_3, z)\}$, since~$z\notin\chi(t_2)$, and
$E_{t_3}^{\text{re}}{=}E\setminus E_{t_2}^{\text{re}}$ for root~$t_3$ of~$\mathcal{T}$.
\end{example}

\smallskip
\FIX{\noindent\textbf{Reachability~$\Pi_\mathcal{R}$.} 
Program~$\Pi_\mathcal{R}$ is constructed as follows.
{
\vspace{-1.5em}
\begin{align}
	\label{red:edgeguess1}&e_{u,v}\leftarrow r_u, \neg ne_{u,v}\hspace{-.4em}&&{\text{for each }(u,v)\in E^{\text{re}}_t}\\
	\label{red:edgeguess2}&ne_{u,v}\leftarrow \neg e_{u,v}&&{\text{for each }(u,v)\in E^{\text{re}}_t}\\
	\label{red:reach}&r_{v}\leftarrow e_{u,v}&&{\text{for each }(u,v)\in E^{\text{re}}_t}%
\end{align}%
}}

\vspace{-1em}
\noindent Rules~(\ref{red:edgeguess1}) and~(\ref{red:edgeguess2}) ensure that
there is a partition of edges in used edges~$e_{u,v}$ and unused edges~$ne_{u,v}$.
Additionally, Rules~(\ref{red:edgeguess1}) make sure that only edges of adjacent, reachable vertices are used.
Naturally, this requires that initially at least one vertex is reachable (constructed below).
Rules~(\ref{red:reach}) %
ensure reachability~$r_v$ %
over used edges~$e_{u,v}$ for %
a vertex~$v$. %

\medskip
\noindent%
\FIX{\textbf{Linking of pairs~$\Pi_\mathcal{L}$.} 
Program~$\Pi_\mathcal{L}$ is constructed as follows.\hspace{-1em}
{
\vspace{-.35em}
\begin{align}
	\label{red:pair1}&\hspace{-.5em}\leftarrow \neg r_{d}&&{\text{for each }(s,d)\in P}\\
	\label{red:pair1st}&\hspace{-.5em}r_{s_1}\leftarrow &&{\text{for }(s_1,d)\in_1 \sigma}\\
	\label{red:paircycles}&\hspace{-.5em}r_{s_i}\leftarrow r_{s_{i-1}}, r_{d_{i-1}}&&{\text{for each }(s_i,d)\in_i\hspace{-.1em} \sigma, (s,d_{i-1})\hspace{-.1em}\in_{i-1}\hspace{-.1em}\sigma} %
\end{align}%
\vspace{-1em}
}}

\noindent Rules~(\ref{red:pair1}) %
make sure that, ultimately, destination vertices of all pairs are reached.
As an initial, reachable vertex, Rule~(\ref{red:pair1st}) sets the source vertex~$s$ reachable, whose pair is closed first.
Then, the linking of pairs is carried out along the TD in the order of closure, as given by~$\sigma$.
Thereby, Rules~(\ref{red:paircycles}) conceptually construct auxiliary links (similar to edges) between different pairs, in the order of~$\sigma$,
which is guided along the TD %
to ensure only a linear increase in treewidth of~$\mathcal{G}_\Pi$ of the resulting program~$\Pi$. %
Interestingly, these additional dependencies, since guided along the TD, do not increase the treewidth by much as we will see in the next subsection.

Then, it is \emph{crucial} that we prevent a source vertex~$s_i$ of a pair~$(s_i,d_i)\in_i \sigma$ 
from reaching a destination vertex~$d_j$ of a pair~$(s_j,d_j)\in_j\sigma$ preceding~$(s_i,d_i)$ in~$\sigma$, i.e., $j<i$.
To this end, we need to construct parts of %
cycles that prevent this.
Concretely, if some source~$s_i$ reaches~$d_j$, i.e., $d_j$ is reachable via~$s_i$, 
the goal is to have a cyclic reachability from~$d_j$ to~$s_i$, with no provability for corresponding reachability atoms of the cycle. %
Actually, Rules~(\ref{red:paircycles}) %
also have the purpose of aiding in the construction of these potential positive cycles.
Thereby we achieve that if~$d_j$ is reachable, this cannot be due to~$s_i$, since reachability of~$d_{j}, s_{j+1}, \ldots, s_i$ (therefore~$s_i$ itself) is required for reachability of~$s_{i}$. %
Consequently, assuming that there is no further rule proving any of these reachability atoms,
which we will ensure in the construction of program~$\Pi_{\mathcal{C}}$ below,
we end up with cyclic reachability if~$s_i$ is reached by $d_j$,
such that none of the atoms of the cycle are proven.
Figure~\ref{fig:cycles} shows the positive dependency graph~$D_{R_{\mathcal{L}}}$
of Rules~(\ref{red:paircycles}), where pairs~$(s_i,d_i)\in_i\sigma$, as discussed in the following example.

\begin{example}\label{ex:cycle}
Consider the dependency graph~$D_{R_{\mathcal{L}}}$ of Rules~(\ref{red:paircycles}),
as depicted in Figure~\ref{fig:cycles}.
Observe that whenever $s_i$ reaches some $d_j$ with~$j<i$,
this causes a cycle~$C{=}r_{s_i},\ldots, r_{d_j}, r_{s_{j+1}}, \ldots, r_{s_{i-1}}, r_{s_i}$ over reachability atoms in~$D_{R_{\mathcal{L}}}$ (cyclic dependency).
If each vertex~$u$ of~$G$ can have at most one outgoing edge, i.e., 
only one atom~$e_{u,v}$ in an answer set of~$\Pi=R(I,\mathcal{T})$,
no atom of $C$ can be proven (no further rule allows provability).
\FIX{Note that~$C$ could also be constructed by causing in the positive dependency graph
$\mathcal{O}(\Card{P}^2)$ many edges from~$r_{d_j}$ to~$r_{s_i}$ for~$j<i$.
This could be achieved, e.g., by constructing large rules, where reachability~$r_{d_{j}}$ of every preceding destination
vertex is required in the positive body in order to reach a certain source vertex~$s_i$, i.e., in order to obtain reachability~$r_{s_i}$.}
However, this would cause an increase of structural dependency, %
and in fact, the treewidth increase would be beyond linear. %
\end{example}

\FIX{
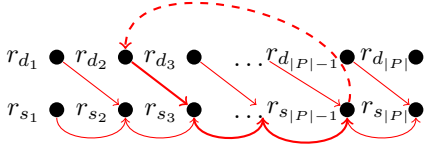
\begin{figure}[t]%
  \centering%
  \vspace{-1em}
	\begin{tikzpicture}[node distance=7mm,every node/.style={fill,circle,inner sep=2pt}]%
		\node (s0) [white] {};
		\node (s1) [right=of s0,label={[text height=1.5ex,yshift=0.0cm,xshift=0.05cm]left:$r_{s_1}$}] {};
		\node (s2) [right=of s1,label={[text height=1.5ex,yshift=0.0cm,xshift=0.05cm]left:$r_{s_2}$}] {};
		\node (s5) [right=of s2,label={[text height=1.5ex,yshift=0.0cm,xshift=0.05cm]left:$r_{s_3}$}] {};
		\node (sdots) [white,right=of s5,label={[text height=1.5ex,yshift=0.0cm,xshift=0.05cm]left:$ $}] {};
		\node (sdots2) [white,right=of sdots,xshift=-2em,label={[text height=1.5ex,yshift=0.0cm,xshift=0.05cm]left:$\ldots$}] {};
		\node (sn1) [right=of sdots2,label={[text height=1.5ex,yshift=0.0cm,xshift=0.15cm]left:$r_{s_{\Card{P}-1}}$}] {};
		\node (sn) [right=of sn1,label={[text height=1.5ex,yshift=0.0cm,xshift=0.2cm]left:$r_{s_{\Card{P}}}$}] {};

		\node (d1) [above of=s1,yshift=.0em,label={[text height=1.5ex,xshift=0.05cm]left:$r_{d_1}$}] {};
		\node (d2) [above of=s2,yshift=.0em,label={[text height=1.5ex,xshift=0.05cm]left:$r_{d_2}$}] {};
		\node (d5) [above of=s5,yshift=.0em,label={[text height=1.5ex,xshift=0.05cm]left:$r_{d_3}$}] {};
		\node (ddots) [white,above of=sdots,yshift=.0em,label={[text height=1.5ex,xshift=0.05cm]left:$ $}] {};
		\node (ddots2) [white,above of=sdots2,yshift=.0em,label={[text height=1.5ex,xshift=0.05cm]left:$\ldots$}] {};
		\node (dn1) [above of=sn1,yshift=.0em,label={[text height=1.5ex,xshift=0.15cm]left:$r_{d_{\Card{P}{-1}}}$}] {};
		\node (dn) [above of=sn,yshift=.0em,label={[text height=1.5ex,xshift=0.2cm]left:$r_{d_{\Card{P}}}$}] {};
		\draw [solid,->,red] (d1) to (s2);
		\draw [solid,->,thick,red] (d2) to (s5);
		\draw [solid,->,red] (d5) to (sdots);
		\draw [solid,->,red] (ddots) to (sn1);
		\draw [solid,->,red] (dn1) to (sn);
		\draw [dashed,thick,->,red,out=77,in=90] (sn1) to (d2);
		\draw [out=-90,in=-90,solid,->,red] (s1) to (s2);
		\draw [out=-90,in=-90,solid,->,red] (s2) to (s5);
		\draw [out=-90,in=-90,solid,thick,->,red] (s5) to (sdots);
		\draw [out=-90,in=-90,solid,thick,->,red] (sdots) to (sn1);
		\draw [out=-90,in=-90,solid,->,red] (sn1) to (sn);
	\end{tikzpicture}%
    \vspace{-.95em}
    \caption{\FIX{Positive dependency graph~$D_{R_{\mathcal{L}}}$ (indicated by solid red edges) of Rules~(\ref{red:paircycles}) constructed for any closure sequence~$\sigma$ such that~$(s_i,d_i)\in_i\sigma$.
If a source~$s_i$ reaches a destination~$d_j$ of a preceding pair, i.e., $j<i$, (depicted via the dashed red edge), this results in a cycle (consisting of all bold-faced edges) such that none of the atoms of the cycle can be proven.}
}
    \label{fig:cycles}
\end{figure}}

\smallskip
\noindent\textbf{Checking of disjointness~$\Pi_\mathcal{C}$.} 
Finally, we create rules in $\Pi$ that enforce at most one outgoing, used edge per vertex.
This is required to ensure that we do not use a vertex twice, as required by the~\problemFont{Disjoint Paths Problem}.
We do this by guiding the information, whether the corresponding outgoing edge was used,
via atoms~$f^u_t$ along the TD to ensure that the treewidth is not increased significantly.
Having at most one outgoing, used edge per vertex of~$G$ further ensures 
that when a source of a pair~$p$ reaches a destination of a pair preceding~$p$ in~$\sigma$, 
then no atom of the resulting cycle as constructed in $\Pi_{\mathcal{L}}$ %
will be provable. %
Consequently, in the end every source of~$p$ has to reach the destination of~$p$ by the pigeon hole principle.
Program~$\Pi_\mathcal{C}$ is constructed for every node~$t$ with~$t',t''{\in}\children(t)$, if~$t$ has child nodes, as follows.
{
\vspace{-.3em}
\begin{flalign}
	\label{red:setf}&f_{t}^u\leftarrow e_{u,v}&&{\text{for each }(u,v)\in E^{\text{re}}_{t}, u\in\chi(t)}\\
	\label{red:propf}&f_t^u\leftarrow f_{t'}^u&&{\text{for each }u\in\chi(t)\cap\chi(t')}\\
	\label{red:prohibitf}&\leftarrow f_{t'}^u, f_{t''}^u&&{\text{for each }u\in\chi(t')\cap\chi(t''), t'\neq t''}\\
	\label{red:tddegree}&\leftarrow f_{t'}^u, e_{u,v} &&{\text{for each }(u,v)\in E^{\text{re}}_{t}, u\in\chi(t')}\\
	\label{red:localdegree}&\leftarrow e_{u,v}, e_{u,w}&&{\text{for each }(u,v),(u,w)\in E^{\text{re}}_{t}\hspace{-.3em},\ v{\neq}w}
\end{flalign}%
\vspace{-1em}
}

\noindent Rules~(\ref{red:setf}) ensure that the finished flag~$f^u_t$ is set for used edges~$e_{u,v}$.
Then, this information of~$f^u_{t'}$ is guided along the TD from child node~$t'$ to parent node~$t$ by Rules~(\ref{red:propf}).
If for a vertex~$u\in V$ we have~$f^u_{t'}$ and~$f^u_{t''}$ for two different child nodes~$t', t''\in\children(t)$,
this indicates that two different edges were encountered both below~$t'$ and below~$t''$. Consequently,
this situation is avoided by Rules~(\ref{red:prohibitf}).
Rules~(\ref{red:tddegree}) make sure to disallow additional edges for vertex~$u$ in a TD node~$t$,
if the flag~$f^u_{t'}$ of child node~$t'$ is set.
Finally, Rules~(\ref{red:localdegree}) prohibit %
two different edges 
for the same vertex~$u$ within a TD node.

\begin{example}
Recall instance~$I=(G,P)$ with $G=(V,E)$ of Figure~\ref{fig:disjpaths}, pair-connected TD~$\mathcal{T}=(T,\chi)$ of~$I$ of Figure~\ref{fig:tds} (right),
and $E_{t_2}^{\text{re}}=\{(y,z), (z,y), (z, d_3),\allowbreak (s_3, z)\}$.
We briefly present the construction of~$\Pi_{\mathcal{C}}$ for node~$t_2$.
\noindent\begin{tabular}{@{\hspace{0.15em}}l@{\hspace{0.15em}}|@{\hspace{0.15em}}l@{\hspace{0.0em}}}
Rules & $\Pi_{\mathcal{L}}$\\
\hline
(\ref{red:setf})& $f^y_{t_2} \leftarrow e_{y,z}$; $f^{s_3}_{t_2} \leftarrow e_{s_3,z}$\\
(\ref{red:propf}) & $f^{s_3}_{t_2} \leftarrow f^{s_3}_{t_1}$; $f^{d_3}_{t_2} \leftarrow f^{d_3}_{t_1}$; $f^y_{t_2} \leftarrow f^y_{t_1}$\\
(\ref{red:tddegree}) & $\leftarrow f^y_{t_1}, e_{y,z}$; $\leftarrow f^z_{t_1}, e_{z,y}$; $\leftarrow f^z_{t_1}, e_{z,d_3}$; $\leftarrow f^{s_3}_{t_1}, e_{s_3,z}$\\
(\ref{red:localdegree})& $\leftarrow e_{z,y}, e_{z,d_3}$
\end{tabular}

\end{example}

\futuresketch{\paragraph{Modification for non-nice TDs.}
Note that the reduction as presented above, conceptually works for a \emph{non-nice}, pair-connected TD~$\mathcal{T}'=(T',\chi')$ of~$(G,P)$.
However, one needs to take care of two issues.
First, the reduction lies on~$\mathcal{T}'$ being pair-connected as defined above.
If there is a node in~$\mathcal{T}'$ with~$c$ many child nodes, pair-connectedness could cause an 
increase of~$3\cdot(c+1)$ in the treewidth, depending on how distributed the pairs of~$P$ in~$\mathcal{T}'$ are.
Then, one needs to guarantee that~$\Card{E^{\text{re}}_t}$ is still manageable for a node~$t$ of~$T'$, since already a number quadratic
in the width of~$\mathcal{T}'$ might be insufficient. %
However, one can keep~$\Card{E^{\text{re}}_t}$ small by adding intermediate nodes between~$t$ and the parent of~$t$ to the TD.
}

\smallskip
\noindent
\textbf{Correctness and Runtime Analysis.}
First, we show that the reduction is correct,
followed by a result stating that the treewidth of the reduction is at most linearly worsened,
which is crucial for the runtime lower bound to hold.
Then, we present the runtime and the (combined) main result of this work.

\begin{theorem}[Correctness]\label{thm:corr}
Reduction~$R$ as proposed in this section is correct.
Let us consider an instance~$I=(G,P)$ of the \problemFont{Disjoint Paths Problem},
and a pair-connected TD~$\mathcal{T}=(T,\chi)$ of~$G$.
Then, $I$ has a solution if and only if the program~$R(I,\mathcal{T})$ admits an answer set.
\end{theorem}
\begin{proof}
The proof is given in \ref{sec:corrhard}.
\end{proof}

\begin{lemma}[Treewidth-awareness]\label{lem:treewidthaware}
Let~$I=(G,P)$ be any instance of the \problemFont{Disjoint Paths Problem},
and~$\mathcal{T}$ be a nice, pair-connected TD %
of~$I$ of width~$k$. 
Then, the treewidth of~$\mathcal{G}_\Pi$,
where~$\Pi=R(I,\mathcal{T})$ is obtained by~$R$, is
at most~$\mathcal{O}(k)$.
\end{lemma}
\begin{proof}
Assume any pair-connected, nice TD~$\mathcal{T}=(T,\chi)$ of~$I=(G,P)$. Since~$\mathcal{T}$ is nice, a node in~$T$ has at most~$\ell=2$ many child nodes.
From~$\mathcal{T}$ we construct a TD~$\mathcal{T}'=(T,\chi')$ of~$\mathcal{G}_\Pi$.
Thereby we set for every node~$t$ of~$T$, $\chi'(t)\eqdef \{r_u, f^u_t \mid u \in \chi(t)\} \cup \{e_{u,v}, ne_{u,v}, r_u, r_v, f^u_{t'} \mid (u,v)\in E^{\text{re}}_t, t'\in\children(t), u\in\chi(t')\} \cup \{f^u_{t'}, f^u_t \mid t' \in\children(t), u\in\chi(t)\cap\chi(t')\}$. %
Observe that~$\mathcal{T}'$ is a valid TD of~$\mathcal{G}_\Pi$.
Further, by construction we have~$\Card{\chi'(t)} \leq 2\cdot \Card{\chi(t)}  + (4 + \ell) \cdot k + (\ell + 1) \cdot \Card{\chi(t)}$,
since~$\Card{E_t^{\text{re}}}\leq k$.
The claim sustains for nice TDs ($\ell=2$).
\end{proof}

\begin{corollary}[Runtime]
Reduction~$R$ as proposed in this section runs for a given
instance~$I=(G,P)$ of the \problemFont{Disjoint Paths Problem} with~$G=(V,E)$,
and a pair-connected, nice TD~$\mathcal{T}$ of~$I$ of width~$k$
and~$h$ many nodes, in time $\mathcal{O}(k\cdot h)$.
\end{corollary}

Next, we are in the position of showing the main result, namely the normal ASP lower bound.

\begin{theorem}[Lower bound]\label{thm:lowerbound}
Consider an arbitrary normal or HCF program~$\Pi$, where~$k$ is the treewidth
of the primal graph of~$\Pi$. 
Then, unless ETH fails, the consistency problem for~$\Pi$ %
cannot be solved in time~$2^{o(k\cdot \log(k))}\cdot \poly(\Card{\at(\Pi)})$.
\end{theorem}
\begin{proof}

Let~$(G,P)$ be an instance of the~\problemFont{Disjoint Paths Problem}.
First, we construct~\cite{BodlaenderEtAl13} a nice TD~$\mathcal{T}$ of~$G=(V,E)$ of treewidth~$k$ in time~$c^k\cdot \Card{V}$
for some constant~$c$ such that the width of~$\mathcal{T}$ is at most~$5k+4$.
Then, we turn the result into a pair-connected TD~$\mathcal{T}'=(T',\chi')$, thereby having width at most~$k'= 2\cdot(5k+4)+ 6$.
Then, we construct program~$\Pi=R(I,\mathcal{T}')$.
By Lemma~\ref{lem:treewidthaware}, the treewidth of~$\mathcal{G}_\Pi$ is in~$\mathcal{O}(k')$, which is in~$\mathcal{O}(k)$.
Assume towards a contradiction that consistency of $\Pi$ can be decided in time~$2^{o(k\cdot \log(k))}\cdot \poly(\Card{\at(\Pi)})$.
By correctness of~$R$ (Theorem~\ref{thm:corr}), this solves~$(G,P)$, contradicting Proposition~\ref{prop:slightlysuper}.
\end{proof}

Our reduction works by construction for any pair-connected TD.
\FIXR{Consequently, this immediately yields a lower bound for the larger parameter \emph{pathwidth},
which is similar to treewidth, as defined in the preliminaries.} %

\begin{corollary}[Pathwidth lower bound]
Consider any normal or HCF program~$\Pi$, where~$k$ is the pathwidth
of the primal graph of~$\Pi$. 
Then, unless ETH fails, the consistency problem for~$\Pi$ %
cannot be solved in time~$2^{o(k\cdot \log(k))}\cdot \poly(\Card{\at(\Pi)})$.
\end{corollary}

From Theorem~\ref{thm:lowerbound}, we infer that %
a general reduction from normal or HCF programs to \SAT formulas
cannot (uner ETH) avoid the treewidth (pathwidth) overhead, 
which renders our reduction from the previous section ETH-tight.

\begin{corollary}[ETH-tightness of the Reduction to \SAT]
Under ETH, the increase of treewidth of the reduction using Formulas~(\ref{red:checkrules})--(\ref{red:checkfirst}) cannot be significantly improved.
\end{corollary}\vspace{-.5em}
\begin{proof}
Assume towards a contradiction that one can 
reduce from an arbitrary normal \ASP
program~$\Pi$, where~$k$ is the treewidth of~$\mathcal{G}_\Pi$ 
to a \SAT formula, whose treewidth is in~$o(k\cdot\log(k))$.
Then, this contradicts Theorem~\ref{thm:lowerbound},
as we can use an algorithm~\cite{SamerSzeider10b,FichteHecherZisser19} for \SAT being single exponential in the treewidth,
thereby deciding consistency of~$\Pi$ in time~$2^{o(k\cdot\log(k))}\cdot\poly(\Card{\at(\Pi)})$.
\end{proof}

\FIXR{Further consequences of our construction are discussed in~\ref{sec:cons}.}

\FIXR{\section{An Empirical Study of Treewidth-Aware Reductions}\label{sec:emp}

Despite the lower bound of the previous section, we show that the reductions of this work could still have practical impact.
Recall the formalization of our treewidth-aware reduction of Section~\ref{sec:tdguided}, which is guided along a tree decomposition.
We implemented this translation in Python, resulting in the prototypical tool \aspmc\footnote{Our translator \aspmc is open source and readily available at \href{https://github.com/hmarkus/asp2sat_translator}{github.com/hmarkus/asp2sat\_translator}.}, in order to compare and study the effect on the (tree)width in practice.

So the overall goal of this section is to provide a neutral comparison, without explicitly focusing on beating existing ASP solvers. 
The reason for this goal lies in two reasons.
First of all, there is a recent observation that modern (\SAT) solvers have been highly optimized for current hardware for many years, cf.,~\cite{FichteHecherSzeider20}, which also takes time for treewidth-based solvers.
Further, finding ways to efficiently utilize treewidth seems to be still an ongoing process. 
While over the last couple of years, efficient solvers emerged, a clear set of techniques has not yet been fully settled.
This can be witnessed by a couple of solvers that adhere to parameterized complexity and in particular to treewidth.
Some techniques for solvers on Boolean formulas at least seem to be particular well-suited for counting problems, e.g., based on dynamic programming~\cite{FichteHecherZisser19,FichteEtAl20}, hybrid solving~\cite{HecherThierWoltran20,BesinHecherWoltran21}, which also includes the winner~\cite{tuukka,KorhonenMatti21,mc21} of two tracks of the most recent model counting competition~\cite{FichteHecherHamiti20}, as well as compact knowledge compilation~\cite{DudekPhanVardi20,DudekPhanVardi20b}.
However, treewidth also allows to improve solving hard decision problems with the help of knowledge compilation~\cite{CharwatWoltran19}.

Despite this still incomplete picture on potential applications for treewidth, we present the usefulness and potential impact of our reduction in the form of a neutral comparison based on treewidth upper bounds. 
Note that computing treewidth itself is NP-complete in general~\cite{Arnborg87}, so in this section we only focus on approximating treewidth upper bounds.
So, in order to show that the reduction of Section~\ref{sec:tdguided} indeed bounds the treewidth, we compare treewidth upper bounds on our reduction with treewidth upper bounds on an established tool of the literature~\cite{Janhunen06}.
These upper bounds are obtained with an efficient decomposer, called \emph{htd}, which aims at obtaining decent tree decompositions fast~\cite{AbseherMusliuWoltran17}.
However, especially if the instances are of high (tree)widths, the obtained upper bounds via the heuristics are oftentimes imprecise, cf.~\cite{Dell17a,FichteLodhaSzeider17}. 
Note that compared to treewidth, htd only provides upper bounds, where sometimes causing additional edges in the graph representation might strangely yield smaller widths.
Nevertheless, for initial estimations and for studying trends, such decomposers still proved valueable in many treewidth-based solvers.

Our implementation \aspmc works as follows:
\begin{enumerate}[itemsep=0pt]
    \item First, we parse the input program $\Pi$, using the Python API of clingo~\cite{GebserEtAl14}. 
    \item Then, we compute an initial treewidth upper bound for the primal graph of $\Pi$ using htd~\cite{AbseherMusliuWoltran17}, which heuristically provides us a tree decomposition. %
    \item Next, we perform our TD-guided reduction of Section~\ref{sec:tdguided} along the lines of the TD of the previous step.
    \item Finally, we use again htd in order to obtain a treewidth upper bound on the result.
\end{enumerate}

\paragraph{Compared Translations}
In our experiments, we mainly compare treewidth upper bounds  of the following
translations.
\begin{itemize}
	\item \aspmc: we use version 1.2 of htd as well as version 4.5 of clingo in order to translate from \ASP to \SAT. 
	\item lp2sat: instances are translated~\cite{Bomanson17} to CNFs by lp2normal 2.18 in combination with lp2atomic 1.17 and lp2sat 1.24. This allows us to compare the effect of global level mappings with the local level mappings of \aspmc.
\end{itemize}

\paragraph{Benchmark Scenarios}
In order to evaluate \aspmc, we considered the following scenarios.
\begin{enumerate}[align=left]
	\item[S1] Acyclic Instances: Comparing treewidth overhead on instances, whose positive dependency graph is acyclic. For \aspmc this still leaves the guidance of proofs along a tree decomposition. %
	\item[S2] Cyclic Instances: Evaluating the treewidth overhead on instances with cycles, which means for \aspmc that local level mappings are used.
	\item[S2b] Cyclic Instances: Studying the treewidth overhead on cyclic instances as in S2, but here we use global level mappings for \aspmc and still guide proofs along decompositions.
\end{enumerate}

For Scenarios S1 and S2, \aspmc uses local level mappings as defined in Section~\ref{sec:tdguided}, whereas for S2b the proofs are still guided along decompositions, which is then, however, mixed with global level mappings similar as in lp2sat. 

\paragraph{Benchmark Instances}
For Scenario S1, we used instances from recent applications that where heavily trained in competitions.
In fact there are still active competitions that utilize answer set programming, which is in contrast to
the ASP competition (ASPCOMP), whose last edition dates back to 2017.
One such active competitition is the international competition on computational models of argumentation~\cite{ICCMA21},
which has a long tradition of applying ASP.
So, for Scenario S1 our instances stem from the abstract argumentation competition~\cite{ICCMA21} from 2019,
since at the time of writing the instances for 2021 were not available. 
Then, on top we used ASP encodings for three canonical problems in that area~\cite{ICCMA21}: Computing admissible, complete, and stable extensions, whose ASP encodings were taken from the ASPARTIX suite~\cite{DvorakEtAl20}.
The obtained programs of those instances with the resulting encodings are acyclic.

For Scenario S2 (S2b) we use real-world graphs, more specifically, public transport networks of several transport agencies over the world~\cite{Fichte16} as instances.
They also have been used in the so-called PACE challenge competitions 2016 and 2017~\cite{Dell17a}. In total these instances amount to 561 graph networks and 2553 subgraphs with a focus on different transportation modes.
On top of these networks, we encoded reachability, where for each instances 
we assume the station with the smallest and largest index to be the start and end stations, respectively.

Note that all our instances
including the corresponding ASP encoding %
and raw data of our benchmark runs are open source and are publicly available on github at \href{https://github.com/hmarkus/asp2sat_translator/tree/benchmarks/results}{github.com/hmarkus/asp2sat\_translator/tree/benchmarks/results}.

\begin{figure}[t]
\centering
\includegraphics[scale=0.43]{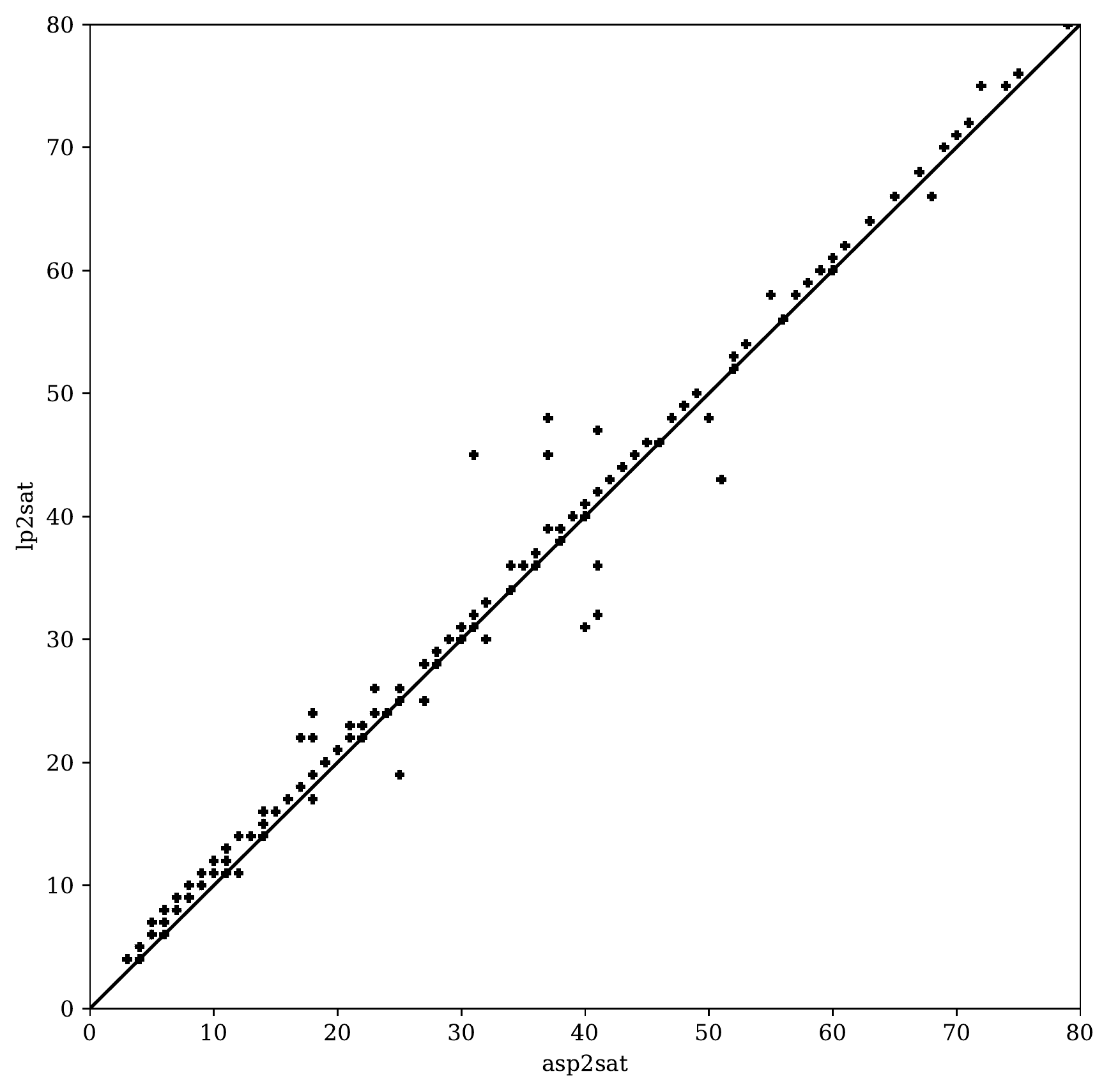}%
\caption{Scatter plot of \aspmc and lp2sat for Scenario S1, which compares the obtained treewidth upper bounds after translation. Observe that dots above the diagonal indicate instances, where \aspmc yields better treewidth upper bounds.}
\label{fig:acyclic}
\end{figure}

\begin{table}[t]
  \centering
\resizebox{.74\columnwidth}{!}{%
  \begin{tabular}{{lllrrrrr}}
    \toprule \multirow{2}{*}{Scenario} & \multirow{2}{*}{$\sum$} & \multirow{2}{*}{Translation}  & \multicolumn{5}{c}{Treewidth upper bounds}\\\cmidrule{4-8}
            && & $\min$ & $\max$ & mean & median & stddev  \\
    \midrule
    \multirow{3}{*}{S1} & \multirow{3}{*}{478} & - & 2 & 82 & 19.7 & 14 & 18.5  \\
    && \aspmc & \textbf{3} & \textbf{79} & \textbf{24.5} & \textbf{17} & \textbf{19.9} \\
    && lp2sat & 4 & 80 & 25.5 & 18 & \textbf{19.9} \\\midrule
    \multirow{3}{*}{S2} & \multirow{3}{*}{861} & - & 1 & 8 & 3.7 & 4 & 1.5  \\
    && \aspmc & \textbf{3} & 80 & \textbf{25.6} & \textbf{19} & {19.2} \\
    && lp2sat & \textbf{3} & \textbf{79} & 32.4 & 29 & \textbf{17.6} \\\midrule
\multirow{3}{*}{S2b} & \multirow{3}{*}{1053} & - & 1 & 35 & 5.0 & 4 & 3.9  \\
    && \aspmc & \textbf{3} & \textbf{80} & \textbf{32.0} & \textbf{29} & \textbf{17.0} \\
    && lp2sat & \textbf{3} & \textbf{80} & 39.1 & 37 & {21.7}\\
    \bottomrule
  \end{tabular}
 }
  \caption{Detailed results over Scenarios S1, S2, and S2b, where we depict statistical data ($\min$, $\max$, mean, median, stddev) among the number of instances ($\sum$) solved by both lp2sat and \aspmc such that the treewith upper bound is below~$80$. Note that translation ``-'' refers to initial treewidth upper bounds, i.e., bounds obtained by using no translation. Overall, \aspmc yields significantly lower bounds than lp2sat.
  }
  \label{tab:transit}
\end{table}

\begin{figure}[t]
\centering
\includegraphics[scale=0.43]{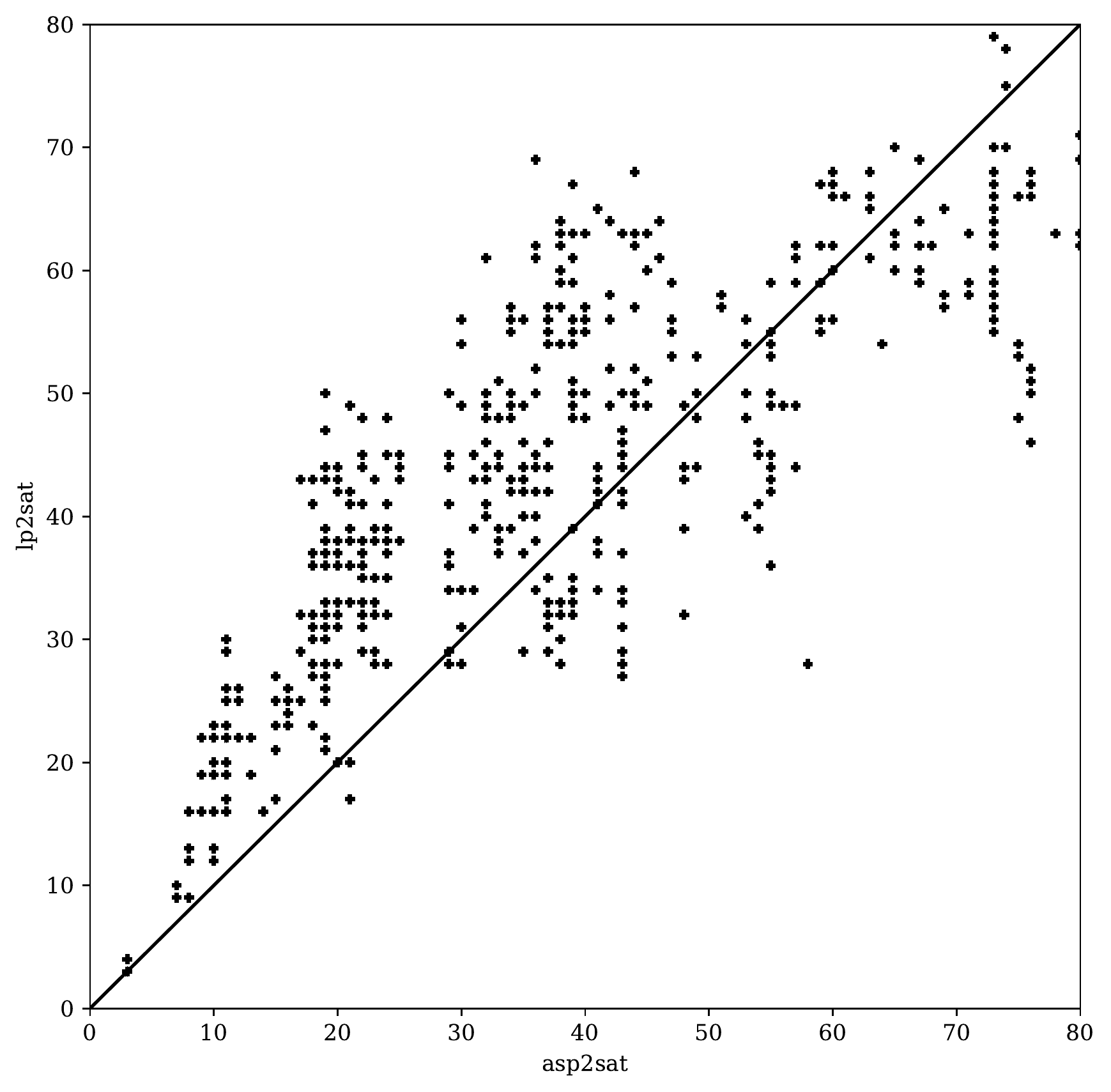}%
\caption{Scatter plot of \aspmc and lp2sat for Scenario S2, comparing treewidth upper bounds, similar to Figure~\ref{fig:acyclic}. Observe that \aspmc performs best up to about~$50$.}
\label{fig:cyclic}
\end{figure}

\paragraph{Benchmark Hardware}
All our solvers ran on a cluster consisting of 12 nodes. 
Each node of the cluster is equipped
with two Intel Xeon E5-2650 CPUs, where each of these 12 physical cores runs
at 2.2 GHz clock speed and has access to 256 GB shared RAM. %
Results are gathered on Ubuntu~16.04.1 LTS powered on kernel~4.4.0-139 with hyperthreading disabled using version 3.7.6 of Python3.
Since we consider these translation to be used in a preprocessing step, we  allow for each instance and translation up to  200 seconds (timeout) and 8 GB of main memory.
We compare the quality of the obtained treewidth upper bounds in the range from~$0$ to~$80$, which is backed up by the maximal width that is reasonable for being utilized by state-of-the-art solvers, e.g.,~\cite{BannachBerndt19,CharwatWoltran19,FichteHecherZisser19,DudekPhanVardi20,DudekPhanVardi20b,HecherThierWoltran20}.

\paragraph{Benchmark Results}

First, we discuss Scenario S1, for which we present in Figure~\ref{fig:acyclic} a scatter plot of the treewidth upper bounds of \aspmc in comparison with lp2sat.
It is easy to see that, despite the need for level mappings in S1, the usage of \aspmc yields better upper bounds in almost all the cases. 
Of course there are some outliers in Figure~\ref{fig:acyclic} below the diagonal, which could, however, also be due to the heuristical decomposer.
Detailed statistical measures are given in Table~\ref{tab:transit}, which confirms our observations that \aspmc slightly decreases also treewidth upper bounds on acyclic instances like the ones used in Scenario S1. 
Indeed, the overall increase compared to the initial treewidth upper bounds is not large for S1.

For Scenario S2, we first observe in Figure~\ref{fig:cyclic} that \aspmc yields significantly smaller treewidth upper bounds compared to lpsat in the range from 0 to 35. Then, up to values of 50, \aspmc still delivers good performance. Finally for larger bounds of \aspmc, the heuristics of htd
probably suffer from larger primal graphs due to auxiliary variables, revealing that the treewidth approximation does not scale.
Table~\ref{tab:transit} shows that compared to lp2sat, our translation \aspmc significantly decreases the mean and median among treewidth upper bounds for a majority of the benchmark instances. %

To mitigate this disadvantage of estimating treewidth in case of many auxiliary variables, 
we slightly adapted \aspmc such that it uses global level mappings as in lp2sat.
Comparing the scatter plots of Figure~\ref{fig:cyclicproven} with Figure~\ref{fig:cyclic} reveals that for treewidth bounds up to about~30 local level mappings pay off. Consequently, 
in these cases the heuristics of htd still efficiently utilize local level mappings.
However, these heuristics are still rather limited: Beyond treewidth upper bounds of 30, Figure~\ref{fig:cyclicproven} reveals that global level mappings work better with state-of-the-art decomposers. This still leaves room for improvement for future decomposition heuristics that improve on large instances.
In any case, the treewidth upper bound improvement of \aspmc compared to lp2sat is still significat in Scenario S2b, as indicated by statistical measures of Table~\ref{tab:transit}.

\begin{figure}[t]
\centering
\includegraphics[scale=0.43]{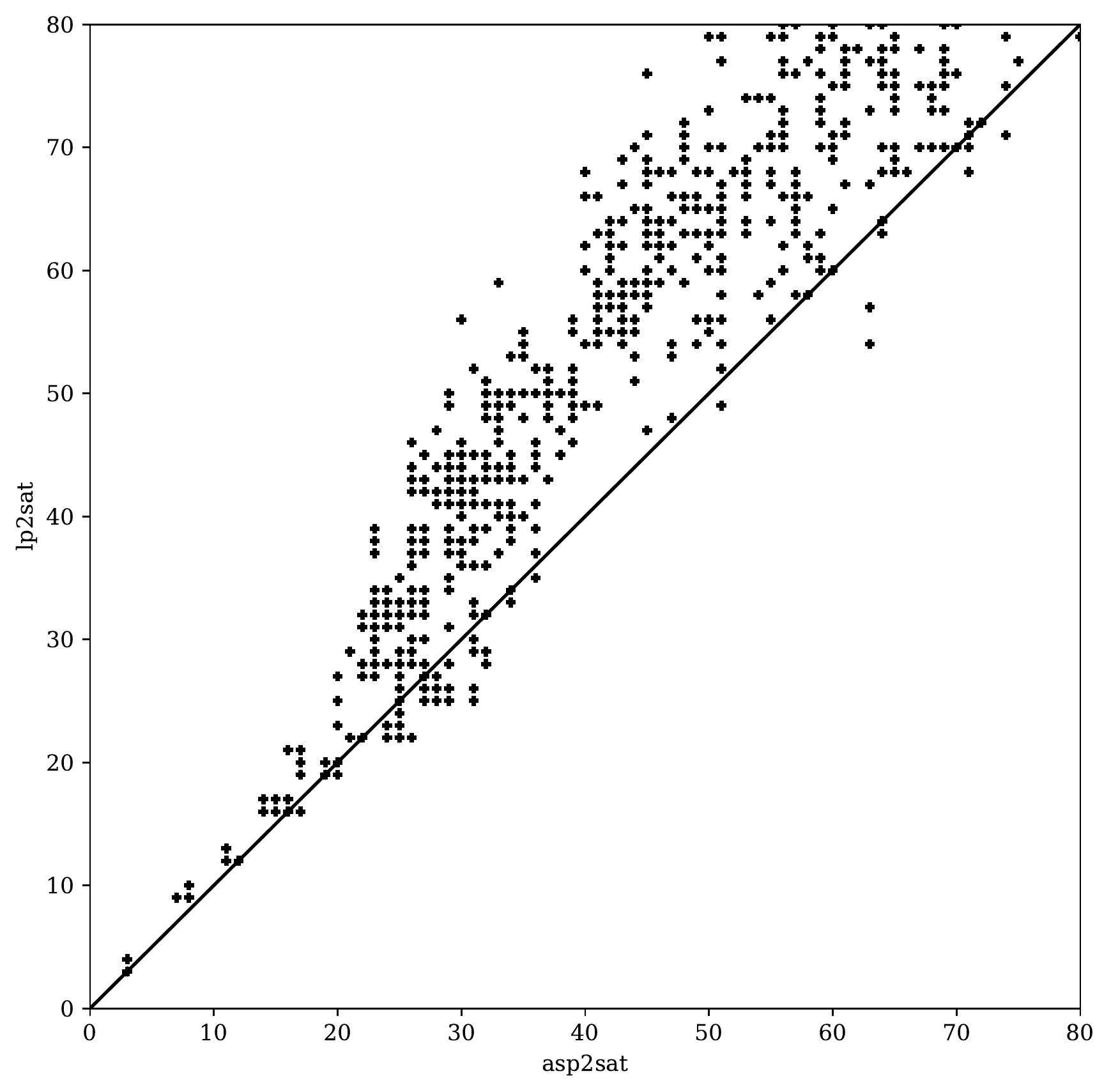}%
\caption{Scatter plot of \aspmc and lp2sat for Scenario S2b, which is similar to S2, but without local level mappings. Observe that compared to Figure~\ref{fig:cyclic}, \aspmc improves the result in almost all the cases.}
\label{fig:cyclicproven}
\end{figure}

\begin{table}[t]
  \centering
\resizebox{1.01\columnwidth}{!}{%
  \begin{tabular}{{@{\hspace{-.1em}}HHlrrrrr@{\hspace{-.1em}}H}}
    \toprule \multirow{2}{*}{Scenario} & \multirow{2}{*}{$\sum$} & \multirow{2}{*}{Translation}  & \multicolumn{5}{c}{Measures}\\\cmidrule{4-7}
            && & Number of Clauses & Number of Atoms & Treewidth & Program Size & Runtime & Source\\ %
    \midrule
    \multirow{3}{*}{S1} & \multirow{3}{*}{478} & - & $m$ & $n$ & $k$ & $s\in \mathcal{O}(nm)$ & - & -\\
    && \aspmc & \textbf{$\mathcal{O}(k^2\log(k)\cdot (n{+}m))$} %
    & $\mathcal{O}(k\log(k)\cdot (n{+}m))$ %
    & \textbf{$\mathcal{O}(k\log(k))$} %
    & $\mathcal{O}(k^2\log(k)\cdot (n{+}m))$
    & $\mathcal{O}(k^2\log(k)\cdot (n{+}m))$\\
    && lp2sat~\cite{Janhunen06} & $\mathcal{O}(s\log(n))$ & $\mathcal{O}(n\log(n))$ & $\mathcal{O}(k\log(n))$ & $\mathcal{O}(s\log(n))$ & 
    $\mathcal{O}(s\log(n))$\\ %
    \bottomrule
  \end{tabular}
 }
  \caption{\FIXR{Comparison of worst-case measures in terms of number of clauses (atoms), treewidth, program size, and runtime for both \aspmc and lp2sat. 
  Note that translation ``-'' refers to initial values of the given input program. %
  Overall, there are cases, where \aspmc yields smaller values than lp2sat,
  which however is not true in general.
  }}
  \label{tab:hardfacts}
\end{table}

To provide some basic decision guideline, whether for some specific setting
\aspmc or lp2sat might be more appropriate, we briefly summarize the hard facts
of both approaches in Table~\ref{tab:hardfacts}.
Note that there are also further flavors and extensions of both approaches, 
cf.~, e.g.,~\cite{GebserJanhunenRitanen14,BomansonEtAl16,FandinnoHecher21}.
}

\section{Discussion, Conclusion, and Future Work}

Understanding the hardness of \ASP and the underlying reasons has attracted the attention of the KR community for a long time.
This paper discusses this question from a different angle,
which hopefully will provide new insights into the hardness of \ASP and foster follow-up work.
The results in this paper indicate that, at least from a structural point of view, deciding the consistency of \ASP is already harder than %
\SAT, since \ASP programs might compactly represent structural dependencies within the formalism.
More concretely, compiling the hidden structural dependencies of a program to a \SAT formula, measured in terms of the well-studied parameter treewidth, causes a blow-up of the treewidth of the resulting formula.
In the light of a known result~\cite{AtseriasFichteThurley11} on the correspondence of treewidth and the resolution width applied in \SAT solving, this reveals that \ASP consistency might be indeed harder than solving \SAT.
We further presented a reduction from \ASP to \SAT that is aware of the treewidth in the sense that the reduction causes not more than this inevitable blow-up of the treewidth in the worst-case. \FIXR{Finally, we present an empirical evaluation of treewidth upper bounds obtained via standard heuristics for treewidth,
where we compare our reduction against existing translations.
In both cyclic as well as acyclic scenarios, our reduction and adaptions thereof seems promising and one might
consider treewidth-aware reductions as a preprocessing tool in a portfolio setting.
Interestingly, the reduction developed in this work already lead to follow-up studies~\cite{FandinnoHecher21}
and treewidth-aware cycle breaking~\cite{EiterHecherKiesel21}.
}

Our paper raises several questions for future work. 
\FIXR{While in this work, we consider treewidth-aware reductions 
for head-cycle-free programs, the construction can be also extended
to cover the SModels intermediate format~\cite{ElkabaniEtAl05}, which can be obtained
via typical grounders that are capable of the ASP-Core-2 format~\cite{CalimeriEtAl20} including, e.g.,
aggregates and choice rules.
Intuitively, the treatment of choice rules is rather similar to existing approaches for treewidth~\cite{FichteEtAl17a} and 
the SModels format does not allow aggregates (so it needs to be removed during grounding).
Still, we think it is an interesting question for future work, how more advanced formats like aspif~\cite{GebserEtAl19} that also considers hybrid solving 
can be treated for treewidth. %
However, for weight rules, treewidth alone is intractable, i.e., it is not expected
that such a reduction exists~\cite{PichlerEtAl14}. Further, for disjunctive rules one might expect
a double exponential runtime, cf.~\cite{JaklPichlerWoltran09,FichteHecher19}.
In terms of lower bounds, our results %
naturally carry over for more expressive fragments of logic programs.
It might be still interesting to generalize our results to extended formalisms like
epistemic logic programs, for which only certain lower bound results are known~\cite{FichteHecherPfandler20,HecherMorakWoltran20}.
}

Currently, we are working on %
different treewidth-aware reductions to \SAT and further variants thereof, 
and how these variants perform in different settings (consistency vs.~counting).
Moreover, we are curious about treewidth-aware reductions to \SAT, which preserve answer sets bijectively or are modular~\cite{Janhunen06}.
We hope this work might reopen the quest to study the correspondence of treewidth and \ASP solving
similarly to~\cite{AtseriasFichteThurley11} for \SAT.
Also investigating further structural parameters ``between'' treewidth and directed variants of treewidth could lead to new insights,
since for \ASP, directed measures~\cite{BliemWoltranOrdyniak16} often do not yield efficient algorithms.
\FIX{Given the fine-grained expressiveness results for different (sub-)classes of normal programs and the resulting expressive power hierarchy~\cite{Janhunen06}, we are curious to see also studies in this direction and to which extent results might differ, when further restricting to treewidth-aware reductions. Of particular interest might be the question of whether one can devise a different hardness proof for normal \ASP and treewidth (cf., Section~\ref{sec:hardness}), such that only unary rules are used.} %

\futuresketch{
\subsection{Consistency of Head-Cycle-Free Programs}%
We can use the algorithm~$\dpa_\PRIM$ to decide the consistency
problem for head-cycle-free programs and simply specify our new local
algorithm (\PRIM) that ``transforms'' tables from one node to another.
As graph representation we use the primal graph.  The idea is to
implicitly apply along the tree decomposition the characterization of
answer sets by~\citex{LinZhao03} extended to head-cycle-free
programs~\cite{Ben-EliyahuDechter94}.
To this end, we store in table~$o(t)$ at each node~$t$ rows of the
form~$\langle I, \mathcal{P}, \sigma\rangle$.
The first position consists of an interpretation~$I$ restricted to the
bag~$\chi(t)$.  For a sequence~$\vec \tabval$, we
write~$\mathcal{I}(\vec \tabval)\eqdef \vec u_{(1)}$ to address the
\emph{interpretation part}.
The second position consists of a set~$\mathcal{P} \subseteq I$ that
represents atoms in~$I$ for which we know that they have already been
proven.
The third position~$\sigma$ is a sequence of the atoms~$I$ such that
there is a super-sequence~$\sigma'$ of~$\sigma$, which induces an ordering~$<_{\sigma'}$.
Our local algorithm~\PRIM stores interpretation parts always
restricted to bag~$\chi(t)$ and ensures that an interpretation can be extended
to a model of sub-program~$\prog_{\leq t}$.
More precisely, it guarantees that interpretation~$I$ can be extended
to a model~$I'\supseteq I$ of~$\prog_{\leq t}$ and that the atoms
in~$I'\setminus I$ (and the atoms in $\mathcal{P}\subseteq I$) have
already been \emph{proven}, using some induced ordering~$<_{\sigma'}$
where $\sigma$ is a sub-sequence of~$\sigma'$.
In the end, an interpretation~$\mathcal{I}(\vec u)$ of a row~$\vec u$
of the table~$o(n)$ at the root~$n$ proves that there is a
superset~$I' \supseteq \mathcal{I}(\vec u)$ that is an answer set
of~$\prog = \progt{n}$.

Listing~\ref{fig:prim} presents the algorithm~\PRIM.  Intuitively,
whenever an atom~$a$ is introduced ($\intr$), we decide whether we
include~$a$ in the interpretation, determine bag atoms that can be
proven in consequence of this decision, and update the
sequence~$\sigma$ accordingly.
To this end, we define 
for a given interpretation~$I$ and a sequence~$\sigma$ the set
$\gatherproof(I, \sigma, \prog_t) \eqdef \bigcup_{r\in \prog_t, a\in
  H_r}\SB a \mid B_r^+ \subseteq I, I \cap B^-_r = \emptyset, I \cap
(H_r\setminus \{a\}) = \emptyset, B^+_r <_\sigma a \SE$ where
$B^+_r <_\sigma a$ holds if $b <_\sigma a$ is true for every
$b \in B^+_r$.  Moreover, given a
level mapping~$\sigma$ and a
set~$A$ of atoms, we compute the potential level mappings
involving~$A$. Therefore, we let
$\possord(\sigma, a, J) \eqdef \SB \MAI{\sigma}{a\mapsto 0} \SM a\notin J \SE \cup
\{\{a\mapsto i, x \mapsto j \mid \sigma(x) = j, \neg sh \text{ or } j < i\} \cup \{x \mapsto j + 1 \mid \sigma(x) = j, j\geq i, sh\} \mid a\in J, 1\leq i \leq k + 1, sh \in \{1, 0\}\SE$.
When removing ($\rem$) an atom~$a$, we only keep those rows where~$a$ has
been proven (contained in~$\mathcal{P}$) and then restrict remaining rows
to the bag (not containing~$a$). In case the node is of
type~$\join$, we combine two rows in two different child tables,
intuitively, we are enforced to agree on the interpretations~$I$ and
sequences~$\sigma$. However, concerning the individual
proofs~$\mathcal{P}$, it suffices that an atom is proven in one of the
rows.

\input{algorithms/prim}%

\setlength{\tabcolsep}{0.25pt}
\renewcommand{\arraystretch}{0.75}
\begin{figure}[t]
\hspace{-0.75em}
\includegraphics[width=0.5\textwidth]{figure-phc-tables}
\caption{Selected tables of~$\tau$ obtained by~$\dpa_{\PRIM}$ on
  TD~${\cal T}$.} %
\label{fig:running2}
\end{figure}

\begin{example}\label{ex:sat}%
  Recall program~$\prog$ from
  Example~\ref{ex:running}. Figure~\ref{fig:running2} depicts a
  TD~$\TTT=(T,\chi)$ of the primal graph~$G_1$ of $\prog$. Further,
  the figure illustrates a snippet of tables of the
  TTD~$(T,\chi,\tau)$, which we obtain when running $\dpa_{\PRIM}$ on
  program~$\prog$ and TD~$\TTT$ according to Listing~\ref{fig:prim}.
  In the following, we briefly discuss some selected rows of those
  tables.
  Note that for simplicity and space reasons, we write $\tau_q$
  instead of $\tau(t_q)$ and identify rows by their node and
  identifier~$i$ in the figure. For example, the row
  $\vec\tabval_{13.3}=\langle I_{13.3}, \mathcal{P}_{13.3},
  \sigma_{13.3}\rangle\in\tab{13}$ refers to the third row of
  table~$\tab{13}$ for node~$t_{13}$. 
  Node~$t_1$ is of type~$\leaf$. Table~$\tab{1}$ has only one row,
  which consists of the empty interpretation, empty set of proven
  atoms, and the empty sequence (Line~\ref{line:primleaf}).
  Node~$t_2$ is of type~$\intr$ and introduces atom~$a$. Executing
  Line~\ref{line:primintr} results in
  $\tab{2}=\{\langle \emptyset, \emptyset,\langle \rangle \rangle,
  \langle \{a\}, \emptyset, \langle a\rangle\rangle\}$.
  Node~$t_3$ is of type~$\intr$ and introduces~$b$. Then, bag-program
  at node~$t_3$ is $\prog_{t_3}=\{a \vee b \hsep\}$.
  By construction (Line~\ref{line:primintr}) we ensure that
  interpretation~$I_{3.i}$ is a model of~$\prog_{t_3}$ for every
  row~$\langle I_{3.i}, \mathcal{P}_{3.i}, \sigma_{3.i}\rangle$
  in~$\tab{3}$.
  Node~$t_{4}$ is of type~$\rem$.  Here, we restrict the rows such
  that they contain only atoms occurring in bag~$\chi(t_4)=\{b\}$.  To
  this end, Line~\ref{line:primrem} takes only
  rows~$\vec\tabval_{3.i}$ of table~$\tab{3}$ where atoms in~$I_{3.i}$
  are also proven,~i.e., contained in~$\mathcal{P}_{3.i}$.
  In particular, every row in table~$\tab{4}$ originates from at least
  one row in~$\tab{3}$ that either proves~$a\in \mathcal{P}_{3.i}$ or
  where~$a\not\in I_{3.i}$. 
  Basic conditions of a TD ensure that once an atom is removed, it
  will not occur in any bag at an ancestor node. Hence, we also
  encountered all rules where atom~$a$ occurs.
  Nodes~$t_5, t_6, t_7$, and~$t_8$ are symmetric to
  nodes~$t_1, t_2, t_3$, and~$t_4$.
  Nodes~$t_9$ and~$t_{10}$ again introduce atoms. 
  Observe that $\mathcal{P}_{10.4} = \{e\}$ since
  $\sigma_{10.4}$ does not allow to prove~$b$ using atom~$e$.
  However, $\mathcal{P}_{10.5}=\{b,e\}$ as the sequence~$\sigma_{10.5}$
  allows to prove~$b$.
  In particular, in row~$\vec\tabval_{10.5}$ atom~$e$ is used to
  derive~$b$.  As a result, atom~$b$ can be proven, whereas
  ordering~$\sigma_{10.4}=\langle b,e\rangle$ does not serve in
  proving~$b$.
  We proceed similar for nodes~$t_{11}$ and $t_{12}$.
  At node~$t_{13}$ we join tables~$\tab{4}$ and~$\tab{12}$ according
  to Line~\ref{line:primjoin}.
  Finally, we have $\tab{14}\neq \emptyset$. Hence, $\prog$ has an
  answer set. We can construct the answer set~$\{b,e\}$ by combining
  the interpretation parts~$I$ of the yellow marked rows of
  Figure~\ref{fig:running2}.
\end{example}

Next, we provide a notion to reconstruct answer sets from a computed
TTD, which allows for computing for a given row its predecessor rows
in the corresponding child tables,~c.f., \cite{FichteEtAl18}.
\newcommand{\llangle}{\ensuremath{\langle\hspace{-2pt}\{\hspace{-0.2pt}}}
\newcommand{\rrangle}{\ensuremath{\}\hspace{-2pt}\rangle}}
\newcommand{\STab}{\ensuremath{\ATab{\AlgA}}}%
Let $\prog$ be a program, $\TTT=(T, \chi, \tau)$ be an~$\AlgA$-TTD
of~$\mathcal{G}_\prog$, and $t$ be a node of~$T$ where
$\children(t)=\langle t_1, \ldots, t_{\ell}\rangle$. %
Given a sequence~$\vec s=\langle s_1, \ldots, s_{\ell} \rangle$, we
let
$\llangle \vec s\rrangle \eqdef \langle \{s_1\}, \ldots, \{s_{\ell}\}
\rangle$. %
For a given $\AlgA$-row~$\vec u$, we define the originating
$\AlgA$-rows of~$\vec u$ in node~$t$ by
$\orig(t,\vec \tabval) \eqdef \SB \vec s \SM \vec s \in \tau(t_1)
\times \cdots \times \tau({t_\ell}), \vec u \in {\AlgA}(t,\chi(t),
\cdot,(\prog_t,\cdot), \llangle \vec s\rrangle) \SE.$ %
We extend this to an $\AlgA$-table~$\rho$ by
$\origs(t,\rho) \eqdef \bigcup_{\vec u \in \rho}\orig(t,\vec u)$.
\begin{example}\label{ex:origins} %
  Consider program~$\prog$ and $\PRIM$-tabled tree
  decomposition~$(T,\chi,\tab{})$ from Example~\ref{ex:sat}.  We focus
  on~$\vec{\tabval_{1.1}} =\langle\emptyset, \emptyset, \langle
  \rangle\rangle$ of table~$\tab{1}$ of leaf~$t_1$. The
  row~$\vec{\tabval_{1.1}}$ has no preceding row,
  since~$\type(t_1)=\leaf$. Hence, we have
  $\origse{\PRIM}(t_1,\vec{\tabval_{1.1}})=\{\langle \rangle\}$.
  The origins of row~$\vec{\tabval_{11.1}}$ of
  table~$\tab{11}$ are given by
  $\origse{\PRIM}(t_{11},\vec{\tabval_{11.1}})$, which correspond to
  the preceding rows in
  table~$\tab{10}$ that lead to
  row~$\vec{\tabval_{11.1}}$ of
  table~$\tab{11}$ when running
  algorithm~$\PRIM$,~i.e.,
  $\origse{\PRIM}(t_{11},\vec{\tabval_{11.1}}) = \{\langle
  \vec{\tabval_{10.1}} \rangle, \langle \vec{\tabval_{10.6}} \rangle,
  \langle \vec{\tabval_{10.7}}
  \rangle\}$. Origins of
  row~$\vec{\tabval}_{12.2}$ are given by
  $\origse{\PRIM}(t_{12},\vec{\tabval_{12.2}}) = \{\langle
  \vec{\tabval_{11.2}} \rangle, \langle
  \vec{\tabval_{11.6}} \rangle\}$. Note
  that~$\vec{\tabval_{11.4}}$
  and~$\vec{\tabval_{11.5}}$ are not among those origins, since
  $d$ is not proven.  Observe that
  $\origse{\PRIM}(t_j,\vec\tabval)=\emptyset$ for any
  row~$\vec\tabval\not\in\tab{j}$.
  For node~$t_{13}$ of type~$\join$ and row~$\vec{\tabval_{13.2}}$, we
  obtain
  $\origse{\PRIM}(t_{13},\vec{\tabval_{13.2}}) =
  \{\langle\vec{\tabval_{4.2}},$ $\vec{\tabval_{12.2}} \rangle, \langle\vec{\tabval_{4.2}},$ $\vec{\tabval_{12.3}} \rangle\}$.
\end{example}

\longversion{\paragraph{Table Algorithm for the Incidence Graph}

\input{algorithms/sinc}%

With the general algorithm in mind (see Figure~\ref{fig:dp-approach}), we are now ready to propose $\INC$, a
new table algorithm for solving \ASP on the semi-incidence graph (see 
Listing~\ref{fig:sinc}). 
As in the general approach, \INC computes and stores witness sets, and their
corresponding counter-witness sets. However, in addition, for each witness set
and counter-witness set, respectively, 
we need to store so-called \emph{satisfiability states}
(or \emph{sat-states}, for short), since the atoms of a rule may no longer be
contained in one single bag of the tree decomposition of the semi-incidence graph. Therefore, we
need to remember in each tree decomposition node, ``how much'' of a rule is already satisfied. The
following describes this in more detail.

By definition of tree decompositions and the semi-incidence graph, for
every atom~$a$ and every rule~$r$ of a program, it is true that if atom~$a$ occurs
in rule~$r$, then $a$ and $r$ occur together in at least one bag of
the tree decomposition. As a consequence, the table algorithm encounters every
occurrence of an atom in any rule. In the end, on removal of~$r$, we
have to ensure that $r$ is among the rules that are already
satisfied. However, we need to keep track of whether a witness satisfies a rule, because not all atoms that occur in a rule
occur together in a bag. Hence, when our algorithm traverses
the tree decomposition and an atom is removed we still need to store this
sat-state, as setting the removed atom to a certain truth
value influences the satisfiability of the rule.
Since the semi-incidence graph contains a clique on every set~$A$ of
atoms that occur together in a weight rule body or
choice rule head, %
those atoms~$A$ occur together in a bag in every tree decomposition of the
semi-incidence graph. For that reason, we do {not} need to
incorporate weight or 
choice rules into the satisfiability state, in
contrast to the table algorithm for the incidence graph discussed later~(c.f. Section~\ref{sec:inc}).

In algorithm~\INC (detailed in Listing~\ref{fig:sinc}), a tuple in the
table~$\tab{t}$ is a triple~$\langle M, \sigma, \CCC \rangle$.  The
set~$M \subseteq \at(\prog)\cap\chi(t)$ represents a witness set. 
The family~$\CCC$ of sets represents counter-witnesses, which we will
discuss in more detail below.
The sat-state~$\sigma$ for $M$ represents rules of $\chi(t)$ satisfied
by a superset of~$M$.  Hence, $M$ witnesses a model~$M'\supseteq M$
where $M' \models \progtneq{t} \cup \sigma$.  
We use the binary operator~$\cup$ to combine sat-states, which ensures
that rules satisfied in at least one operand remain satisfied. For a node~$t$, our algorithm considers a local-program depending on the bag~$\chi(t)$. Intuitively, this provides a local view on the program.

For a node~$t$, our algorithm considers a local-program depending on the bag~$\chi(t)$. Intuitively, this provides a local view on the program.
\begin{definition}\label{def:bagprogram}%
  Let $\prog$ be a program, $\TTT=(\cdot,\chi)$ a tree decomposition of $S(\prog)$,
  $t$ a node of $\TTT$ and ${R} \subseteq \prog_t$.
  The \emph{local-program mapping}~${R}^{(t)}: R \rightarrow prog(\at(R)\cap\chi(t))$ assigns to each rule~$r\in R$ a rule obtained from~$r$
  by %
  removing all
    literals~$a$ and $\neg a$ where $a \not\in \chi(t)$.
  \shortversion{}%
\end{definition}%

\begin{example}
  Observe
  $\prog_{t_1}^{(t_1)} = \{(r_1, b \hsep \neg a), (r_2, a \hsep \neg b)\}$ and
  $\prog_{t_2}^{(t_2)} = \{(r_1, b \hsep \neg a), (r_2, a \hsep \neg b), (r_3, d\hsep)\}$ for $\prog_{t_1}$, $\prog_{t_2}$ of
  Figure~\ref{fig:graph-td2}. %
\end{example}

Since the local-program mapping~$\prog^{(t)}$ depends on the considered
node~$t$, we may have different local-program mappings for node~$t$ and
its child~$t'$. In particular, the mappings~$\{r\}^{(t)}$ and
$\{r\}^{(t')}$ might already differ for a
rule~$r \in \chi(t) \cap \chi(t')$. In consequence for satisfiability
with respect to sat-states, we need to keep track of a representative
of a rule. 

$\SP(\dot{R}, \sigma) \eqdef \{ a \mid (r, s) \in \dot{R}, a \in H_s, a >_\sigma r\}$

$\checkord(\dot{R}, \sigma, \phi, a) \eqdef \text{true iff } a >_\sigma r \implies r\not\in\phi \text{ for any } (r,s) \in\dot{R}$ where $a\in\at(s)$

$\checkmod(\dot{R}, J, \sigma) \eqdef \text{true iff } J \cap \at(s)  = B_s^+ \cup X$, $\Card{X}\leq 1$, and~$X\subseteq H_s$ where $X = \{a\mid a \in \at(s), a >_\sigma r\} \text{ for any } (r,s) \in\dot{R} \text{ with } \sigma = \langle \ldots, r, \ldots \rangle$

Note that in the end~$X \cap H_s\neq \emptyset$ in at least some bag,
since otherwise the rule is not satisfied anyway, but we also have to
enforce that no body atom can in~$X$!}

Next, we provide statements on correctness and a runtime analysis of
our algorithm.

\begin{theorem}[$\star$\footnote{Proofs of statements marked with ``$\star$'' can be found in the supplemental material.}]\label{thm:primcorrectness}
  The algorithm~$\dpa_\PRIM$ is correct. \\
  In other words, given a head-cycle-free program~$\prog$ and a
  TTD~${\cal T} = (T,\chi,\cdot)$ of~$\mathcal{G}_\prog$
  where~$T=(N,\cdot,n)$ with root~$n$. Then,
  algorithm~$\dpa_\PRIM((\prog,\cdot),\TTT)$ returns the
  $\PRIM$-TTD~$(T,\chi,\tau)$ such that $\prog$ has an answer set if
  and only if
  $\langle \emptyset, \emptyset, \langle
  \rangle\rangle\in\tau(n)$. Further, we can construct all the answer
  sets of~$\prog$ from transitively following the origins
  of~$\tau(n)$.\footnoteitext{\label{foot:nu}$\nu$ contains rows
    obtained by recursively following origins of~$\tau(n)$. Formal details are in Def.~\ref{def:extensions} (supplemental material).}\footnoteitext{\label{foot:phc}Later we use (among others)~\mdpa{\PRIM} where~$\AlgA=\PRIM$.}
\end{theorem}
\begin{proof}[Proof (Idea).]
  For soundness, we state an invariant and establish that this
  invariant holds for every node~$t\in N$.  For each
  row~$\vec\tabval=\langle I, \mathcal{P}, \sigma\rangle\in\tau(t)$,
  we have~$I\subseteq\chi(t), \mathcal{P}\subseteq I$, and~$\sigma$ is
  a sequence over atoms in~$I$. Intuitively, we ensure that
  $I\models\progt{t}$ and that exactly the atoms in~$\attneq{t}$
  and~$\mathcal{P}$ can be proven using a super-sequence~$\sigma'$
  of~$\sigma$.  By construction, we guarantee that we can decide
  consistency if
  row~$\langle \emptyset, \emptyset, \langle
  \rangle\rangle\in\tau(n)$. Further, we can even reconstruct answer
  sets, by following $\origa{\PRIM}$ of this single row back to the
  leaves.
  For completeness, we show that we obtain all the rows required to
  output all the answer sets of~$\prog$.
\end{proof}

\begin{theorem}%
  \label{thm:primruntime}
  Given a head-cycle-free program~$\prog$ and a tree
  decomposition~${\cal T} = (T,\chi)$ of~$\mathcal{G}_\prog$ of width~$k$ with $g$
  nodes. Algorithm~$\dpa_{\PRIM}$ runs in time
  $\mathcal{O} (3^{k}\cdot k !  \cdot g)$.
\end{theorem}
\begin{proof}[Proof (Sketch).]
  Let~$d = k+1$ be maximum bag size of the tree
  decomposition~$\TTT$. %
  The table~$\tau(t)$ has at most
  $3^{d} \cdot d!$ rows, since for a row~$\langle I, \mathcal{P}, \sigma\rangle$ we have~$d!$ many sequences~$\sigma$, and by construction of algorithm~$\PRIM$, an atom can be either in~$I$, both in~$I$ and~$\mathcal{P}$, or neither in~$I$ nor in~$\mathcal{P}$.
  In total, with the help of efficient data structures, e.g., for nodes~$t$ with~$\type(t)=\join$, one can establish a runtime bound of~$\bigO{{3^{d}\cdot d!}}$.
  Then, we apply this to every node~$t$ of the tree decomposition,
  which results in running
  time~$\bigO{{3^{d}\cdot d!} \cdot g}\subseteq \bigO{3^{k}\cdot k!\cdot g}$.
  Consequently, the theorem holds.
\end{proof}

In order to obtain an upper bound on factorial, we can simply take
$k! \leq 2^k$ for any fixed~$k\geq 4$. However, more precisely the
factorial is asymptotically bounded as follows.

\begin{proposition}[$\star$]\label{prop:kfact}
  Given any positive integer~$i \geq 1$ and
  functions~$f(k)\eqdef k!, g(k) \eqdef 2^{k^{(i+1)/i}}$. Then,
  $f \in \bigO{g}$.
\end{proposition}

A natural question is whether we can significantly improve this
algorithm for fixed~$k$.  To this end, we take the \emph{exponential
  time hypothesis (ETH)} into account, which states that there is some
real~$s > 0$ such that we cannot decide satisfiability of a given
3-CNF formula~$F$ in
time~$2^{s\cdot\Card{F}}\cdot\CCard{F}^{\mathcal{O}(1)}$.

\begin{proposition}
  Unless ETH fails, consistency of head-cycle-free program~$\prog$
  cannot be decided in time~$2^{o(k)} \cdot \CCard{\prog}^{o(k)}$
  where~$k$ is the treewidth of the primal graph of~$\prog$.
\end{proposition}
\begin{proof}
  The result follows by reducing from SAT to ASP (head-cycle-free)
  similar to the proof of Proposition~\ref{prop:hcfproj}.
\end{proof}

In the construction above, we store an arbitrary but fixed ordering on
the involved atoms. We believe that we cannot avoid these orderings in
general, since we have to compensate arbitrarily ``bad'' orderings
induced by the decomposition, which leads us to the following
conjecture.

\begin{conjecture}
  Unless ETH fails, consistency of a head-cycle-free program~$\prog$
  cannot be decided in
  time~$2^{o(k\cdot \text{log}(k))} \cdot \CCard{\prog}^{o(k)}$
  where~$k$ is the treewidth of the primal graph of~$\prog$.
\end{conjecture}

In other words, we claim that consistency for head-cycle-free programs
is slightly superexponential. We would like to mention
that~\citex{LokshtanovMarxSaurabh11} argue that whenever we cannot
avoid an ordering the problem is expected to be slightly
superexponential.
If the conjecture holds, our algorithm is asymptotically worst-case
optimal, even for fixed treewidth~$k$ since~$\dpa_{\PRIM}$ runs in
time~$\mathcal{O}(2^{k\cdot \text{log}(k)}\cdot g)$, where number~$g$
of decomposition nodes is linear in the size of the
instance~\cite{Bodlaender96}.

\section{Dynamic Programming for~$\PASP$}%

\label{sec:projmodelcounting}

\begin{figure}[t]
\centering
\includegraphics[scale=0.8]{figure_projection.pdf}
\caption{Algorithm~$\mdpa{\AlgA}$ consists of~$\dpa_\AlgA$
  and~$\dpa_\PROJ$. %
}
\label{fig:multiarch}
\end{figure}%

In this section, we present our dynamic programming
algorithm$^{\ref{foot:phc}}$~\mdpa{\AlgA}, which allows for solving the projected answer
set counting problem (\PASP).
\mdpa{\AlgA} is based on an approach of projected counting for propositional
formulas~\cite{FichteEtAl18} where TDs are traversed multiple times.
We show that ideas from that approach can be fruitfully extended to
answer set programming.
Figure~\ref{fig:multiarch} illustrates the steps of \mdpa{\AlgA}.
First, we construct the primal graph~$\mathcal{G}_\prog$ of the input program~$\prog$
and compute a TD of $\prog$. Then, we traverse the TD a first time by
running $\dpa_\AlgA$ (Step~3a), which outputs a
TTD~$\TTT_{\text{cons}}=(T,\chi,\tau)$.
Afterwards, we traverse $\TTT_{\text{cons}}$ in pre-order and remove
all rows from the tables %
that cannot be extended to an answer set (\emph{``Purge
  non-solutions''}).
In other words, we keep only rows~$\vec u$ of table~$\tau(t)$ at
node~$t$, if~$\vec u$ is involved in those rows that are used to
construct an answer set of~$\prog$, and let the resulting TTD
be~$\TTT_{\text{purged}}=(T,\chi,\nu)^{\ref{foot:nu}}$. We refer to $\nu$
as~\emph{purged table mapping}.
In Step~3b ($\dpa_\PROJ$), we traverse $\TTT_{\text{purged}}$ to count
interpretations with respect to the projection atoms and obtain
$\TTT_{\text{proj}}=(T,\chi,\pi)$. From the table~$\pi(n)$ at the root~$n$ of
$T$, we can then read the projected answer sets count of the input instance.
In the following, we only describe the local algorithm ($\PROJ$),
since the traversal in $\dpa_\PROJ$ is the same as before.
For \PROJ, %
a row at a node~$t$ is a pair $\langle\rho, c \rangle\in\pi(t)$ where
$\rho \subseteq \nu(t)$ is an $\AlgA$-table and $c$ is a non-negative
integer.
In fact, integer~$c$ stores the number of intersecting solutions
($\ipmc$). However, we aim for the projected answer sets count
($\pmc$), whose computation requires a few additional
definitions. Therefore, we can simply widen definitions from very
recent work~\cite{FichteEtAl18}.

In the remainder, %
we assume~$(\prog, P)$ to be an instance of~\PASP, $(T, \chi, \tau)$
to be an~$\AlgA$-TTD of~$\mathcal{G}_\prog$ and the mappings~$\tau$, $\nu$, and
$\pi$ as used above. 
Further, let~$t$ be a node of~$T$ with~$\children(t)=\langle t_1, \ldots, t_\ell\rangle$ and let $\rho \subseteq \nu(t)$.
\newcommand{\RRR}{\ensuremath{\mathcal{R}}}
  The relation~$\bucket \subseteq \rho \times \rho$ considers
  equivalent rows with respect to the projection of its
  interpretations by %
  $\bucket \eqdef \SB (\vec u,\vec v) \SM \vec u, \vec v \in \rho,
  \restrict{\mathcal{I}(\vec u)}{P} = \restrict{\mathcal{I}(\vec
    v)}{P}\SE.$
  Let $\buckets_P(\rho)$ be the set of equivalence classes induced
  by~$\bucket$ on~$\rho$,~i.e.,
  $\buckets_P(\rho) \eqdef\, (\rho / \bucket) = \SB [\vec u]_P \SM
  \vec u \in \rho\SE$, where
  $[\vec u]_P = \SB \vec v \SM \vec v \bucket \vec u,\vec v \in
  \rho\}$~\cite{Wilder12a}.
  Further, 
  $\subbuckets_P(\rho) \eqdef \SB S \SM \emptyset \subsetneq S
  \subseteq B, B \in \buckets_P(\rho)\SE$.

\begin{example}\label{ex:equiv} %
  Consider program~$\prog$, set~$P$ of projection atoms,
  TTD~$(T,\chi, \tau)$, and table~$\tab{10}$ from
  Example~\ref{ex:running0} and Figure~\ref{fig:running2}.
  Note that during purging rows~$\vec {u_{10.2}}$ and
  $\vec {u_{10.8}}, \ldots, \vec {u_{10.13}}$ are removed (highlighted gray),
  since they are not involved in any answer set, resulting in table~$\nu_{10}$.
  Then, $\vec{ u_{10.4}} =_P \vec{ u_{10.5}}$ and
  $\vec{ u_{10.6}} =_P \vec{ u_{10.7}}$.  The set~$\nu_{10}/\bucket$
  of equivalence classes of $\nu_{10}$
  is~$\buckets_P(\nu_{10})=\SB \{\vec{ u_{10.1}}\}, \{\vec{ u_{10.3}}\}, \{\vec{ u_{10.4}}, \vec{ u_{10.5}}\}, \{\vec{
    u_{10.6}}, \vec{ u_{10.7}}\}\SE$.
\end{example}

Later, we require to construct already computed projected counts for
tables of children of a given node~$t$. Therefore, we define the
\emph{stored $\ipmc$} of a table~$\rho \subseteq \nu(t)$ in
table~$\pi(t)$ by
$\sipmc(\pi(t), \rho) \eqdef \sum_{\langle \rho, c\rangle \in \pi(t)}
c.$ 
We extend this to a
sequence~$s=\langle \pi(t_1), \ldots, \pi(t_\ell)\rangle$ of tables of
length $\ell$ and a
set~$O = \{\langle \rho_1, \ldots, \rho_\ell\rangle, \langle \rho_1',
\ldots, \rho_\ell'\rangle, \ldots\}$ of sequences of~$\ell$ tables by
$\sipmc(s, O)=\prod_{i \in \{1, \ldots,
  \ell\}}\sipmc(s_{(i)},O_{(i)}).$
In other words, we select the $i$-th position of the sequence together
with sets of the $i$-th positions from the set of sequences.

Intuitively, when we are at a node~$t$ in algorithm~$\dpa_\PROJ$ we
have already computed~$\pi(t')$ of $\TTT_{\text{proj}}$ for every node~$t'$
below~$t$.
Then, we compute the projected answer sets count
of~$\rho \subseteq \nu(t)$. Therefore, we apply the
inclusion-exclusion principle to the stored projected answer sets count
of origins.
We define $\pcnt(t,\rho, \langle\pi(t_1),\ldots\rangle) \eqdef %
\sum_{\emptyset \subsetneq O \subseteq {\origs(t,\rho)}}
(-1)^{(\Card{O} - 1)} \cdot \sipmc(\langle \pi(t_1), \ldots\rangle, O)$. %
Vaguely speaking, $\pcnt$ determines the $\AlgA$-origins of table~$\rho$, 
goes over all subsets of these origins and looks up the
stored counts ($\sipmc$) in the \PROJ-tables of the children~$t_i$ of~$t$.

\begin{example}\label{ex:pcnt} %
  Consider again program~$\prog$ and TD~$\TTT$ from
  Example~\ref{ex:running1} and Figure~\ref{fig:running2}. First, we
  compute the projected count $\pcnt(t_{4},\{\vec{ u_{4.1}}\}, \langle\pi(t_{3})\rangle)$
  for row~$\vec{ u_{4.1}}$ of table~$\nu(t_{4})$ where
  $\pi(t_3) \eqdef\allowdisplaybreaks[4] \big\SB
  \langle \{\vec{ u_{3.1}}\}, 1\rangle,$
  $\langle \{\vec{ u_{3.2}}\},1\rangle, \langle \{\vec{u_{3.1}}, \vec{
    u_{3.2}}\},1\rangle\big\SE$ with
  $\vec{u_{3.1}}=\langle \emptyset, \emptyset, \langle\rangle \rangle$
  and~$\vec{u_{3.2}}=\langle \{a\}, \emptyset, \langle a\rangle
  \rangle$.
  Note that~$t_5$ has only the child~$t_4$ and therefore the product in~$\sipmc$
  consists of only one factor. 
  Since
  $\origse{\PRIM}(t_4, \vec{ u_{4.1}}) = \{\langle\vec{
    u_{3.1}}\rangle\}$, only the value of~$\sipmc$ for
  set~$\{\langle\vec{ u_{3.1}}\rangle\}$ is non-zero. Hence, we obtain
  $\pcnt(t_4,\{\vec{ u_{4.1}}\}, \langle\pi(t_3)\rangle)=1$. 
  Next, we compute
  $\pcnt(t_{4},\{\vec{ u_{4.1}}, \vec{u_{4.2}}\}, \langle\pi(t_3)\rangle)$. Observe that
  $\origse{\PRIM}(t_4, \{\vec{ u_{4.1}}, \vec{ u_{4.2}}\}) =
  \{\langle\vec{ u_{3.1}}\rangle, \langle\vec{ u_{3.2}}\rangle\}$. We
  sum up the values of~$\sipmc$ for sets~$\{\vec{ u_{4.1}}\}$
  and~$\{\vec{ u_{4.2}}\}$ and subtract the one for
  set~$\{\vec{ u_{4.1}}, \vec{ u_{4.2}}\}$.  Hence, we obtain
  $\pcnt(t_4,\{\vec{ u_{4.1}}, \vec{ u_{4.2}}\}, \langle\pi(t_3)\rangle)=1+1-1=1$.
\end{example}

Next, we provide a definition to compute $\ipmc$, which can be
computed at a node~$t$ for given table~$\rho\subseteq \nu(t)$ by
computing the $\pmc$ for children~$t_i$ of~$t$ using stored $\ipmc$
values from tables~$\pi(t_i)$, subtracting and adding~$\ipmc$ values
for subsets~$\emptyset\subsetneq\varphi\subsetneq\rho$ accordingly.
Formally, $\icnt(t,\rho,s)\eqdef 1$ if $\type(t) = \leaf$ and
otherwise
$\icnt(t,\rho,s)\eqdef \big|\pcnt(t,\rho, s)\;+
\quad\sum_{\emptyset\subsetneq\varphi\subsetneq\rho}(-1)^{\Card{\varphi}}
\cdot \ipmc(t,\varphi, s)\big|$ where
$s = \langle \pi(t_1), \ldots\rangle$.
In other words, if a node is of type~$\leaf$ the $\ipmc$ is one, since
bags of leaf nodes are empty.
Otherwise, we compute the ``non-overlapping'' count of given
table~$\rho\subseteq\nu(t)$ with respect to~$P$, by exploiting the
inclusion-exclusion principle on $\AlgA$-origins of~$\rho$ such that
we count every projected answer set only once. Then we have to %
subtract and add $\ipmc$ values (``all-overlapping'' counts) for
strict subsets~$\varphi$ of~$\rho$, accordingly.

Finally, Listing~\ref{fig:dpontd3} presents the local algorithm~\PROJ,
which stores~$\pi(t)$ consisting of every sub-bucket of the given
table~$\nu(t)$ together with its $\ipmc$.

\input{algorithms/dponpass3}%

\setlength{\tabcolsep}{0.25pt}
\renewcommand{\arraystretch}{0.75}
\begin{figure*}[t]
\centering
\begin{tikzpicture}[node distance=0.5mm]
\tikzset{every path/.style=thick}

\node (l1) [stdnode,label={[tdlabel, xshift=0em,yshift=+0em]right:${t_1}$}]{$\emptyset$};
\node (i1) [stdnode, above=of l1, label={[tdlabel, xshift=0em,yshift=+0em]left:${t_2}$}]{$\{a\}$};
\node (i12) [stdnode, above=of i1, label={[tdlabel, xshift=0em,yshift=+0.35em]left:${t_3}$}]{$\{a,b\}$};
\node (i13) [stdnode, above=of i12, label={[tdlabel, xshift=+0.05em,yshift=+0em]right:${t_4}$}]{$\{b\}$};
\node (lrx) [stdnode, right=2.5em of l1, yshift=-6.9em, label={[tdlabel, xshift=0em,yshift=+0em]left:${t_5}$}]{$\emptyset$};
\node (cc) [stdnode, above=of lrx, label={[tdlabel, xshift=0em,yshift=+0em]left:${t_6}$}]{$\{c\}$};
\node (bc) [stdnode, above=of cc, label={[tdlabel, xshift=0em,yshift=+0em]left:${t_7}$}]{$\{c,e\}$};
\node (l2) [stdnode, above=of bc, label={[tdlabel, xshift=0em,yshift=+0em]left:${t_8}$}]{$\{e\}$};
\node (i2) [stdnode, above=of l2, label={[tdlabel, xshift=0em,yshift=+0em]right:${t_{9}}$}]{$\{d, e\}$};
\node (i22) [stdnode, above=of i2, label={[tdlabel, xshift=0em,yshift=+0em]left:${t_{10}}$}]{$\{b,d,e\}$};
\node (r2) [stdnode, above=of i22, label={[tdlabel, xshift=0em,yshift=+0em]left:${t_{11}}$}]{$\{b, d\}$};
\node (r22) [stdnode, above=of r2, label={[tdlabel, xshift=0.05em,yshift=+0em]left:${t_{12}}$}]{$\{b\}$};
\node (j) [stdnode, above left=of r22, xshift=0.0em, yshift=-0.0em, label={[tdlabel, xshift=0em,yshift=+0.3em]right:${t_{13}}$}]{$\{b\}$};
\node (rt) [stdnode,ultra thick, above=of j, label={[tdlabel, xshift=0em,yshift=+0em]right:${t_{14}}$}]{$\emptyset$};
\node (label) [font=\scriptsize,left=of rt,xshift=0.45em]{${\cal T}$:};
\node (leaf1) [stdnodetable, left=3em of i1, yshift=2em, label={[tdlabel, xshift=2em,yshift=1em]below right:$\pi_{3}$}]{%
	\begin{tabular}{l@{\hspace{0.0em}}l@{\hspace{0.0em}}r}%
		\multicolumn{1}{l}{$\langle \tuplecolor{black}{\nu_{3.i}}, $}&\multicolumn{1}{r}{$\tuplecolor{\specialPredColor}{c_{3.i}} \rangle$}\\
		\hline\hline
		$\langle\tuplecolor{black}{\{\langle\tuplecolor{\inputPredColor}{\{a\}}, \tuplecolor{\outputPredColor}{\{a\}}, \tuplecolor{\statePredColor}{\langle a\rangle}\rangle\}}, $&$\tuplecolor{\specialPredColor}{1}\rangle$ \\\hline
		$\langle\tuplecolor{black}{\{\langle\tuplecolor{\inputPredColor}{\{b\}}, \tuplecolor{\outputPredColor}{\{b\}}, \tuplecolor{\statePredColor}{\langle b\rangle}\rangle\}}, $&$\tuplecolor{\specialPredColor}{1}\rangle$ \\\hline
		$\langle\tuplecolor{black}{\{\langle\tuplecolor{\inputPredColor}{\{a\}}, \tuplecolor{\outputPredColor}{\{a\}}, \tuplecolor{\statePredColor}{\langle a\rangle}\rangle}, $& \multirow{2}{*}{$\tuplecolor{\specialPredColor}{1}\rangle$} \\ %
		$\hspace{1.05em}\tuplecolor{black}{\langle\tuplecolor{\inputPredColor}{\{b\}}, \tuplecolor{\outputPredColor}{\{b\}}, \tuplecolor{\statePredColor}{\langle b\rangle}\rangle\},}$ %
	\end{tabular}%
};
\node (leaf1b) [stdnodenum,left=of leaf1,xshift=0.6em,yshift=+0.0em]{%
	\begin{tabular}{c}%
		\multirow{1}{*}{$i$}\\ %
		\hline\hline
		$1$ \\\hline
		$2$ \\\hline
		\multirow{2}{*}{$3$}\\\\ %
	\end{tabular}%
};
\node (leaf0x) [stdnodetable, left=-3.5em of leaf1b, yshift=4.5em, label={[tdlabel, xshift=2em,yshift=-0.15em]below left:$\pi_{4}$}]{%
	\begin{tabular}{l@{\hspace{0.0em}}l@{\hspace{0.0em}}r}%
		\multicolumn{1}{l}{$\langle \tuplecolor{black}{\nu_{4.i}}, $}&\multicolumn{1}{r}{$\tuplecolor{\specialPredColor}{c_{4.i}} \rangle$}\\
		\hline\hline
		$\langle\tuplecolor{black}{\{\langle\tuplecolor{\inputPredColor}{\emptyset}, \tuplecolor{\outputPredColor}{\emptyset}, \tuplecolor{\statePredColor}{\langle \rangle}\rangle\}}, $&$\tuplecolor{\specialPredColor}{1}\rangle$ \\\hline
		$\langle\tuplecolor{black}{\{\langle\tuplecolor{\inputPredColor}{\{b\}}, \tuplecolor{\outputPredColor}{\{b\}}, \tuplecolor{\statePredColor}{\langle b\rangle}\rangle\}}, $&$\tuplecolor{\specialPredColor}{1}\rangle$ \\\hline
		$\langle\tuplecolor{black}{\{\langle\tuplecolor{\inputPredColor}{\emptyset}, \tuplecolor{\outputPredColor}{\emptyset}, \tuplecolor{\statePredColor}{\langle \rangle}\rangle}, $& \multirow{2}{*}{$\tuplecolor{\specialPredColor}{1}\rangle$} \\ %
		$\hspace{1.05em}\tuplecolor{black}{\langle\tuplecolor{\inputPredColor}{\{b\}}, \tuplecolor{\outputPredColor}{\{b\}}, \tuplecolor{\statePredColor}{\langle b\rangle}\rangle\},}$ %
	\end{tabular}%
};
\node (leaf0b) [stdnodenum,left=of leaf0x,xshift=0.6em,yshift=0pt]{%
	\begin{tabular}{c}%
		\multirow{1}{*}{$i$}\\ %
		\hline\hline
		$1$ \\\hline
		$2$ \\\hline
		\multirow{2}{*}{$3$}\\\\ %
	\end{tabular}%
};
\node (leaf2b) [stdnodenum,left=6.5em of j,xshift=-0.25em,yshift=-11.5em]  {%
	\begin{tabular}{c}%
		\multirow{1}{*}{$i$}\\ %
		\hline\hline
		$1$\\\specialrule{.1em}{.05em}{.05em}	%
		$2$\\\specialrule{.1em}{.05em}{.05em}	%
		$3$\\\hline
		$4$\\\hline
		\multirow{2}{*}{$5$}\\\\ %
	\end{tabular}%
};
\node (leaf2) [stdnodetable,right=-3em of leaf2b, label={[tdlabel, xshift=0.1em,yshift=-0.25em]right:$\pi_{9}$}]  {%
	\begin{tabular}{l@{\hspace{0.0em}}l@{\hspace{0.0em}}r}%
		\multicolumn{1}{l}{$\langle \tuplecolor{black}{\nu_{9.i}}, $}&\multicolumn{1}{r}{$\tuplecolor{\specialPredColor}{c_{9.i}} \rangle$}\\
		\hline\hline
		$\langle\tuplecolor{black}{\{\langle\tuplecolor{\inputPredColor}{\{d\}}, \tuplecolor{\outputPredColor}{\emptyset}, \tuplecolor{\statePredColor}{\langle \rangle}\rangle\}}, $&$\tuplecolor{\specialPredColor}{1}\rangle$ \\\specialrule{.1em}{.05em}{.05em}	
		$\langle\tuplecolor{black}{\{\langle\tuplecolor{\inputPredColor}{\{e\}}, \tuplecolor{\outputPredColor}{\{e\}}, \tuplecolor{\statePredColor}{\langle e\rangle}\rangle\}}, $&$\tuplecolor{\specialPredColor}{1}\rangle$ \\\specialrule{.1em}{.05em}{.05em}	
		$\langle\tuplecolor{black}{\{\langle\tuplecolor{\inputPredColor}{\{d,e\}}, \tuplecolor{\outputPredColor}{\{e\}}, \tuplecolor{\statePredColor}{\langle d,e\rangle}\rangle\}}, $&$\tuplecolor{\specialPredColor}{1}\rangle$ \\\hline
		$\langle\tuplecolor{black}{\{\langle\tuplecolor{\inputPredColor}{\{d,e\}}, \tuplecolor{\outputPredColor}{\{e\}}, \tuplecolor{\statePredColor}{\langle e,d\rangle}\rangle\}}, $&$\tuplecolor{\specialPredColor}{1}\rangle$ \\\hline
		$\langle\tuplecolor{black}{\{\langle\tuplecolor{\inputPredColor}{\{d,e\}}, \tuplecolor{\outputPredColor}{\{e\}}, \tuplecolor{\statePredColor}{\langle d,e\rangle}\rangle}, $& \multirow{2}{*}{$\tuplecolor{\specialPredColor}{1}\rangle$} \\ %
		$\hspace{1.05em}\tuplecolor{black}{\langle\tuplecolor{\inputPredColor}{\{d,e\}}, \tuplecolor{\outputPredColor}{\{e\}}, \tuplecolor{\statePredColor}{\langle e,d\rangle}\rangle\},}$ %
	\end{tabular}%
};
\node (joinrb2) [stdnodenum,left=-0.45em of leaf2] {%
	\begin{tabular}{c}
		\multirow{1}{*}{$i$}\\
		\hline\hline
		$1$ \\\specialrule{.1em}{.05em}{.05em}	
		$2$ \\\specialrule{.1em}{.05em}{.05em}	
		$3$ \\\hline
		$4$ \\\hline
		\multirow{2}{*}{$5$} \\\\
	\end{tabular}%
};
\coordinate (middle) at ($ (leaf1.north east)!.5!(leaf2.north west) $);
\node (join) [stdnodetable,left=-0.2em of i13, yshift=4.5em, label={[tdlabel, xshift=0.1em,yshift=-0.15em]below right:$\pi_{{13}}$}] {%
	\begin{tabular}{l@{\hspace{0.0em}}l@{\hspace{0.0em}}r}%
		\multicolumn{1}{l}{$\langle \tuplecolor{black}{\nu_{13.i}}, $}&\multicolumn{1}{r}{$\tuplecolor{\specialPredColor}{c_{13.i}} \rangle$}\\
		\hline\hline
		$\langle\tuplecolor{black}{\{\langle\tuplecolor{\inputPredColor}{\emptyset}, \tuplecolor{\outputPredColor}{\emptyset}, \tuplecolor{\statePredColor}{\langle \rangle}\rangle\}}, $&$\tuplecolor{\specialPredColor}{2}\rangle$ \\\hline
		$\langle\tuplecolor{black}{\{\langle\tuplecolor{\inputPredColor}{\{b\}}, \tuplecolor{\outputPredColor}{\{b\}}, \tuplecolor{\statePredColor}{\langle b\rangle}\rangle\}}, $&$\tuplecolor{\specialPredColor}{2}\rangle$ \\\hline
		$\langle\tuplecolor{black}{\{\langle\tuplecolor{\inputPredColor}{\emptyset}, \tuplecolor{\outputPredColor}{\emptyset}, \tuplecolor{\statePredColor}{\langle \rangle}\rangle}, $& \multirow{2}{*}{$\tuplecolor{\specialPredColor}{1}\rangle$} \\ %
		$\hspace{1.05em}\tuplecolor{black}{\langle\tuplecolor{\inputPredColor}{\{b\}}, \tuplecolor{\outputPredColor}{\{b\}}, \tuplecolor{\statePredColor}{\langle b\rangle}\rangle\},}$ %
	\end{tabular}%
};
\node (joinb) [stdnodenum,left=-0.45em of join] {%
	\begin{tabular}{c}
		\multirow{1}{*}{$i$}\\
		\hline\hline
		$1$ \\\hline
		$2$ \\\hline
		\multirow{2}{*}{$3$}\\\\
	\end{tabular}%
};
\node (leaf0) [stdnodetable,below=-1.5em of l1, xshift=-5em, label={[tdlabel, xshift=0.1em,yshift=0.15em]right:$\pi_{1}$}] {%
	\begin{tabular}{l@{\hspace{0.0em}}l@{\hspace{0.0em}}r}%
		\multicolumn{1}{l}{$\langle \tuplecolor{black}{\nu_{1.i}}, $}&\multicolumn{1}{r}{$\tuplecolor{\specialPredColor}{c_{1.i}} \rangle$}\\
		\hline\hline
		$\langle\tuplecolor{black}{\{\langle\tuplecolor{\inputPredColor}{\emptyset}, \tuplecolor{\outputPredColor}{\emptyset}, \tuplecolor{\statePredColor}{\langle\rangle}\rangle\}}, $&$\tuplecolor{\specialPredColor}{1}\rangle$
	\end{tabular}%
};
\node (leaf0r) [stdnodetable,above=0.5em of rt, xshift=1.1em, label={[tdlabel, xshift=0.1em,yshift=-0.1em]below right:$\pi_{14}$}] {%
	\begin{tabular}{l@{\hspace{0.0em}}l@{\hspace{0.0em}}r}%
		\multicolumn{1}{l}{$\langle \tuplecolor{black}{\nu_{14.i}}, $}&\multicolumn{1}{r}{$\tuplecolor{\specialPredColor}{c_{14.i}} \rangle$}\\
		\hline\hline
		$\langle\tuplecolor{black}{\{\langle\tuplecolor{\inputPredColor}{\emptyset}, \tuplecolor{\outputPredColor}{\emptyset}, \tuplecolor{\statePredColor}{\langle\rangle}\rangle\}}, $&$\tuplecolor{\specialPredColor}{3}\rangle$
	\end{tabular}%
};\node (leaf0nr) [stdnodenum,yshift=0.0em, left=-0.5em of leaf0r] {%
	\begin{tabular}{c}%
		\multirow{1}{*}{$i$}\\ %
		\hline\hline
		$1$
	\end{tabular}%
};
\node (leaf0n) [stdnodenum,yshift=0.0em, left=-0.5em of leaf0] {%
	\begin{tabular}{c}%
		\multirow{1}{*}{$i$}\\ %
		\hline\hline
		$1$
	\end{tabular}%
};
\node (joinrrt) [stdnodetable,right=6.5em of i13, yshift=3em, label={[tdlabel, xshift=0.1em,yshift=+0.25em]right:$\pi_{{12}}$}] {%
		\begin{tabular}{l@{\hspace{0.0em}}l@{\hspace{0.0em}}r}%
		\multicolumn{1}{l}{$\langle \tuplecolor{black}{\nu_{12.i}}, $}&\multicolumn{1}{r}{$\tuplecolor{\specialPredColor}{c_{12.i}} \rangle$}\\
		\hline\hline
		$\langle\tuplecolor{black}{\{\langle\tuplecolor{\inputPredColor}{\emptyset}, \tuplecolor{\outputPredColor}{\emptyset}, \tuplecolor{\statePredColor}{\langle \rangle}\rangle\}}, $&$\tuplecolor{\specialPredColor}{2}\rangle$ \\\hline
		$\langle\tuplecolor{black}{\{\langle\tuplecolor{\inputPredColor}{\{b\}}, \tuplecolor{\outputPredColor}{\emptyset}, \tuplecolor{\statePredColor}{\langle b \rangle}\rangle\}}, $&$\tuplecolor{\specialPredColor}{2}\rangle$ \\\hline
		$\langle\tuplecolor{black}{\{\langle\tuplecolor{\inputPredColor}{\{b\}}, \tuplecolor{\outputPredColor}{\{b\}}, \tuplecolor{\statePredColor}{\langle b \rangle}\rangle\}}, $&$\tuplecolor{\specialPredColor}{1}\rangle$ \\\hline
		$\langle\tuplecolor{black}{\{\langle\tuplecolor{\inputPredColor}{\emptyset}, \tuplecolor{\outputPredColor}{\emptyset}, \tuplecolor{\statePredColor}{\langle \rangle}\rangle}, \langle\tuplecolor{\inputPredColor}{\{b\}}, \tuplecolor{\outputPredColor}{\emptyset}, \tuplecolor{\statePredColor}{\langle b\rangle}\rangle\},$&{$\tuplecolor{\specialPredColor}{1}\rangle$} \\\hline
		$\langle\tuplecolor{black}{\{\langle\tuplecolor{\inputPredColor}{\emptyset}, \tuplecolor{\outputPredColor}{\emptyset}, \tuplecolor{\statePredColor}{\langle \rangle}\rangle}, \langle\tuplecolor{\inputPredColor}{\{b\}}, \tuplecolor{\outputPredColor}{\{b\}}, \tuplecolor{\statePredColor}{\langle b\rangle}\rangle\}, $&{$ \tuplecolor{\specialPredColor}{0}\rangle$} \\\hline
		$\langle\tuplecolor{black}{\{\langle\tuplecolor{\inputPredColor}{\{b\}}, \tuplecolor{\outputPredColor}{\emptyset}, \tuplecolor{\statePredColor}{\langle b \rangle}\rangle}, \langle\tuplecolor{\inputPredColor}{\{b\}}, \tuplecolor{\outputPredColor}{\{b\}}, \tuplecolor{\statePredColor}{\langle b\rangle}\rangle\}, $& {$\tuplecolor{\specialPredColor}{1}\rangle$} \\\hline
		$\langle\tuplecolor{black}{\{\langle\tuplecolor{\inputPredColor}{\emptyset}, \tuplecolor{\outputPredColor}{\emptyset}, \tuplecolor{\statePredColor}{\langle \rangle}\rangle, \langle\tuplecolor{\inputPredColor}{\{b\}}, \tuplecolor{\outputPredColor}{\emptyset}, \tuplecolor{\statePredColor}{\langle b \rangle}\rangle}, $&\multirow{2}{*}{$\tuplecolor{\specialPredColor}{0}\rangle$} \\
		\hspace{1.05em}$\tuplecolor{black}{\langle\tuplecolor{\inputPredColor}{\{b\}}, \tuplecolor{\outputPredColor}{\{b\}}, \tuplecolor{\statePredColor}{\langle b\rangle}\rangle\}}, $ %
	\end{tabular}%
};
\node (joinrbrt) [stdnodenum,left=-0.45em of joinrrt] {%
	\begin{tabular}{c}
		\multirow{1}{*}{$i$}\\
		\hline\hline
		$1$ \\\hline
		$2$ \\\hline
		$3$ \\\hline
		{$4$}\\\hline
		{$5$}\\\hline
		{$6$}\\\hline
		\multirow{2}{*}{$7$}\\\\
	\end{tabular}%
};
\node (joinr) [stdnodetable,right=6.5em of i13, yshift=-7em, label={[tdlabel, xshift=-1.1em,yshift=+0.1em]above right:$\pi_{{10}}$}] {%
	\begin{tabular}{l@{\hspace{0.0em}}l@{\hspace{0.0em}}r}%
		\multicolumn{1}{l}{$\langle \tuplecolor{black}{\nu_{10.i}}, $}&\multicolumn{1}{r}{$\tuplecolor{\specialPredColor}{c_{10.i}} \rangle$}\\
		\hline\hline
		$\langle\tuplecolor{black}{\{\langle\tuplecolor{\inputPredColor}{\{d\}}, \tuplecolor{\outputPredColor}{\{d\}}, \tuplecolor{\statePredColor}{\langle d \rangle}\rangle\}}, $&$\tuplecolor{\specialPredColor}{1}\rangle$ \\\hline
		$\langle\tuplecolor{black}{\{\langle\tuplecolor{\inputPredColor}{\{b,d\}}, \tuplecolor{\outputPredColor}{\{d\}}, \tuplecolor{\statePredColor}{\langle b,d \rangle}\rangle\}}, $&$\tuplecolor{\specialPredColor}{1}\rangle$ \\\hline
		$\langle\tuplecolor{black}{\{\langle\tuplecolor{\inputPredColor}{\{d\}}, \tuplecolor{\outputPredColor}{\{d\}}, \tuplecolor{\statePredColor}{\langle d\rangle}\rangle}, $& \multirow{2}{*}{$\tuplecolor{\specialPredColor}{1}\rangle$} \\ %
		$\hspace{1.05em}\tuplecolor{black}{\langle\tuplecolor{\inputPredColor}{\{b,d\}}, \tuplecolor{\outputPredColor}{\{d\}}, \tuplecolor{\statePredColor}{\langle b,d\rangle}\rangle\},}$\\\specialrule{.1em}{.05em}{.05em}	
		$\langle\tuplecolor{black}{\{\langle\tuplecolor{\inputPredColor}{\{b,e\}}, \tuplecolor{\outputPredColor}{\{e\}}, \tuplecolor{\statePredColor}{\langle b,e\rangle}\rangle\}}, $&$\tuplecolor{\specialPredColor}{1}\rangle$ \\\hline
		$\langle\tuplecolor{black}{\{\langle\tuplecolor{\inputPredColor}{\{b,e\}}, \tuplecolor{\outputPredColor}{\{b,e\}}, \tuplecolor{\statePredColor}{\langle e,b\rangle}\rangle\}}, $&$\tuplecolor{\specialPredColor}{1}\rangle$ \\\hline
		$\langle\tuplecolor{black}{\{\langle\tuplecolor{\inputPredColor}{\{b,e\}}, \tuplecolor{\outputPredColor}{\{e\}}, \tuplecolor{\statePredColor}{\langle b,e\rangle}\rangle}, $& \multirow{2}{*}{$\tuplecolor{\specialPredColor}{1}\rangle$} \\ %
		$\hspace{1.05em}\tuplecolor{black}{\langle\tuplecolor{\inputPredColor}{\{b,e\}}, \tuplecolor{\outputPredColor}{\{b,e\}}, \tuplecolor{\statePredColor}{\langle e,b\rangle}\rangle\},}$\\\specialrule{.1em}{.05em}{.05em}	
		$\langle\tuplecolor{black}{\{\langle\tuplecolor{\inputPredColor}{\{d,e\}}, \tuplecolor{\outputPredColor}{\{d,e\}}, \tuplecolor{\statePredColor}{\langle d,e\rangle}\rangle\}}, $&$\tuplecolor{\specialPredColor}{1}\rangle$ \\\hline
		$\langle\tuplecolor{black}{\{\langle\tuplecolor{\inputPredColor}{\{d,e\}}, \tuplecolor{\outputPredColor}{\{d,e\}}, \tuplecolor{\statePredColor}{\langle e,d\rangle}\rangle\}}, $&$\tuplecolor{\specialPredColor}{1}\rangle$ \\\hline
		$\langle\tuplecolor{black}{\{\langle\tuplecolor{\inputPredColor}{\{d,e\}}, \tuplecolor{\outputPredColor}{\{d,e\}}, \tuplecolor{\statePredColor}{\langle d,e\rangle}\rangle}, $& \multirow{2}{*}{$\tuplecolor{\specialPredColor}{1}\rangle$} \\ %
		$\hspace{1.05em}\tuplecolor{black}{\langle\tuplecolor{\inputPredColor}{\{d,e\}}, \tuplecolor{\outputPredColor}{\{d,e\}}, \tuplecolor{\statePredColor}{\langle e,d\rangle}\rangle\},}$ %
	\end{tabular}%
};
\node (joinr2) [stdnodetable,right=0.5em of joinr, yshift=0em, label={[tdlabel, xshift=-1.1em,yshift=+0.1em]above right:$\pi_{{11}}$}] {%
	\begin{tabular}{l@{\hspace{0.0em}}l@{\hspace{0.0em}}r}%
		\multicolumn{1}{l}{$\langle \tuplecolor{black}{\nu_{11.i}}, $}&\multicolumn{1}{r}{$\tuplecolor{\specialPredColor}{c_{11.i}} \rangle$}\\
		\hline\hline
		$\langle\tuplecolor{black}{\{\langle\tuplecolor{\inputPredColor}{\{d\}}, \tuplecolor{\outputPredColor}{\{d\}}, \tuplecolor{\statePredColor}{\langle d \rangle}\rangle\}}, $&$\tuplecolor{\specialPredColor}{2}\rangle$ \\\hline
		$\langle\tuplecolor{black}{\{\langle\tuplecolor{\inputPredColor}{\{b,d\}}, \tuplecolor{\outputPredColor}{\{d\}}, \tuplecolor{\statePredColor}{\langle b,d \rangle}\rangle\}}, $&$\tuplecolor{\specialPredColor}{1}\rangle$ \\\hline
		$\langle\tuplecolor{black}{\{\langle\tuplecolor{\inputPredColor}{\{d\}}, \tuplecolor{\outputPredColor}{\{d\}}, \tuplecolor{\statePredColor}{\langle d\rangle}\rangle}, $& \multirow{2}{*}{$\tuplecolor{\specialPredColor}{1}\rangle$} \\ %
		$\hspace{1.05em}\tuplecolor{black}{\langle\tuplecolor{\inputPredColor}{\{b,d\}}, \tuplecolor{\outputPredColor}{\{d\}}, \tuplecolor{\statePredColor}{\langle b,d\rangle}\rangle\},}$\\\specialrule{.1em}{.05em}{.05em}	
		$\langle\tuplecolor{black}{\{\langle\tuplecolor{\inputPredColor}{\{b\}}, \tuplecolor{\outputPredColor}{\emptyset}, \tuplecolor{\statePredColor}{\langle b\rangle}\rangle\}}, $&$\tuplecolor{\specialPredColor}{1}\rangle$ \\\hline
		$\langle\tuplecolor{black}{\{\langle\tuplecolor{\inputPredColor}{\{b\}}, \tuplecolor{\outputPredColor}{\{b\}}, \tuplecolor{\statePredColor}{\langle b\rangle}\rangle\}}, $&$\tuplecolor{\specialPredColor}{1}\rangle$ \\\hline
		$\langle\tuplecolor{black}{\{\langle\tuplecolor{\inputPredColor}{\{b\}}, \tuplecolor{\outputPredColor}{\emptyset}, \tuplecolor{\statePredColor}{\langle b\rangle}\rangle}, $& \multirow{2}{*}{$\tuplecolor{\specialPredColor}{1}\rangle$} \\ %
		$\hspace{1.05em}\tuplecolor{black}{\langle\tuplecolor{\inputPredColor}{\{b\}}, \tuplecolor{\outputPredColor}{\{b\}}, \tuplecolor{\statePredColor}{\langle b\rangle}\rangle\},}$\\\specialrule{.1em}{.05em}{.05em}	
		$\langle\tuplecolor{black}{\{\langle\tuplecolor{\inputPredColor}{\{d,e\}}, \tuplecolor{\outputPredColor}{\{d,e\}}, \tuplecolor{\statePredColor}{\langle d,e\rangle}\rangle\}}, $&$\tuplecolor{\specialPredColor}{1}\rangle$
	\end{tabular}%
};
\node (joinrb2) [stdnodenum,right=-0.45em of joinr2] {%
	\begin{tabular}{c}
		\multirow{1}{*}{$i$}\\
		\hline\hline
		$1$ \\\hline
		$2$ \\\hline
		\multirow{2}{*}{$3$} \\\\\specialrule{.1em}{.05em}{.05em}	
		$4$ \\\hline
		$5$ \\\hline
		\multirow{2}{*}{$6$} \\\\\specialrule{.1em}{.05em}{.05em}	
		$7$
	\end{tabular}%
};
\node (joinrb) [stdnodenum,left=-0.45em of joinr] {%
	\begin{tabular}{c}
		\multirow{1}{*}{$i$}\\
		\hline\hline
		$1$ \\\hline
		$2$ \\\hline
		\multirow{2}{*}{$3$} \\\\\specialrule{.1em}{.05em}{.05em}	
		$4$ \\\hline
		$5$ \\\hline
		\multirow{2}{*}{$6$} \\\\\specialrule{.1em}{.05em}{.05em}	
		$7$ \\\hline
		$8$ \\\hline
		\multirow{2}{*}{$9$} \\\\
	\end{tabular}%
};
\coordinate (top) at ($ (leaf2.north east)+(0.6em,-0.5em) $);
\coordinate (bot) at ($ (top)+(0,-12.9em) $);

\draw [<-] (j) to (rt);
\draw [->] (j) to ($ (i13.north)$);
\draw [->] (j) to ($ (r22.north)$);
\draw [->](r2) to (i22);
\draw [->](r22) to (r2);
\draw [<-](i2) to (i22);
\draw [<-](l2) to (i2);
\draw [<-](l1) to (i1);
\draw [->](i12) to (i1);
\draw [->](i13) to (i12);
\draw [<-](bc) to (l2);
\draw [<-](cc) to (bc);
\draw [<-](lrx) to (cc);

\draw [dashed, bend right=40] (leaf0r) to (rt);
\draw [dashed, bend left=15] (joinrrt) to (r22);
\draw [dashed] (j) to (join);
\draw [dashed, bend right=15] (i2) to (leaf2);
\draw [dashed, bend left=5] (i12) to (leaf1);
\draw [dashed, bend left=22] (leaf0) to (l1);
\draw [dashed, bend right=1] (leaf0x) to (i13);
\draw [dashed, bend right=45] (joinr) to (i22);
\end{tikzpicture}
\caption{Selected tables of~$\pi$ obtained by~$\dpa_{\algo{PROJ}}$ on
  TD~${\cal T}$ and purged table mapping~$\nu$ (obtained by purging on~$\tau$, c.f, %
  Figure~\ref{fig:running2}).} %
\label{fig:running3}
\end{figure*}

\begin{example} %
  Recall instance~$(\prog,P)$, TD~$\TTT$, and tables~$\tab{1}$,
  $\ldots$, $\tab{14}$ from Examples~\ref{ex:running0}, \ref{ex:sat},
  and Figure~\ref{fig:running2}. Figure~\ref{fig:running3} depicts
  selected tables of~$\pi_1, \ldots, \pi_{14}$ obtained after
  running~$\dpa_\PROJ$ for counting projected answer sets.
  We assume that row $i$ in table $\pi_t$ corresponds to
  $\vec{v_{t.i}} = \langle \rho_{t.i}, c_{t.i} \rangle$
  where~$\rho_{t.i}\subseteq\nu(t)$.
  Recall that for some nodes~$t$, there are rows among
  different~$\PRIM$-tables that are removed (highlighted gray in Figure~\ref{fig:running2}) during purging. %
  By purging we avoid to correct stored counters (backtracking)
  whenever a row has no ``succeeding'' row in the
  parent table.
  
  Next, we discuss selected rows obtained by
  $\dpa_\PROJ((\prog,P),(T,\chi,\nu))$. Tables $\pi_1$, $\ldots$,
  $\pi_{14}$ are shown in Figure~\ref{fig:running3}.
  Since~$\type(t_1)= \leaf$, we have
  $\pi_1=\langle\{\langle \emptyset , \emptyset, \langle \rangle
  \rangle \}, 1\rangle$.  Intuitively, at~$t_1$ the
  row~$\langle\emptyset, \emptyset, \langle\rangle\rangle$ belongs to~$1$ bucket.
  Node~$t_2$ introduces atom~$a$, which results in
  table~$\pi_2\eqdef\big\SB\langle \{\vec{u_{2.1}}\},
  1\rangle, \langle \{\vec{u_{2.2}}\},
  1\rangle, \langle \{\vec{u_{2.1}}, \vec{u_{2.2}}\},
  1\rangle\big\SE$, where~$\vec{u_{2.1}}=\langle \emptyset, \emptyset, \langle \rangle\rangle$ and~$\vec{u_{2.2}}=\langle \{a\}, \emptyset, \langle a\rangle \rangle$ 
  (derived similarly to table~$\pi_{4}$ as in Example~\ref{ex:pcnt}). 
  Node~$t_{10}$ introduces projected atom~$e$, and
  node~$t_{11}$ removes~$e$.  
  For row~$\vec{v_{11.1}}$ we compute the
  count~$\ipmc(t_{11},\{\vec{\tabval_{11.1}}\},
  \langle\pi_{10}\rangle)$ by means of~$\pcnt$. Therefore, take
  for~$\varphi$ the singleton set~$\{\vec{\tabval_{11.1}}\}$.
  We simply have
  $\ipmc(t_{11},\{\vec{\tabval_{11.1}}\}, \langle\pi_{10}\rangle) =
  \pmc(t_{11},\{\vec{\tabval_{11.1}}\}, \langle\pi_{10}\rangle)$.  To
  compute
  $\pmc(t_{11},\{\vec{\tabval_{11.1}}\}, \langle\pi_{10}\rangle)$, we
  take for~$O$ the sets~$\{\vec{u_{10.1}}\}$, $\{\vec{u_{10.6}}\}$,
  $\{\vec{u_{10.7}}\}$, and~$\{\vec{u_{10.6}}, \vec{u_{10.7}}\}$ into
  account, since all other non-empty subsets of origins
  of~$\vec{\tabval_{11.1}}$ in~$\nu_{10}$ do not occur in~$\pi_{10}$.
  Then, we take the sum over the values
  $\sipmc(\langle \pi_{10}\rangle,\{\vec{\tabval_{10.1}}\})=1$,
  $\sipmc(\langle \pi_{10}\rangle,\{\vec{\tabval_{10.6}}\})=1$,
  $\sipmc(\langle \pi_{10}\rangle,\{\vec{\tabval_{10.7}}\})=1$ and
  subtract
  $\sipmc(\langle \pi_{10}\rangle,\{\vec{\tabval_{10.6}},
  \vec{\tabval_{10.7}}\})=1$. This results
  in~$\pmc(t_{11},\{\vec{\tabval_{11.1}}\}, \langle\pi_{10}\rangle) =
  c_{10.1} + c_{10.7}\; + $ $ c_{10.8} - c_{10.9} = 2$. We proceed similarly
  for row~$v_{11.2}$, resulting in~$c_{11.2}=1$.
  Then for row~$v_{11.3}$,
  $\ipmc(t_{11},\{\vec{\tabval_{11.1}},\vec{\tabval_{11.6}}\}, \langle\pi_{10}\rangle) = | %
  \pmc(t_{11},\{\vec{\tabval_{11.1}},\vec{\tabval_{11.6}}\}, \langle\pi_{10}\rangle) - \ipmc(t_{11},\{\vec{\tabval_{11.1}}\}, \langle\pi_{10}\rangle)$ $- \ipmc(t_{11},\{\vec{\tabval_{11.6}}\}, \langle\pi_{10}\rangle) | = \Card{2-c_{11.1}-c_{11.2}}= \Card{2 -2 - 1} =\Card{-1} = 1 = c_{11.3}$.
  Hence, $c_{11.3} = 1$ represents the number of projected answer sets,
  both rows~$\vec{u_{11.1}}$ and~$\vec{u_{11.6}}$ have in common. We
  then use it for table~$t_{12}$.  Node~$t_{12}$ removes projection
  atom~$d$.  For node~$t_{13}$ where $\type(t_{13}) = \join$ one
  multiplies stored $\sipmc$ values for \AlgA-rows in the two children
  of~$t_{13}$ accordingly.  In the end, the projected answer sets count
  of~$\prog$ corresponds to~$\sipmc(\langle\pi_{14}\rangle,\vec{u_{14.1}})=3$.
\end{example}

\subsection{Runtime Analysis and Correctness}
Next, we present asymptotic upper bounds on the runtime of our
Algorithm~$\dpa_{\PROJ}$.  
We assume~$\gamma(n)$ to be the number of operations that are required
to multiply two~$n$-bit integers, which can be achieved in time
$n\cdot log\, n \cdot log\, log\,n$~\cite{Knuth1998,Harvey2016}.  
Often even constant-time multiplication is assumed.

\begin{theorem}%
  \label{thm:runtime}
  Given a \PASP instance~$(\prog,P)$ and a tabled tree decomposition
  $\TTT_{\text{purged}} = (T,\chi,\nu)$ of~$\mathcal{G}_\prog$ of width~$k$ with $g$
  nodes. Then, $\dpa_{\PROJ}$ runs in time
  $\mathcal{O}(2^{4m}\cdot g \cdot \gamma(\CCard{\prog}))$
  where~$m\eqdef \max(\{\nu(t) \mid t\in N\})$.
\end{theorem}
\begin{proof}
  Let~$d = k+1$ be maximum bag size of the TD~$\TTT$. For each
  node~$t$ of $T$, we consider the table $\nu(t)$ of $\TTT_{\text{purged}}$.
  Let TDD~$(T,\chi,\pi)$ be the output of~$\dpa_\PROJ$. In the worst-case,
  we store in~$\pi(t)$ each subset~$\rho \subseteq \nu(t)$ together
  with exactly one counter. Hence, we have at most $2^{m}$ many rows
  in $\rho$.
  In order to compute $\ipmc$ for~$\rho$, we consider every
  subset~$\varphi \subseteq \rho$ and compute~$\pcnt$. Since
  $\Card{\rho}\leq m$, we have at most~$2^{m}$ many subsets $\varphi$
  of $\rho$. Finally, for computing $\pcnt$, we consider in the worst
  case each subset of the origins of~$\varphi$ for each child table,
  which are at most~$2^{m}\cdot 2^{m}$ because of nodes~$t$
  with~$\type(t)=\join$.
  In total, we obtain a runtime bound
  of~$\bigO{2^{m} \cdot 2^{m} \cdot 2^{m}\cdot 2^{m} \cdot
    \gamma(\CCard{\prog})} \subseteq \bigO{2^{4m} \cdot
    \gamma(\CCard{\prog}})$ due to multiplication of two $n$-bit
  integers for nodes~$t$ with~$\type(t)=\join$ at costs~$\gamma(n)$.
  Then, we apply this to every node of~$T$ %
  resulting in
  runtime~$\bigO{2^{4m} \cdot g \cdot \gamma(\CCard{\prog})}$.
\end{proof}

\begin{corollary}\label{cor:runtime}
  Given an instance $(\prog,P)$ of \PASP where $\prog$ is
  head-cycle-free and has treewidth~$k$. Then, $\mdpa{\PRIM}$ runs in
  time~$\mathcal{O}(2^{3^{k+1.27}\cdot k!}\cdot \CCard{\prog}\cdot
  \gamma(\CCard{\prog}))$.
\end{corollary}
\begin{proof}
  We can compute in time~$2^{\mathcal{O}(k^3)}\cdot\CCard{\mathcal{G}_\prog}$ a
  TD~${\cal T'}$ with~$g\leq \CCard{\prog}$ nodes of width at
  most~$k$~\cite{Bodlaender96}. Then, we can simply
  run~$\dpa_{\PRIM}$, which runs in
  time~$\mathcal{O}({3^{k}\cdot k!}\cdot \CCard{\prog})$ by
  Theorem~\ref{thm:primruntime} and since the number of nodes of a
  tree decomposition is linear in the size of the input
  instance~\cite{Bodlaender96}. %
  Then, we again traverse the TD for purging and output
  $\TTT_{\text{purged}}$, which runs in time single exponential of the
  treewidth and linear of the instance size. Finally, we run
  $\dpa_{\PROJ}$ and obtain by Theorem~\ref{thm:runtime} that the
  runtime bound
  $\mathcal{O}(2^{4\cdot3^{k}\cdot k!}\cdot \CCard{\prog}\cdot
  \gamma(\CCard{\prog})) \subseteq $
  $\mathcal{O}(2^{3^{k + 1.27}\cdot k!}\cdot \CCard{\prog}\cdot
  \gamma(\CCard{\prog}))$.  %
  Hence, the corollary holds.
\end{proof}

The next result establishes lower bounds. %

\begin{theorem}
  Unless ETH fails, $\PASP$ cannot be solved in
  time~$2^{2^{o(k)}}\cdot \CCard{\prog}^{o(k)}$ for a given instance
  $(\prog,P)$ where~$k$ is the treewidth of the primal graph of~$\prog$.
\end{theorem}
\begin{proof}
  Assume for proof by contradiction that there is such an algorithm.
  We show that this contradicts a very recent
  result~\cite{LampisMitsou17,FichteEtAl18}, which states that one
  cannot decide the validity of a QBF
  $\forall{V_1}.\exists V_2.E$ in
  time~$2^{2^{o(k)}}\cdot \CCard{E}^{o(k)}$,
  where 
  $E$ is in CNF.
  Let $(\forall{V_1}.\exists V_2.E,k)$ be an instance
  of~$\forall\exists$-\SAT parameterized by the treewidth~$k$. Then,
  we reduce to an instance~$((\prog,P),2k)$ of the decision
  version~$\PASP$-exactly-$2^{\Card{V_1}}$ when parameterized by
  treewidth of~$\mathcal{G}_\prog$ such that $P=V_1$, the number of solutions is
  exactly~$2^{\Card{V_1}}$, and~$\prog$ is as follows.  For
  each~$v\in V_1 \cup V_2$, program~$\prog$ contains %
  rule~$v \lor nv \hsep$.
  Each clause~$x_1, \ldots, x_i, \neg x_{i+1}, \ldots, \neg x_j$
  results in one additional
  rule~$\hsep \neg x_1,\ldots, \neg x_i, x_{i+1}, \ldots, x_{j}$.
  It is easy to see that the reduction is correct
  and therefore instance~$((\prog,P), 2k)$ is
  a yes instance of %
  $\PASP$-exactly-$2^{\Card{V_1}}$ 
  if and only if~$(\forall{V_1}.\exists V_2.E,k)$
  is a yes instance of %
  problem~$\forall\exists$-\SAT. %
  In fact, the reduction is also an fpl-reduction, since the treewidth
  of $\prog$ at most doubles due to duplication of atoms.
  Note that we require an \emph{fpl}-reduction here, as results do not
  carry over from simple fpt-reductions.
  This concludes the proof and establishes the theorem.
\end{proof}

\longversion{
\begin{corollary}
  Unless ETH fails, $\PASP$ cannot be solved in
  time~$2^{2^{o(k)}}\cdot \CCard{\prog}^{o(k)}$ for a given instance
  $(\prog,P)$ where~$k$ is the treewidth of the incidence graph of~$\prog$.
\end{corollary}
\begin{proof}
  Let $w_i$ and $w_p$ be the treewidth of the incidence graph and
  primal graph of~$\prog$, respectively. Then,
  $w_i \leq w_p +1$~\cite{SamerSzeider10b}, which establishes the
  claim.
\end{proof}

\begin{corollary}
  Given an instance $(\prog,P)$ of \PASP where $\prog$ has treewidth~$k$. Then,
  Algorithm~$\mdpa{\AlgA}$ runs in
  time~$2^{2^{\Theta(k)}} \cdot \CCard{\prog}^c$ for some positive
  integer~$c$.
\end{corollary}}

Finally, we state that indeed~$\mdpa{\PRIM}$ gives the projected answer sets count of a given head-cycle-free program~$\prog$.

\begin{proposition}[$\star$]\label{prop:phcworks}
  Algorithm $\mdpa{\PRIM}$ is correct and outputs for any instance
  of \PASP its projected answer sets count.
\end{proposition}
\begin{proof}
Soundness follows by establishing an invariant for any row of~$\pi(t)$ guaranteeing that the values of~$\ipmc$ indeed capture ``all-overlapping'' counts of~$\progt{t}$. One can show that the invariant is a consequence of the properties of~\PRIM and the additional ``purging'' step, which neither destroys soundness nor completeness of~$\dpa_\PRIM$. Further, completeness guarantees that indeed all the required rows are computed.
\end{proof}

\subsection{Solving \PDASP for Disjunctive Programs}

In this section, we extend our algorithm to solve the projected answer
set counting problem (\PDASP) for disjunctive programs. Therefore, we
simply use a local algorithm \algo{PRIM} for disjunctive ASP that was
introduced in the
literature~\cite{FichteEtAl17a,JaklPichlerWoltran09}.
Recall algorithm~\mdpa{\AlgA} illustrated in
Figure~\ref{fig:multiarch}.
First, we %
construct a graph representation and heuristically compute a tree
decomposition of this graph. Then, we run $\dpa_\algo{PRIM}$ as first
traversal resulting in TTD~$(T,\chi,\tau)$. Next, we purge rows
of~$\tau$, which cannot be extended to an answer set resulting in
TTD~$(T,\chi,\nu)$. Finally, we compute the projected answer sets count
by~$\dpa_{\PROJ}$ and obtain TTD~$(T,\chi,\pi)$.

\begin{proposition}[$\star$]\label{prop:disjworks}
  $\mdpa{\algo{PRIM}}$ is correct, i.e., it outputs the projected answer sets count for any instance
  of \PDASP.
\end{proposition}

The following corollary states the runtime results.

\begin{corollary}\label{cor:disjruntime}
  Given an instance $(\prog,P)$ of \PDASP where $\prog$ is a
  disjunctive program of treewidth~$k$. Then, $\mdpa{\algo{PRIM}}$
  runs in
  time~$\mathcal{O}(2^{2^{2^{k+3}}}\cdot \CCard{\prog}\cdot
  \gamma(\CCard{\prog}))$.
\end{corollary}
\begin{proof}
  The first two steps follow the proof of Corollary~\ref{cor:runtime}.
  However, $\dpa_{\algo{PRIM}}$ runs in
  time~$\mathcal{O}(2^{2^{k+2}}\cdot
  \CCard{\prog})$~\cite{FichteEtAl17a}. Finally, we run $\dpa_{\PROJ}$
  and obtain by Theorem~\ref{thm:runtime} that
  $\mathcal{O}(2^{4\cdot2^{2^{k+2}}}\cdot \CCard{\prog}\cdot
  \gamma(\CCard{\prog})) \subseteq $
  $\mathcal{O}(2^{2^{2^{k+3}}}\cdot \CCard{\prog}\cdot
  \gamma(\CCard{\prog}))$.  
\end{proof}

Again, we are interested in whether we can improve the algorithm
significantly. While we obtain lower bounds from the ETH for~$\SAT$
(single-exponential) and
for~$\forall\exists$-\SAT/$\exists\forall$-\SAT (double-exponential),
to our knowledge it is unproven whether this extends to
$\forall\exists\forall$-\SAT and~$\exists\forall\exists$-\SAT
(triple-exponential).
Since it was anticipated by~\citex{MarxMitsou16} that it follows
just by assuming ETH, we state this as hypothesis. In particular, they
claimed that alternating quantifier alternations are the reason for
large dependence on treewidth. However, the proofs can be quite
involved, trading an additional alternation for exponential
compression.

\begin{hypothesis}\label{hyp:lampis3}
  The $\forall\exists\forall$-\SAT problem for a QBF~$Q$ in DNF of
  treewidth~$k$ cannot be decided in
  time~${2^{2^{2^{o(k)}}}}\cdot \CCard{Q}^{o(k)}$.
\end{hypothesis}

\longversion{\begin{proposition}
  Unless ETH fails, QBFs of the form $\exists V_1.\forall V_2.\cdots\forall V_\ell. E$
where~$k$ is the treewidth of the primal graph of DNF formula~$E$, cannot be solved in time~$2^{2^{\dots^{2^{o(k)}}}}\cdot \CCard{\prog}^{o(k)}$,
where the height of the tower is~$\ell$.
\end{proposition}

\begin{proof}[Idea]
	Follows by construction defined in the proof of~\cite{LampisMitsou17} for QBFs of the form~$\exists V_1.\forall V_2. E$. \todo{provide rigorous proof?}
\end{proof}

\begin{corollary}
Unless ETH fails, QBFs of the form $\forall V_1.\exists V_2.\cdots\forall V_\ell. E$
where~$k$ is the treewidth of the primal graph of CNF formula~$E$, cannot be solved in time~$2^{2^{\dots^{2^{o(k)}}}}\cdot \CCard{\prog}^{o(k)}$,
where the height of the tower is~$\ell$.
\end{corollary}}

\begin{theorem}\label{thm:lowerbound_disj}
  Unless Hypothesis~\ref{hyp:lampis3} fails, \PDASP for disjunctive programs~$\prog$ cannot be
  solved in time~$2^{2^{2^{o(k)}}} \cdot \CCard{\prog}^{o(k)}$ for
  given instance~$(\prog, P)$ of treewidth~$k$.
\end{theorem}
\begin{proof}
  Assume for proof by contradiction that there is such an algorithm.
  We show that this contradicts Hypothesis~\ref{hyp:lampis3},~i.e.,
  we cannot decide the validity of a QBF %
  $Q=\forall{V_1}.\exists V_2.\forall V_3.E$ in
  time~$2^{2^{2^{o(k)}}}\cdot \CCard{E}^{o(k)}$
  where %
  $E$ is in DNF. 
  Assume we have
  given such an instance when parameterized by the treewidth~$k$.
  In the following, we employ a well-known
  reduction~$R$~\cite{EiterGottlob95}, which
  transforms~$\exists V_2.\forall V_3. E$
  into~$\prog=R(\exists V_2.\forall V_3. E)$ and gives a yes
  instance~$\prog$ of consistency if and only
  if~$\exists V_2. \forall V_3. E$ is a yes instance of
  $\exists\forall$-\SAT.
  Then, we reduce instance~$(Q,k)$ via a reduction~$S$ to an
  instance~$((\prog',V_1),2k+2)$, where $\prog'=R(\exists V_2'.\forall V_3. E)$, $V_2'\eqdef V_1\cup V_2$, of the decision
  version~$\PDASP$-exactly-$2^{\Card{V_1}}$ of~$\PDASP$ when
  parameterized by treewidth such that the number of projected answer
  sets is
  exactly~$2^{\Card{V_1}}$.
  It is easy to see that reduction~$S$
  gives a yes instance~$(\prog',V_1)$
  of~$\PDASP$-exactly-$2^{\Card{V_1}}$ if and only
  if~$\forall V_1.\exists V_2. \forall V_3. E$ is a yes instance
  of~$\forall\exists\forall$-\SAT.
  However, it remains to show that the reduction~$S$ indeed increases
  the treewidth only linearly.
  Therefore, let $\TTT=(T,\chi)$ be TD of~$E$. We transform~$\TTT$
  into a TD~$\TTT'=(T,\chi')$ of~$\mathcal{G}_{\prog'}$ as follows.  For each
  bag~$\chi(t)$ of~$\TTT$, we add vertices for the atoms~$w$ and $w'$
  (two additional atoms introduced in reduction~$R$) and in addition
  we duplicate each vertex~$v$ in~$\chi(t)$ (due to corresponding duplicate
  atoms introduced in reduction~$R$). Observe
  that~$\width(\TTT') \leq 2\cdot \width(\TTT) + 2$. By construction
  of~$R$, $\TTT'$ is then a TD of~$\mathcal{G}_{\prog'}$.
  Hence, $S$ is also an fpl-reduction.
\end{proof}

\longversion{
\begin{corollary}
Unless ETH fails, \PDASP for disjunctive programs~$\prog$ cannot be solved in time~$2^{2^{2^{o(k)}}} \cdot \CCard{\prog}^{o(k)}$ for given instance
~$(\prog, P)$ where~$k$ is the treewidth of the incidence graph of~$\prog$.
\end{corollary}}

Then, the runtime of algorithm~$\mdpa{\algo{PRIM}}$ is asymptotically
worst-case optimal, depending on multiplication costs~$\gamma(n)$. %
\longversion{
In fact, we can conclude the following
corollary, which renders algorithm~$\mdpa{\algo{PRIM}}$ asymptotically
worst-case optimal, depending on the costs~$\gamma(n)$ for multiplying
two~$n$-bit numbers.

\begin{corollary}
Unless ETH fails, \PDASP for disjunctive programs~$\prog$ runs in time~$2^{2^{2^{\Theta(k)}}} \cdot \CCard{\prog} \cdot \gamma(\CCard{\prog})$ for given instance
~$(\prog, P)$ where~$k$ is the treewidth of the primal graph of~$\prog$.
\end{corollary}
\begin{proof}
Lower bounds by Theorem~\ref{thm:lowerbound_disj}. Upper bound by Algorithm~$\dpa_{\algo{PRIM}}$ in the first pass (c.f., Corollary~\ref{cor:disjruntime}),
followed by purging and the projection algorithm~$\dpa_\PROJ$ in the second pass.
\end{proof}
}

\section{Conclusions}\label{sec:conclusions}

In the light of very recent works~\cite{EibenEtAl19,ijcai,HecherMorakWoltran20},
which provide methods on dealing with high treewidth for \SAT, QBF and other formalisms,
it seems that the full potential of treewidth has not been unleashed for ASP yet.
Towards making these advancements accessible for ASP, we present novel reductions for the important fragments of normal and HCF programs to \SAT.
We introduced novel algorithms to count the projected answer sets
(\PASP) of head-cycle-free or disjunctive programs. Our algorithms
employ dynamic programming and exploit small treewidth of the
primal graph of the input program. The second algorithm, which solves
arbitrary disjunctive programs, is expected asymptotically optimal
assuming the exponential time hypothesis (ETH).
More precisely, runtime is triple exponential in the treewidth and
polynomial in the size of the input instance. When we restrict the
input to head-cycle-free programs, the runtime drops to double
exponential.

Our results extend previous work to
answer set programming and we believe that it can be applicable to
other hard combinatorial problems, such as
circumscription~\cite{DurandHermannKolaitis05}, quantified Boolean
formulas (QBF)~\cite{CharwatWoltran16a}, or default
logic~\cite{FichteHecherSchindler18a}.
}

\section*{Acknowledgements}
We would like to thank the reviewers of the conference paper~\cite{Hecher20} as well as all involved reviewers of this work 
for their detailed, constructive and therefore extremely valuable feedback. 
Special appreciation goes to Andreas Pfandler for early discussions
as well as to Jorge Fandinno for spotting a technical issue in an earlier version of the conference paper~\cite{Hecher20}.
Finally, we acknowledge Rafael Kiesel for his contribution in an early
version of asp2sat.
This work has been supported by the Austrian Science Fund (FWF),
 Grants P32830 and Y698, as well as the Vienna Science and Technology Fund, Grant WWTF ICT19-065. 

\bibliographystyle{abbrv}
\bibliography{references}

\futuresketch{
\clearpage
\section{Additional Resources}

\subsection{Additional Examples}
\begin{example}[c.f.,\citey{FichteEtAl17a}]\label{ex:bagprog} 
  Intuitively, the tree decomposition of Figure~\ref{fig:graph-td}
  enables us to evaluate program $\prog$ by analyzing sub-programs
  $\{r_2\}$ and $\{r_3,r_4, r_5\}$, and combining results agreeing on
  $e$ followed by analyzing~$\{r_1\}$.  Indeed, for the given tree
  decomposition of Figure~\ref{fig:graph-td}, $\progt{t_1}=\{r_2\}$,
  $\progt{t_2}=\{r_3,r_4, r_5\}$ and
  $\prog=\progt{t_3}=\{r_1\} \cup \progtneq{t_3}$. Note that
  here~$\prog=\progt{t_3} \neq \progtneq{t_3}$ and the tree
  decomposition is not nice.  \longversion{For the tree decomposition
    of Figure~\ref{fig:graph-td2}, we have
    $\progt{t_1} = \{r_1,r_2\}$, %
    as well as $\progt{t_3} = \{r_3\}$.} %
\end{example}%

\subsection{Worst-Case Analysis of $\dpa_{\PRIM}$: Omitted proofs}

\begin{restateproposition}[prop:kfact]
\begin{proposition}
Given any positive integer~$i \geq 1$ and functions~$f(k)\eqdef k!, g(k) \eqdef 2^{k^{(i+1)/i}}$. Then, $f \in O(g)$.
\end{proposition}
\end{restateproposition}
\begin{proof}
We proceed by simultaneous induction.\\
Base case ($k=i=1$): Obviously, $1^{2} \geq 1!$.\\
Induction hypothesis: $k! \in O(2^{k^{(i+1)/i}})$\\
Induction step ($k \rightarrow k+1$): \\We have to show that for $k\geq k_0$ for some fixed $k_0$, the following equation holds.
\begin{align*}
  2^{(k+1)^{(i+1)/i}} \geq (k+1)\cdot k!\\
  2^{(k+1)^{1/i}\cdot(k+1)} \geq (k+1)\cdot k!\\
  2^{(k+1)^{1/i}+k\cdot(k+1)^{1/i}} \geq (k+1)\cdot k!\\
  2^{(k+1)^{1/i}}\cdot 2^{k\cdot(k+1)^{1/i}} \geq (k+1)\cdot k!\\
  2^{(k+1)^{1/i}} \cdot k! \geq^{IH} (k+1)\cdot k!\\
  2^{(k+1)^{1/i}} \geq (k+1)\\
  2^{(k+1)^{1/i}} \geq 2^{\text{log}_2(k+1)}\geq (k+1)\\
  \text{ where } k\geq k_0 \text{ for some fixed } k_0 \text{ since } \text{log}_2\in O(\text{exp}(1/i))
\end{align*}
Induction step ($k \rightarrow k+1, i \rightarrow i+1$): Analogous, previous step works for any~$i$.\\
Induction step ($i \rightarrow i+1$): Analogous.
\end{proof}

\subsection{Characterizing Extensions}

In the following, we assume~$(\prog,P)$ to be an instance of~$\PASP$. 
Further, let~$\mathcal{T}=(T,\chi,\tau)$
be an~$\AlgA$-TTD of~$\mathcal{G}_\prog$ where~$T=(N,\cdot,n)$, node~$t\in N$, and~$\rho\subseteq\tau(t)$.

\begin{definition}\label{def:extensions}
  Let $\vec u$ be a row of $\rho$.

  An \emph{extension below~$t$} is a set of pairs where a pair consist
  of a node~$t'$ of the \emph{induced sub-tree~$T[t]$ rooted at~$t$} and a row~$\vec v$ of $\tau(t')$
  and the cardinality of the set equals the number of nodes in the
  sub-tree~$T[t]$. 
  
  We define the family of \emph{extensions below~$t$}
  recursively as follows.  If $t$ is of type~\leaf, then
  $\Ext_{\leq t}(\vec u) \eqdef \{\{\langle t,\vec u\rangle\}\}$;
  otherwise
  $\Ext_{\leq t}(\vec u) \eqdef \bigcup_{\vec v \in \origs(t,\vec u)}
  \big\SB\{\langle t,\vec u\rangle\}\cup X_1 \cup \ldots \cup X_\ell
  \SM X_i\in\Ext_{\leq t_i}({\vec v}_{(i)})\big\SE$ %
  for the~$\ell$ children~$t_1, \ldots, t_\ell$ of~$t$.
  We extend this notation for an $\AlgS$-table~$\rho$ by
  $\Ext_{\leq t}(\rho)\eqdef \bigcup_{\vec u\in\rho} \Ext_{\leq
    t}(\vec u)$.  Further, we
  let~$\Exts \eqdef \Ext_{\leq n}(\tau(n))$ be the
  \emph{family of all extensions}. 
  
  Further, we define \emph{the local table for node}~$t$ and family~$E$ of extensions (below some node) as
  $\local(t,E)\eqdef \bigcup_{\hat\rho \in E}\{ \langle \vec{\tabval}\rangle \mid
  \langle t, \vec{\tabval}\rangle \in \hat{\rho}\}$.

\end{definition}

If we would construct all extensions below the root~$n$, it allows us
to also obtain all models of program~$\prog$.  To this end, we state the following definition.

\begin{definition}\label{def:satext}
  We define %
  the \emph{satisfiable
    extensions below~$t$} for~$\rho$ by
  \[\PExt_{\leq t}(\rho)\eqdef \bigcup_{\vec u\in\rho} \SB X \SM X
    \in \Ext_{\leq t}(\vec u), X \subseteq Y, Y \in \Exts\SE.\]
\end{definition}

\begin{observation}
$\PExt_{\leq n}(\tau(n)) = \Exts$.
\end{observation}

\begin{definition}
We define the \emph{purged table mapping~$\nu$ of~$\tau$} by
$\nu(t)\eqdef \local(t,\PExt_{\leq t}[\tau(t)])$ for every~$t\in N$.
\end{definition}

Next, we define an auxiliary notation that gives us a way to
reconstruct interpretations from families of extensions.

\begin{definition}\label{def:iextensions}
  Let $E$ be a family of extensions
  below~$t$. %
  We define the \emph{set~$\mathcal{I}(E)$ of interpretations} of~$E$
  by
  $\mathcal{I}(E) \eqdef \big\SB \bigcup_{\langle \cdot, \vec u
    \rangle \in X} \mathcal{I}(\vec u) \mid X \in E \big\SE$
  and the set~$\mathcal{I}_P(E)$ of \emph{projected interpretations} by
  $\mathcal{I}_P(E) \eqdef \big\SB \bigcup_{\langle \cdot, \vec u \rangle \in X}
  \mathcal{I}(\vec u) \cap P \mid X \in E \big\SE$.

\end{definition}

\begin{example} %
  Consider again program~$\prog$ and TTD~$(T,\chi,\tau)$ of~$\mathcal{G}_\prog$,
  where~$t_{14}$ is the root of~$T$, from Example~\ref{ex:sat}.
  Let~$X=\{\langle t_{13}, \langle\{b\}, \{b\}, \langle b\rangle\rangle\rangle, \langle t_{12},
  \langle\{b\}, \emptyset, \langle b\rangle\rangle\rangle,
  \langle t_{11},
  \langle\{b\}, \emptyset, \langle b\rangle\rangle\rangle,$
  $\langle t_{10},
  \langle\{b,e\}, \{e\}, \langle b,e\rangle\rangle\rangle,
  \langle t_{9},
  \langle\{e\}, \{e\}, \langle e\rangle\rangle\rangle,
  \langle t_{4},
  \langle\{b\}, \{b\},$ 
  $\langle b\rangle\rangle\rangle,
  \langle t_{3},
  \langle\{b\}, \{b\}, \langle b\rangle\rangle\rangle,
  \langle t_{1},
  \langle\emptyset, \emptyset, \langle \rangle\rangle\rangle\}$
  be an extension
  below~$t_{14}$.  Observe that~$X\in\Exts$ and that
  Figure~\ref{fig:running2} highlights those rows of tables for
  nodes~$t_{13}, t_{12}, t_{11}, t_{10}, t_{9}, t_4, t_3$ and~$t_1$ that also occur in~$X$
  (in yellow). Further, $\mathcal{I}(\{X\})=\{b,e\}$ computes the
  corresponding answer set of~$X$, and $\mathcal{I}_P(\{X\}) = \{e\}$ derives
  the projected answer sets of~$X$.  $\mathcal{I}(\Exts)$ refers to the set
  of answer sets of~$\prog$, whereas~$\mathcal{I}_P(\Exts)$ is the set
  of projected answer sets of~$\prog$.
\end{example}

\subsection{Correctness of~$\dpa_{\PRIM}$: Omitted proofs}

In the following, we assume~$\prog$ to be a head-cycle-free program. Further, let~$\mathcal{T}=(T,\chi,\tau)$
be an~$\AlgA$-TTD of~$\mathcal{G}_\prog$ where~$T=(N,\cdot,n)$ and~$t\in N$ is a node.

We state definitions required for the correctness
proofs of our algorithm \PRIM. In the end, we only store rows that
are restricted to the bag content to maintain runtime bounds. 
Similar to related work~\cite{FichteEtAl17a}, we define the
content of our tables in two steps. First, we define the properties of
so-called \emph{$\PRIM$-solutions up to~$t$}. Second, we restrict
these solutions to~\emph{$\PRIM$-row solutions} at~$t$.

\begin{definition}\label{def:globalhcf}
Let~$\hat I\subseteq\att{t}$ be an interpretation,
$\hat{\mathcal{P}}\subseteq \hat I$ be a set of atoms and~$\hat\sigma$ be an ordering over atoms~$\hat I$.
Then, $\langle \hat I, \hat{\mathcal{P}}, \hat\sigma\rangle$ is referred to as~\emph{$\PRIM$-solution up to~$t$} if the following holds.
  \begin{enumerate}
    \item~$\hat I\models\progt{t}$,
    \item for each~$a\in\hat I\cap\attneq{t}$, we have~$a\in\hat{\mathcal{P}}$, and
    \item $a\in\hat{\mathcal{P}}$ if and only if~$a$ is proven using program~$\progt{t}$ and ordering~$\hat\sigma$.
  \end{enumerate}
\end{definition}

Next, we observe that the $\PRIM$-solutions up to~$n$ suffice to capture all the answer sets.

\begin{proposition}\label{prop:hcfglobal}
The set of~$\PRIM$-solutions up to~$n$ characterizes the set of answer sets of~$\prog$.
In particular: $\{\hat I \mid  \langle \hat I, \hat{\mathcal{P}}, \hat\sigma \rangle \text{ is a } \PRIM\text{-solution up to }n\} = \{I \mid I \text{ is an answer set of }\prog\}$.
\end{proposition}
\begin{proof}
Observe that Definition~\ref{def:globalhcf} for root node~$t=n$ indeed suffices for~$\hat I$ to be a model of~$\progt{n}=\prog$,
and, moreover, every atom in~$\hat I=\hat P$ is proven in~$\prog$ by ordering~$\hat\sigma$.
\end{proof}

\begin{definition}\label{def:localhcf}
Let~$\langle \hat I, \hat{\mathcal{P}}, \hat\sigma\rangle$ be a~$\PRIM$-solution up to~$t$. Then, $\langle \hat I \cap \chi(t), \hat{\mathcal{P}} \cap \chi(t), \sigma \rangle$, where~$\sigma$ is the partial ordering of~$\hat\sigma$ only containing~$\chi(t)$, is referred to as~\emph{$\PRIM$-row solution at node~$t$}.
\end{definition}

Given a~$\PRIM$-solution~$\vec{\hat\tabval}$ up to~$t$ and a~$\PRIM$-row solution~$\vec\tabval$ at~$t$.
We say~$\vec{\hat\tabval}$ is a \emph{corresponding} $\PRIM$-solution up to~$t$ of~$\PRIM$-row solution at~$t$ if~$\vec{\hat\tabval}$ can be used
to construct~$\vec\tabval$ according to Definition~\ref{def:localhcf}.

In fact,~\emph{$\PRIM$-row solutions} at~$t$ suffice to capture all the answer sets of~$\prog$.
Before we show that, we need the following definition.

\begin{definition}
Let $t\in N$ be a node of~$\TTT$
  with~$\children(t) = \langle t_1, \ldots, t_\ell \rangle$.
Further, let~$\vec{\hat\tabval}=\langle \hat I, \hat{\mathcal{P}}, \hat\sigma\rangle$ be a~$\PRIM$-solution up to~$t$ 
and~$\vec{\hat v}=\langle \hat{I'}, \hat{\mathcal{P}'}, \hat{\sigma'} \rangle$ be a~$\PRIM$-solution up to
$t_i$. Then,~$\vec\tabval$ is \emph{compatible with~$\vec v$} (and vice-versa) if
	\begin{enumerate}
		\item $\hat{I'} = \hat{I}\cap \att{t_i}$
		\item $\hat{\mathcal{P}'} = \hat{\mathcal{P}}\cap \att{t_i}$
		\item $\hat{\sigma'}$ is a sub-sequence of~$\hat\sigma$ such that~$\hat\sigma$ may additionally contain atoms in~$\att{t}\setminus\att{t_i}$
	\end{enumerate}
\end{definition}

\begin{lemma}[Soundness]\label{lem:paspcorrect}
  Let $t\in N$ be a node of~$\TTT$
  with~$\children(t,T) = \langle t_1, \ldots, t_\ell \rangle$.
  Further, let $\vec v_i$ be a~$\PRIM$-row solution at~$t_i$ for~$1\leq i\leq \ell$.
  Then, each row~$\vec\tabval = \langle I, \mathcal{P}, \sigma \rangle$ in~$\tau(t)$ 
  with~$\langle \vec v_1, \ldots, \vec v_\ell \rangle \in \origa{\PRIM}(t, \vec\tabval)$ is also a~$\PRIM$-row solution at
  node~$t$. Moreover, for any corresponding~$\PRIM$-solution~${\vec{\hat\tabval}}$ up to~$t$ (of~$\vec\tabval$)
  there are corresponding \emph{compatible}~$\PRIM$-solutions~$\vec{\hat{v_i}}$ up to~$t_i$ (for~$\vec v_i$). %
\end{lemma}
\begin{proof}[Proof (Sketch)]
We proceed by case distinctions.
Assume case(i):~$\type(t)=\leaf$. Then, $\langle \emptyset, \emptyset, \langle \rangle \rangle$ is a~\PRIM-row solution at~$t$. This concludes case(i).

Assume case(ii):~$\type(t)=\intr$ and~$\chi(t)\setminus\chi(t')=\{a\}$. Let~$\vec v_1=\langle I, \mathcal{P}, \sigma\rangle$ be any \PRIM-row solution at child node~$t_1$,
and~$\vec{\hat{v_1}}=\langle \hat I, \hat{\mathcal{P}}, \hat\sigma\rangle$ be any corresponding \PRIM-solution up to~$t_1$, which exists by Definition~\ref{def:localhcf}. In the following, we show that the way~\PRIM transforms
\PRIM-row solution~$\vec v_1$ at~$t_1$ to a \PRIM-row solution~$\vec\tabval=\langle I', \mathcal{P}', \sigma'\rangle$ at~$t$ is sound.
We identify several sub-cases.

Case (a): Atom~$a\not\in I'$ is set to false. Then, \PRIM constructs~$\vec\tabval$ where~$I'=I, \sigma'=\sigma$ and~$\mathcal{P}'=\mathcal{P}\cup \gatherproof(I',\sigma', \prog_t)$. Note that by construction~$I'\models \prog_t$.
Towards showing soundness, we define how to transform~$\vec{\hat{v_1}}$ into~$\vec{\hat\tabval}$ such that~$\vec{\hat\tabval}$ is indeed the corresponding~$\PRIM$-solution up to~$t$ of row~$\vec{\tabval}$ constructed by~\PRIM. To this end, we define~$\vec{\hat\tabval}$ as follows: $\vec{\hat\tabval} = \langle \hat I, \hat{\mathcal{P}} \cup  \gatherproof(I',\sigma', \prog_t), \hat\sigma\rangle$. Observe that~$\vec{\hat\tabval}$ is a~\PRIM-solution up to~$t$ according to Definition~\ref{def:globalhcf}.
Moreover, by construction and Definition~\ref{def:localhcf}, $\vec{\hat\tabval}$ is a corresponding~$\PRIM$-solution up to~$t$ of~$\hat\tabval$. 
It remains to show, that indeed for any
corresponding~$\PRIM$-solution~${\vec{\hat\tabval}}=\langle \hat {I'},
\hat{\mathcal{P}'}, \hat{\sigma'} \rangle$ up to~$t$
(of~$\vec\tabval$, there is a
corresponding~$\PRIM$-solution~$\vec{\hat{\zeta_1}}$ up to~$t_1$
(of~$\vec{{v_1}}$). %
To this end, we
define~$\vec{\hat{\zeta_1}}=\langle \hat{I'}, \hat{\mathcal{P}'}
\setminus (\mathcal{P}'\setminus\mathcal{P}), \hat{\sigma'}\rangle$
that is by construction according to Definition~\ref{def:globalhcf}
indeed a corresponding~$\PRIM$-solution up to~$t_1$
of~$\vec{\hat{v_1}}$.  This concludes case (a).

Case (b): Atom~$a\in I'$ is set to true. Conceptually, the case works analogously. %
This concludes cases (b) and (ii). 

The remaining cases for nodes~$t$ with~$\type(t)=\rem$ (slightly easier) and nodes~$t$ with~$\type(t)=\join$, 
where we need to consider \PRIM-row solutions at two different child nodes of~$t$, go through similarly.
\end{proof}

\begin{lemma}[Completeness]\label{lem:primcomplete}
  Let~$t\in N$ be node of~$\TTT$ where
  $\type(t) \neq \leaf$ and~$\children(t,T) = \langle t_1, \ldots, t_\ell \rangle$. Given a
  $\PRIM$-row solution~$\vec\tabval=\langle I, \mathcal{P}, \sigma \rangle$ at node~$t$,
  and any corresponding~$\PRIM$-solution~$\vec{\hat\tabval}$ up to~$t$ (of~$\vec\tabval$).
  Then, there exists $\vec s=\langle {v_1}, \ldots, {v_\ell}\rangle$ where ${v_i}$ is a
  $\PRIM$-row solution at~$t_i$ %
  such that~$\vec s\in\origa{\PRIM}(t,\vec\tabval)$,
  and corresponding~$\PRIM$-solution~$\vec{\hat{v_i}}$ up to~$t_i$ (of~$v_i$) that is
  compatible with~$\vec{\hat\tabval}$.
\end{lemma}
\begin{proof}[Proof (Idea)]
Since~$\vec\tabval$ is a~\PRIM-row solution at~$t$, there is by Definition~\ref{def:localhcf} a corresponding~\PRIM-solution~$\vec{\hat\tabval}=\langle \hat I, \hat{\mathcal{P}}, \hat\sigma\rangle$ up to~$t$. 

We proceed again by case distinction. Assume that~$\type(t)=\intr$. Then we define~$\vec{\hat{v_1}}\eqdef \langle \hat I \setminus \{a\}, \hat{\mathcal{P}'}, \hat{\sigma'}\rangle$,
where~$\hat{\sigma'}$ is a sub-sequence of~$\hat\sigma$ that does not contain~$a$ and~$\hat{\mathcal P}'=\gatherproof(\hat I \setminus \{a\}, t_1, \progt{t_1})$. 
Observe that all the conditions of Definition~\ref{def:globalhcf} are met and that~$\hat{\mathcal P}'\subseteq \hat{\mathcal{P}'}$. Then, we can easily define \PRIM-row solution~$\vec{v_1}$ at~$t_1$ according to Definition~\ref{def:localhcf} by using~$\vec{\hat{v_1}}$. By construction of~$\vec{\hat{v_1}}$ and by the definition of~$\gatherproof$, we conclude that~$\vec\tabval$ can be constructed with~$\PRIM$
using~$\vec{v_1}$. Moreover, \PRIM-solution~$\vec{\hat{v_1}}$ up to~$t_1$ is indeed compatible with~$\vec{\hat\tabval}$.

Assume that~$\type(t)=\rem$. The case is slightly easier as the one above, and the remainder works similar.

Similarly, one can show the result for the remaining node with~$\type(t)=\join$, but define \PRIM-row solutions for two preceding child nodes of~$t$.
\end{proof}

We are now in the position to proof our theorem.

\begin{restatetheorem}[thm:primcorrectness]%
\begin{theorem}%
  The algorithm~$\dpa_\PRIM$ is correct. \\
  More precisely, %
  the algorithm~$\dpa_\PRIM((\prog,\cdot),\TTT)$ returns
  $\PRIM$-TTD~$(T,\chi,\tau)$ such that we can decide consistency of~$\prog$ and even reconstruct the answer sets of~$\prog$:
  \begin{align*}
	&\mathcal{I}(\Ext_{\leq n}[\tau(n)])=
	\{\hat I \mid  \langle \hat I, \hat{\mathcal{P}}, \hat\sigma \rangle \text{ is a } \PRIM\text{-solution up to }n\}\\
	&=\{I \mid I \in \ta{\at(\prog)}, I \text{ is an answer set of }\prog\}.
  \end{align*}
\end{theorem}
\end{restatetheorem}
\begin{proof}[Proof (Idea).]
  By Lemma~\ref{lem:paspcorrect} we have soundness for every
  node~$t \in N$ and hence only valid rows as output of table
  algorithm~$\PRIM$ when traversing the tree decomposition in
  post-order up to the root~$n$.
  By Proposition~\ref{prop:hcfglobal} we then know that we can reconstruct answer sets
  given~\PRIM-solutions up to~$n$.
  In more detail, we proceed by means of induction. 
  For the induction base we only store~\PRIM-row solutions~$\vec\tabval\in\tau(t)$ at a certain node~$t$ starting at the leaves.
  For nodes~$t$ with~$\type(t)=\leaf$, obviously there is only the following (one)~\PRIM-row solution at~$t$: $\vec\tabval=\langle \emptyset, \emptyset, \langle \rangle\rangle$.
  
  Then, by Lemma~\ref{lem:paspcorrect} we establish the induction step, since algorithm~\PRIM only creates~\PRIM-row solutions at every node~$t$,
  assuming that it gets~\PRIM-row solutions at~$t_i$ for every child node~$t_i$ of~$t$.
  As a result, if there is no answer set of~$\prog$, the table~$\tau(n)$ is empty.
  On the other hand, if there is an answer set of~$\prog$, we obtain a~\PRIM-row solution~$\vec\tabval$ at root node~$n$, 
  for which by Definition~\ref{def:localhcf} a corresponding~\PRIM-solution~$\vec{\hat\tabval}$ up to~$n$ exists.
  Further, in the induction step we ensured that~\PRIM-solutions up to~$t$ for every~\PRIM-row solution at~$t$ for every node~$t\in N$ can be found that are compatible to~$\vec{\hat\tabval}$. In other words, by keeping track of corresponding origin~\PRIM-row solutions of~$\vec\tabval$ we can combine interpretation positions~$\mathcal{I}(\cdot)$ of rows by following origin rows top-down in order to reconstruct only valid answer set.

  Next, we establish completeness by induction starting from the
  root~$n$. Let therefore,
  $\hat\rho=\langle \hat I, \hat{\mathcal{P}}, \hat\sigma \rangle$ be
  the~\PRIM-solution up to node~$n$. If~$\hat\rho$ does not exist for
  node~$n$, there is by definition no answer set
  of~$\prog$. Otherwise, by Definition~\ref{def:localhcf}, we know
  that for the root~$n$ we can construct \PROJ-row solutions at~$n$ of
  the form~$\rho=\langle\emptyset, \emptyset, \langle \rangle\rangle$
  for~$\hat\rho$.  We already established the induction step in
  Lemma~\ref{lem:primcomplete} using~$\rho$ and~$\hat\rho$. As a
  consequence, we can reconstruct exactly \emph{all the answer sets}
  of~$\prog$ by following origin rows (see
  Definition of~$\orig$) back to the leaves and combining
  interpretation parts~$\mathcal{I}(\cdot)$, accordingly.
  Hence, we obtain some (corresponding) rows for every
  node~$t$. Finally, we stop at the leaves.

  In consequence, we have shown both soundness and completeness. As a
  result, Theorem~\ref{thm:primcorrectness} is sustains.
\end{proof}

\begin{corollary}\label{cor:primcorrectness}
  Algorithm~$\dpa_\PRIM((\prog,\cdot),\TTT)$ returns
  $\PRIM$-TTD~$(T,\chi,\tau)$ such that:
  \begin{align*}
    &\mathcal{I}(\PExt_{\leq t}[\tau(t)])\\
    &=\{\hat I \mid  \langle \hat I, \hat{\mathcal{P}}, \hat\sigma \rangle \text{ is a } \PRIM\text{-solution up to }t, \text{ there is answer set }\\
    &\quad\;\;\, I' \supseteq \hat I \text{ of } \prog \text{ such that }
      I' \subseteq I \cup (\at(\prog) \setminus \att{t})\}\\
    &=\{I \mid I \in \ta{\att{t}}, I \models \prog_{\leq t}, \text{ there is an answer set }\\
    &\quad\;\;\, I' \supseteq I \text{ of } \prog \text{ such that } I'\subseteq I \cup (\at(\prog) \setminus \att{t})\}.
  \end{align*}
\end{corollary}
\begin{proof}
The corollary follows from the proof of Theorem~\ref{thm:primcorrectness} applied up to node~$t$ and by considering only rows that are involved in reconstructing answer sets (see Definition~\ref{def:satext}).
\end{proof}

\subsection{Correctness of~$\mdpa{\AlgA}$: Omitted proofs}

In the following, we assume~$(\prog, P)$ to be an instance of~$\PASP$. Further, let~$\mathcal{T}=(T,\chi,\tau)$
be an~$\AlgA$-TTD of~$\mathcal{G}_\prog$ where~$T=(N,\cdot,n)$, node~$t\in N$, and~$\rho\subseteq\tau(t)$.

\begin{definition}\label{def:asplocalsol}
Table algorithm $\AlgA$ is referred to as \emph{admissible}, if for each row $\vec{u_{t.i}}\in\tau(t)$ of any node~$t\in T$ the following holds:
  \begin{enumerate}
    \item $\mathcal{I}(\vec{\tabval_{t.i}}) \subseteq \chi(t)$
    \item For any $\vec v \in \tau(t')$, $\vec w \in \tau(t'')$ we have $\mathcal{I}(\vec v) \cap \chi(t') \cap \chi(t'') = \mathcal{I}(\vec w) \cap \chi(t') \cap \chi(t'')$
    \item $\mathcal{I}(\PExt_{\leq t}[\tau(t)]) = \{I \mid I \in
    \ta{\att{t}}, I \models \prog_{\leq t}, \text{ there is an answer set } I \cup (\at(\prog) \setminus \att{t}) \supseteq  I' \supseteq I \text{ of } \prog\}$
    \item If~$t=n$ or~$\type(t)=\leaf$: $\Card{\local(t,\PExt_{\leq t}[\tau(t)])} \leq 1$
  \end{enumerate}
\end{definition}

Note that the last condition is not a hard restriction, since the bags of the leaf and root nodes of a tree decomposition are defined to be empty anyway. However, it rather serves as technical trick simplifying proofs.

\begin{observation}
Table algorithms~$\PRIM$ and~$\algo{PRIM}$ are admissible.
\end{observation}
\begin{proof}
  Obviously, Conditions 1, 2, and 4 hold by construction of the table algorithms and by properties auf tree decompositions. For condition 3, we have to check for correctness and completeness, which has been shown~\cite{FichteEtAl17a} for algorithm~$\algo{PRIM}$. For~$\PRIM$, see Theorem~\ref{thm:primcorrectness} and Corollary~\ref{cor:primcorrectness}.
\end{proof}

In the following, we assume that whenever~$\AlgA$ occurs, $\AlgA$ is an admissible table algorithm.

\begin{proposition}\label{prop:sat}
$\mathcal{I}(\PExt_{\leq n}[\tau(n)]) = \mathcal{I}(\Exts) = \{I \mid I \in
    \ta{\at(\prog)}, I \text{ is an answer set of } \prog\}.$
\end{proposition}
\begin{proof}
  Fill in Definition~\ref{def:asplocalsol} with root~$n$ of $\AlgA$-TTD ${\cal T}$.
\end{proof}

The following definition is key for the correctness proof, since later we show that these are equivalent with the result of~$\dpa_\PROJ$ using purged table mapping~$\nu$.

\begin{definition}\label{def:pmc}
  The \emph{projected answer sets count} $\pmc_{\leq t}(\rho)$ of
  $\rho$ below~$t$ is the size of the union over projected
  interpretations of the satisfiable extensions of~$\rho$ below~$t$,
  formally,
  $\pmc_{\leq t}(\rho) \eqdef \Card{\bigcup_{\vec u\in\rho}
    \mathcal{I}_P(\PExt_{\leq t}(\{\vec u\}))}$.

  The \emph{intersection projected answer sets count}
  $\ipmc_{\leq t}(\rho)$ of $\rho$ below~$t$ is the size of the
  intersection over projected interpretations of the satisfiable
  extensions of~$\rho$ below~$t$,~i.e.,
  $\ipmc_{\leq t}(\rho) \eqdef \Card{\bigcap_{\vec u\in\rho}
    \mathcal{I}_P(\PExt_{\leq t}(\{\vec u\}))}$.
\end{definition}

In the following, we state definitions required for the correctness
proofs of our algorithm \PROJ.  In the end, we only store rows that
are restricted to the bag content to maintain runtime bounds. 
We define the
content of our tables in two steps. First, we define the properties of
so-called \emph{$\PROJ$-solutions up to~$t$}. Second, we restrict
these solutions to~\emph{$\PROJ$-row solutions} at~$t$.

\begin{definition}\label{def:globalsol}
  Let~$\emptyset \subsetneq \rho \subseteq \tau(t)$ be a
  table with $\rho \in \subbuckets_P(\tau(t))$.
  We define a \emph{${\PROJ}$-solution up to~$t$} to be the sequence
  $\langle \hat {\rho}\rangle = \langle\PExt_{\leq t}(\rho)\rangle$.
\end{definition}

Before we present equivalence results between~$\ipmc_{\leq t}(\ldots)$
and the recursive version~$\ipmc(t, \ldots)$
used during the computation of
$\dpa_\PROJ$, recall that~$\ipmc_{\leq t}$ and~$\pmc_{\leq t}$
(Definition~\ref{def:pmc}) are key to compute the projected answer sets
count. The following corollary states that computing $\ipmc_{\leq n}$
at the root~$n$ actually suffices to compute~$\pmc_{\leq n}$, which is
in fact the projected answer sets count of the input program.

\begin{corollary}\label{cor:psat}
  \begin{align*}
    &\ipmc_{\leq n}(\local(n,\PExt_{\leq n}[\tau(n)]))\\
 =& \pmc_{\leq n}(\local(n,\PExt_{\leq n}[\tau(n)]))\\
    =& \Card{\mathcal{I}_P(\PExt_{\leq n}[\tau(n)])}\\
    =& \Card{\mathcal{I}_P(\Exts)}\\
    =& \,|\{J \cap P
       \mid J \in \ta{\at(\prog)},\\
    & J \text{ is an answer set of } \prog\}|
  \end{align*}
\end{corollary}
\begin{proof}
  The corollary immediately follows from Proposition~\ref{prop:sat}
  and since the cardinality of $\local(n,\PExt_{\leq n}[\tau(n)])$ is at
  most one at root~$n$, by Definition~\ref{def:asplocalsol}.
\end{proof}

The following lemma establishes that the \PROJ-solutions up to
root~$n$ of a given tree decomposition solve the \PASP problem.

\begin{lemma}\label{lem:global}
  The
  value~$c = \sum_{\langle\hat{\rho}\rangle\text{ is a \PROJ-solution
      up to } n}\Card{\mathcal{I}_P(\hat{\rho})}$ corresponds to the
  projected answer sets count of~$\prog$ with respect to the set~$P$ of
  projection atoms.
\end{lemma}
\begin{proof}
  (``$\Longrightarrow$''): Assume
  that~$c = \sum_{\langle\hat{\rho}\rangle\text{ is a \PROJ-solution
      up to } n}\Card{\mathcal{I}_P(\hat {\rho})}$. Observe that there can be at
  most one projected solution up to~$n$ by Definition~\ref{def:asplocalsol}. %
  If~$c=0$, then $\tau(n)$ contains no rows. Hence, $\prog$ has no
  answer sets,~c.f., Proposition~\ref{prop:sat}, and obviously also no
  answer sets projected to~$P$. Consequently, $c$ is the projected answer sets
  count of~$\prog$.  
  If~$c>0$ we have by Corollary~\ref{cor:psat} that~$c$ is
  equivalent to the projected answer sets count of~$\prog$ with respect to~$P$.

  (``$\Longleftarrow$''): The proof proceeds similar to the only-if
  direction.
\end{proof}

\medskip %

In the following, we provide for a given node~$t$ and a given \PROJ-solution up to~$t$,
the definition of a \PROJ-row solution at~$t$.

\begin{definition}\label{def:loctab}~%
  Let %
  $\langle \hat{\rho} \rangle$ be a~$\PROJ$-solution up to~$t$. Then, we
  define the \emph{$\PROJ$-row solution at $t$} by
  $\langle \local(t,\hat{\rho}), \Card{\mathcal{I}_P(\hat{\rho})}\rangle$.
\end{definition}

\begin{observation}\label{obs:unique}
  Let $\langle \hat {\rho}\rangle$ be a \PROJ-solution up to a
  node~$t\in N$.  There is exactly one corresponding \PROJ-row
  solution
  $\langle \local(t,\hat{\rho}), \Card{\mathcal{I}_P(\hat{\rho})}\rangle$ at~$t$.

  Vice versa, let $\langle \rho, c\rangle$ at~$t$ be a \PROJ-row
  solution at~$t$ for some integer~$c$. Then, there is exactly one
  corresponding \PROJ-solution~$\langle\PExt_{\leq t}(\rho)\rangle$
  up to~$t$.
\end{observation}

We need to ensure that storing~$\PROJ$-row solutions at a
node~$t \in N$ suffices to solve the~\PASP problem, which is necessary
to obtain the runtime bounds as presented in
Corollary~\ref{cor:runtime}. For the root node~$n$, this is sufficient, shown in the following.

\begin{lemma}\label{lem:local}
  There is a
  \PROJ-row solution at the root~$n$ if and only if the projected
  answer sets count of~$\prog$ is larger than zero. Further, if there is a \PROJ-row solution~$\langle \rho, c\rangle$ at root~$n$, then~$c$ is the projected answer sets count of~$\prog$.
\end{lemma}
\begin{proof}%

  (``$\Longrightarrow$''): Let $\langle \rho, c\rangle$ be a
  \PROJ-row solution at root~$n$ where $\rho$ is an $\AlgA$-table and
  $c$ is a positive integer. Then, by Definition~\ref{def:loctab}
  there also exists a
  corresponding~$\PROJ$-solution~$\langle \hat{\rho} \rangle$ up
  to~$n$ such that $\rho = \local(n,\hat{\rho})$ and
  $c=\Card{\mathcal{I}_P(\hat{\rho})}$.
  Moreover, by Definition~\ref{def:asplocalsol}, we
  have~$\Card{\local(n,\PExt_{\leq n}[\tau(n)])}=1$.  
  Then, by Definition~\ref{def:globalsol},
  $\hat{\rho} = \PExt_{\leq n}[\tau(n)]$. By Corollary~\ref{cor:psat}, we
  have $c=\Card{\mathcal{I}_P(\PExt_{\leq n}[\tau(n)])}$ equals the projected answer sets count of~$\prog$.
  Finally, the claim follows.

  (``$\Longleftarrow$''): The proof proceeds similar to the only-if
  direction.
\end{proof}

Before we show that \PROJ-row solutions suffice, we require the following lemma.

\begin{observation}\label{obs:main_incl_excl}
  Let $n$ be a positive integer, $X = \{1, \ldots, n\}$, and $X_1$,
  $X_2$, $\ldots$, $X_n$ subsets of $X$.
  The number of elements in the intersection over all sets~$A_i$ is
    \[\Card{\bigcap_{i \in X} X_i} 
    = %
       \Bigg|\,\Card{\bigcup^n_{j = 1} X_j} %
                                         + \sum_{\emptyset \subsetneq I \subsetneq X} (-1)^{\Card{I}} 
                                              \Card{\bigcap_{i \in I} X_i}\,\Bigg|.\]
\end{observation}
\begin{proof}
  We take the well-known inclusion-exclusion
  principle~\cite{GrahamGrotschelLovasz95a} and rearrange the
  equation.
\end{proof}

\begin{lemma}\label{lem:main_incl_excl}
  Let $t\in N$ be a node of~$\TTT$
  with~$\children(t,T) = \langle t_1, \ldots, t_\ell \rangle$ and let
  $\langle\rho,\cdot\rangle$ be a~\PROJ-row solution at~$t$. Further, let~$\pi$ be a partial mapping of~$\pi'$ (finally returned by~$\dpa_\PROJ((\prog,P),\TTT)=(T,\chi,\pi')$), which maps nodes of the sub-tree~$T[t]$ rooted at~$t$ (excluding~$t$) to~$\PROJ$-tables.
  Then,
  \begin{enumerate}
  \item %
    $\ipmc(t,\rho,\langle\pi(t_1), \ldots,
    \pi(t_\ell)\rangle) = \ipmc_{\leq t}(\rho)$
  \item \smallskip%
    for $\type(t) \neq \leaf$:\\
    $\pmc(t,\rho,\langle\pi(t_1), \ldots,
    \pi(t_\ell)\rangle) = \pmc_{\leq t}(\rho)$.
  \end{enumerate}
\end{lemma}
\begin{proof}[Sketch]
  We prove the statement by simultaneous induction.
  
  (``Induction Hypothesis''): Lemma~\ref{lem:main_incl_excl} holds for the nodes in~$\children(t,T)$ and also for node~$t$, but on strict subsets~$\varphi\subsetneq\rho$.
  (``Base Cases''): Let $\type(t) = \leaf$.
  Then by definition,
  $\ipmc(t,\{\langle \emptyset, \ldots\rangle\}, \langle \rangle) = \ipmc_{\leq t}(\{\langle\emptyset,\ldots\rangle\}) =
  1$.  
  Recall that for $\pmc$ the equivalence does not hold for leaves, but we use a node
  that has a node~$t'\in N$ with~$\type(t') = \leaf$ as child for the
  base case. Observe that by definition of a tree decomposition
  such a node~$t$ can have exactly one child.
  Then, we have that
  $\pmc(t,\rho,\langle\pi(t')\rangle) = \sum_{\emptyset
    \subsetneq O \subseteq {\origs(t,\rho)}} (-1)^{(\Card{O} - 1)}
  \cdot \sipmc(\langle \tau(t')\rangle, O) =
  \Card{\bigcup_{\vec u\in\rho} \mathcal{I}_P(\PExt_{\leq t}(\{\vec u\}))} =
  \pmc_{\leq t}(\rho) = 1$ where $\langle\rho,\cdot\rangle$ is
  a~\PROJ-row solution at~$t$.

  (``Induction Step''): We proceed by case distinction.

  Assume that $\type(t) = \intr$.
  Let $a \in (\chi(t) \setminus \chi(t'))$ be an introduced
  atom. We have two cases. Case (i) $a$ also belongs to
  $(\at(\prog) \setminus P)$,~i.e., $a$ is not a projection atom; and
  Case (ii) $a$ also belongs to $P$,~i.e., $a$ is a projection
  atom.
  Assume that we have Case~(i).
  Let~$\langle \rho, c \rangle$ be a \PROJ-row solution at~$t$ for
  some integer~$c$. As a consequence of admissible algorithm~$\AlgA$ (see Definition~\ref{def:asplocalsol})
  there can be many rows in the table~$\tau(t)$ for one row in
  the table~$\tau(t')$, more precisely,
  $\Card{\buckets_P(\rho)} = 1$.
  As a result,
  $\pmc_{\leq t}(\rho) = \pmc_{\leq t'}(\orig(t,\rho))$ by
  applying Observation~\ref{obs:unique}.
  We apply the inclusion-exclusion principle on every subset~$\varphi$ of
  the origins of~$\rho$ in the definition of~$\pmc$ and by induction
  hypothesis we know that
  $\ipmc(t',\varphi,\langle\pi(t')\rangle) = \ipmc_{\leq
    t'}(\varphi)$, therefore,
  $\sipmc(\pi(t'), \varphi) = \ipmc_{\leq t'}(\varphi)$.  This
  concludes Case~(i) for $\pmc$. The induction step for $\ipmc$ works
  similar %
  by applying
  Observation~\ref{obs:main_incl_excl} and comparing the corresponding
  \PROJ-solutions up to~$t$ or $t'$, respectively. 
  Further, for showing the lemma for~$\ipmc$, one has to additionally apply the hypothesis for node~$t$, but on strict subsets~$\emptyset\subsetneq\varphi\subsetneq\rho$ of~$\rho$.
  Assume that we have Case~(ii). We proceed similar as in Case~(i),
  since Case~(ii) is just a special case here, more precisely, we also
  have $\Card{\buckets_P(\rho)} = 1$ here.

  Assume that $\type(t) = \rem$. Let
  $a \in (\chi(t') \setminus \chi(t))$ be a removed atom. We have
  two cases. Case (i) $a$ also belongs to
  $(\at(\prog) \setminus P)$,~i.e., $a$ is not a projection atom; and
  Case (ii) $a$ also belongs to $P$,~i.e., $a$ is a projection
  atom.
  Assume that we have Case~(i).  Let~$\langle \rho, c \rangle$ be a
  \PROJ-row solution at~$t$ for some integer~$c$.
  As a consequence of admissible table algorithms~$\AlgA$ (see Definition~\ref{def:asplocalsol}) there can be many rows
  in the table~$\tau(t)$ for one row in the
  table~$\tau(t')$ (and vice-versa). Nonetheless we still have
  $\pmc_{\leq t}(\rho) = \pmc_{\leq t'}(\orig(t,\rho))$, because
  $a \notin P$ by applying Observation~\ref{obs:unique}.
  We apply the inclusion-exclusion principle on every subset~$\varphi$ of
  the origins of~$\rho$ in the definition of~$\pmc$ and by induction
  hypothesis we know that
  $\ipmc(t',\varphi,\langle\pi(t')\rangle) = \ipmc_{\leq
    t'}(\varphi)$, therefore,
  $\sipmc(\pi(t'), \varphi) = \ipmc_{\leq t'}(\varphi)$.  This
  concludes Case~(i) for $\pmc$. Again, the induction step for $\ipmc$
  works similar, but swapped.
  Assume that we have Case~(ii).
  Let~$\langle \rho, c \rangle$ be a \PROJ-row solution at~$t$ for
  some integer~$c$.
  Here we cannot ensure
  $\pmc_{\leq t}(\rho) = \pmc_{\leq t'}(\orig(t,\rho))$, since
  buckets fall together.  However, by applying
  Observation~\ref{obs:unique} we have
  $\pmc_{\leq t}(\rho) = \sum_{\varphi \in
    \buckets_P(\origs(t,\rho)_{(1)})} \pmc(t', \varphi, C) $ where the
  sequence~$C$ consists of the tables~$\pi(t'_i)$ of the children~$t'_i$ of~$t'$.
  For every~$\varphi \in \subbuckets_P(\origs(t,\rho)_{(1)})$ by
  induction hypothesis we know that
  $\ipmc(t',\varphi,\langle\pi(t')\rangle) = \ipmc_{\leq
    t'}(\varphi)$.
  Hence, we apply the inclusion-exclusion principle over all
  subsets~$\zeta$ of~$\varphi$ for all~$\varphi$ independently.  By
  construction
  $\sipmc(\pi(t'), \zeta) = \ipmc_{\leq t'}(\zeta)$.  Then,
  by construction
  $\pcnt(t,\rho, C') = \sum_{\emptyset \subsetneq O \subseteq
    {\origs(t,\rho)}} (-1)^{(\Card{O} - 1)} \cdot \sipmc(C', O) =
  \pmc_{\leq t}(\rho)$ where
  $C' = \langle \pi(t') \rangle$, since for the remaining
  terms $\sipmc(C', O)$ is simply zero, including cases where
  different buckets are involved.
  This concludes Case~(ii) for $\pmc$. Again, the induction step for
  $\ipmc$ works similar, but swapped by again applying
  Observation~\ref{obs:main_incl_excl}.

  Assume that $\type(t) = \join$. We proceed similar to the introduce
  case. However, we have two \PROJ-tables for the children of~$t$.
  Hence, we have to consider both sides when computing $\sipmc$
  (see Definition of~$\sipmc$). %
  There we consider the
  cross-product of two \AlgA-tables and we can also correctly apply
  the inclusion-exclusion principle on subsets of this cross-product,
  which we can do by simply multiplying $\sipmc$-values
  accordingly. The multiplication is closely related to the join case
  in table algorithm~\AlgA. For $\ipmc$ this does not apply, since the
  inclusion-exclusion principle is carried out at the node~$t$ and not
  for its children.

  Since we outlined all cases that can occur for node~$t$, this
  concludes the proof sketch.
\end{proof}

\begin{lemma}[Soundness]\label{lem:correct}
  Let $t\in N$ be a node of~$\TTT$
  with~$\children(t,T) = \langle t_1, \ldots, t_\ell \rangle$.
  Then, each row~$\langle \rho, c \rangle$ at node~$t$ constructed
  by table algorithm~$\PROJ$ is also a~\PROJ-row solution for
  node~$t$.
\end{lemma}
\begin{proof}[Idea]
  Observe that Listing~\ref{fig:dpontd3} computes a row for each
  sub-bucket $\rho \in \subbuckets_P(\local(t,\PExt_{\leq t}[\tau(t)]))$. The
  resulting row~$\langle\rho, c \rangle$ obtained by~$\ipmc$ is
  indeed a \PROJ-row solution for~$t$ according to
  Lemma~\ref{lem:main_incl_excl}.
\end{proof}

\begin{lemma}[Completeness]\label{lem:complete}
  Let~$t\in N$ be node of~$\TTT$ where
  $\type(t) \neq \leaf$ and~$\children(t,T) = \langle t_1, \ldots, t_\ell \rangle$. Given a
  \PROJ-row solution~$\langle \rho, c \rangle$ at node~$t$.
  There exists $\langle C_1, \ldots, C_\ell\rangle$ where $C_i$ is set
  of \PROJ-row solutions at~$t_i$ %
  such that
  $\rho \in \PROJ(t, \cdot, \tau(t), \cdot, P, \langle C_1, \ldots,
  C_\ell\rangle)$.
\end{lemma}
\begin{proof}[Idea]
Since~$\langle\rho,c \rangle$ is a~\PROJ-row solution for~$t$, there is by Definition~\ref{def:loctab} a corresponding ~\PROJ-solution~$\langle\hat\rho\rangle$ up to~$t$ such that~$\local(t,\hat\rho) = \rho$. 

We proceed again by case distinction. Assume that~$\type(t)=\intr$. Then we define~$\hat{\rho'}\eqdef \{(t',\hat\varphi) \mid (t', \hat\varphi)\in \rho, t \neq t'\}$. Then, for each subset~$\emptyset\subsetneq\varphi\subseteq\local(t',\hat{\rho'})$, we define~$\langle \varphi, \Card{\mathcal{I}_P(\PExt_{\leq t}(\varphi))}\rangle$ in accordance with Definition~\ref{def:loctab}. By Observation~\ref{obs:unique}, we have that~$\langle \varphi, \Card{\mathcal{I}_P(\PExt_{\leq t}(\varphi))}\rangle$ is an \AlgA-row solution at node~$t'$. 
Since we defined the~\PROJ-row solutions for~$t'$ for all the respective \PROJ-solutions up to~$t'$, we encountered every~\PROJ-row solution for~$t'$ that is required for deriving~$\langle \rho, c\rangle$ via~\PROJ (c.f., Definitions of~$\ipmc$ and of~$\pmc$). %

Assume that~$\type(t)=\rem$. The case is slightly easier as the one
above. We do not need to define a~\PROJ-row solution for~$t'$ for all
subsets~$\varphi$, since we only have to consider subsets~$\varphi$ here,
with~$\Card{\buckets_P(\varphi)}=1$. The remainder works similar.

Similarly, one can show the result for the remaining node with~$\type(t)=\join$, but define \PROJ-row solutions for two preceding child nodes of~$t$.
\end{proof}

We are now in the position to proof our theorem.

\begin{theorem}\label{thm:correctness}
  The algorithm~$\dpa_\PROJ$ is correct. \\
  More precisely, %
  the algorithm~$\dpa_\PROJ((\prog,P),\TTT)$ returns
  $\PROJ$-TTD~$(T,\chi,\pi$) such that $c=\sipmc(\pi(n), \cdot)$
  is the projected answer sets count of~$\prog$ with respect to the set~$P$ of
  projection atoms.
\end{theorem}
\begin{proof}
  By Lemma~\ref{lem:correct} we have soundness for every
  node~$t \in N$ and hence only valid rows as output of table
  algorithm~$\PROJ$ when traversing the tree decomposition in
  post-order up to the root~$n$.
  By Lemma~\ref{lem:local} we know that the projected answer sets count~$c$
  of~$\prog$ is larger than zero if and only if there exists a
  certain~\PROJ-row solution for~$n$.
  This~\PROJ-row solution at node~$n$ is of the
  form~$\langle \{\langle\emptyset, \ldots\rangle\} ,c\rangle$. If
  there is no \PROJ-row solution at node~$n$,
  then~$\tau(n)=\emptyset$ since the table algorithm~$\AlgA$
  is admissible (c.f., Proposition~\ref{prop:sat}). Consequently, we have
  $c=0$. Therefore, $c=\sipmc(\pi(n), \cdot)$ is the
  projected answer sets count of~$\prog$ with respect to~$P$ in both cases.

  Next, we establish completeness by induction starting from the
  root~$n$. Let therefore, $\langle \hat\rho \rangle$ be
  the~\PROJ-solution up to node~$n$, where for each row
  in~$\vec u\in \hat\rho$, $\mathcal{I}(\vec u)$ corresponds to an
  answer set of~$\prog$.  By Definition~\ref{def:loctab}, we know that
  for the root~$n$ we can construct a \PROJ-row solution at~$n$ of the
  form~$\langle \{\langle\emptyset, \ldots\rangle\} ,c\rangle$
  for~$\hat\rho$.  We already established the induction step in
  Lemma~\ref{lem:complete}.
  Hence, we obtain some (corresponding) rows for every
  node~$t$. Finally, we stop at the leaves.

  In consequence, we have shown both soundness and completeness. As a
  result, Theorem~\ref{thm:correctness} is sustains.
\end{proof}

\begin{corollary}\label{cor:correctness}
  The algorithm $\mdpa{\AlgA}$ is correct and outputs for any instance
  of \PASP its projected answer sets count.
\end{corollary}
\begin{proof}
  The result follows immediately, since~$\mdpa{\AlgA}$ consists of two
  dynamic programming passes~$\dpa_\AlgA$, a purging step, and~$\dpa_\PROJ$. For the
  soundness and completeness of~$\dpa_\algo{PRIM}$ we refer to other
  sources~\cite{FichteEtAl17a}. By Proposition~\ref{prop:sat}, the
  ``purging'' step does neither destroy soundness nor completeness
  of~$\dpa_\AlgA$.
\end{proof}

\begin{restateproposition}[prop:phcworks]
\begin{proposition}
  The algorithm $\mdpa{\PRIM}$ is correct and outputs for any instance
  of \PASP its projected answer sets count.
\end{proposition}
\end{restateproposition}
\begin{proof}
This is a direct consequence of Corollary~\ref{cor:correctness}.
\end{proof}

\begin{restateproposition}[prop:disjworks]
\begin{proposition}
  The algorithm $\mdpa{\algo{PRIM}}$ is correct and outputs for any instance
  of \PDASP its projected answer sets count.
\end{proposition}
\end{restateproposition}
\begin{proof}
This is a direct consequence of Corollary~\ref{cor:correctness}.
\end{proof}

\longversion{
\section{Correctness of QBF lower bound}

\begin{definition}[\cite{MarxMitsou16}]
  Given graph~$G=(V,E)$, integers~$i,j,r$ where~$i\leq j$ and total list-capacity function~$f: V \rightarrow \{i,\ldots,j\}$. Then an instance~$(G,r,f)$ of~$(i,j)$-Choosability Deletion asks for the existence of a set of vertices~$V'\subseteq V$ with~$\Card{V'}\leq r$ and~$V_1\eqdef V\setminus V'$, such that for all assignments~$\mathcal{L}: V_1 \rightarrow 2^{\{i,\ldots,j\}}$ with $\Card{\mathcal{L}(v)} = f(v)$ for all~$v\in V_1$, there is a coloring~$c: V_1 \rightarrow \mathcal{L}(v)$ such that for every edge~$(u,v)\in E\setminus (V_1\times V_1): c(u) \neq c(v)$.
\end{definition}

\begin{proposition}[\cite{MarxMitsou16}]
  Instances~$(G,r,f)$ of~$(1,4)$-Choosability Deletion where~$k$
  is the treewidth of~$G$, cannot be
  solved in time~${2^{2^{2^{o(k)}}}}\cdot \CCard{G}^{o(k)}$.
\end{proposition}

\begin{restateproposition}[prop:lampis3]
\begin{proposition}
  QBFs of the form $\forall V_1.\exists V_2.\forall V_3. E$ where~$k$
  is the treewidth of the primal graph of DNF formula~$E$, cannot be
  solved in time~${2^{2^{2^{o(k)}}}}\cdot \CCard{E}^{o(k)}$.
\end{proposition}
\end{restateproposition}
\begin{proof}
We proof the result by reducing from the problem~$(1,4)$-Choosability Deletion,
\end{proof}}
}

\clearpage
\appendix
\section{Correctness of the Reduction of Section~\ref{sec:tdguided}}\label{sec:appendix1}
\begin{restatetheorem}[thm:corr1]
\begin{theorem}[Correctness]%
The reduction from an HCF program~$\Pi$ and a TD~$\mathcal{T}=(T,\chi)$ of~$\mathcal{G}_\Pi$ to SAT formula~$F$ consisting of Formulas~(\ref{red:checkrules})--(\ref{red:checkfirst}) is correct.
Precisely, for each answer set of~$\Pi$ there is a model of~$F$ and vice versa.
\end{theorem}
\end{restatetheorem}
\begin{proof}
``$\Rightarrow$'':
Assume an answer set~$M$ of~$\Pi$. 
Then, there is an ordering~$\varphi$ over~$\at(\Pi)$,
where every atom of~$M$ is proven. %
Next, we construct a model~$I$ of~$F$ as follows.
For each~$x\in\at(\Pi)$, we let (c1) $x\in I$ if $x\in M$.
For each node~$t$ of~$T$, %
and~$x\in\chi(t)$: %
(c2) For every~$l\in \bvali{x}{t}{i}$ with~$i=\hat\varphi_t(x)$, %
we set $l\in I$ if~$l$ is a variable. 
(c3) If there is a rule~$r\in\Pi_t$ proving~$x$, %
we let both~$p^x_{<t}, p^x_{t}\in I$.
Finally, (c4) we set $p^x_{<t} \in I$, if~$p^x_{<{t'}}\in I$ for~$t'\in\children(t)$. %

It remains to show that~$I$ is indeed a model of~$F$.
By (c1), Formulas~(\ref{red:checkrules}) are satisfied by~$I$.
Further, by (c2) of~$I$, the order of~$\varphi$ is preserved
among~$\chi(t)$ for each node~$t$ of~$T$, therefore Formulas~(\ref{red:prop}) are satisfied by~$I$.
Further, by definition of TDs, for each rule~$r\in \Pi$ there is a node~$t$ with~$r\in\Pi_t$.
Consequently, $M$ is proven with ordering~$\varphi$, %
for each~$x\in M$ there is a node~$t$ and a rule~$r\in\Pi_t$ proving~$x$.
Then, Formulas~(\ref{red:checkfirst}) are satisfied by~$I$ due to (c3), 
and Formulas~(\ref{red:check}) are satisfied by~$I$ due to (c4). 
Finally, by connectedness of TDs, also Formulas~(\ref{red:checkremove}) and~(\ref{red:checkremove2}) are satisfied.

``$\Leftarrow$'':
Assume any model~$I$ of~$F$.
Then, we construct an answer set~$M$ of~$\Pi$ as follows.
We set~$a\in M$ if~$a\in I$ for any~$a\in\at(\Pi)$.
We define for each node~$t$ a $t$-local ordering~$\varphi_t$,
where we set $\varphi_t(x)$ to~$j$ for each~$x\in\chi(t)$ such that~$j$ is the decimal number of the binary number for~$x$ in~$t$ given by~$I$.
Concretely, $\varphi_t(x)\eqdef j$, where~$j$ is such $I \models\bvali{x}{t}{j}$.
Then, we define an ordering~$\varphi$ iteratively as follows.
We set~$\varphi(a)\eqdef 0$ for each~$a\in \at(\Pi)$,
where there is no node~$t$ of~$T$ with~$\varphi_t(b) < \varphi_t(a)$.
Then, we set~$\varphi(a)\eqdef 1$ for each~$a\in \at(\Pi)$,
where there is no node~$t$ of~$T$ with~$\varphi_t(b) < \varphi_t(a)$ for some~$b\in\chi(t)$ not already assigned in the previous iteration, and so on. %
In turn, we construct~$\varphi$ iteratively by assigning increasing values to~$\varphi$.
Observe that $\varphi$ is well-defined, i.e., each atom~$a\in\at(\Pi)$ gets a unique value since it cannot be the case for two nodes~$t,t'$ and atoms $x,x'\in\chi(t)\cap\chi(t')$ that~$\varphi_t(x) < \varphi_t(x')$, but~$\varphi_{t'}(x) \geq \varphi_{t'}(x')$.
Indeed, this is prohibited by Formulas~(\ref{red:prop})
and connectedness of~$\mathcal{T}$ ensuring that $\mathcal{T}$ restricted to~$x$ is still 
connected.

It remains to show that~$\varphi$ is an ordering for~$\Pi$ proving~$M$.
Assume towards a contradiction that there is an atom~$a\in M$
that is not proven.
Observe that either~$a$ is in the bag~$\chi(n)$ of the root node~$n$ of~$T$, or it is forgotten below~$n$.
In both cases we require a node~$t$ such that~$p^x_{<t} \notin I$
by Formulas~(\ref{red:checkremove2}) and~(\ref{red:checkremove}), respectively.
Consequently, by connectedness of $\mathcal{T}$ and  Formulas~(\ref{red:check}) there is a node~$t'$,
where~$p_{t'}^x \in I$.
But then, since Formulas~(\ref{red:checkfirst}) are satisfied by~$I$,
there is a rule~$r\in\Pi_{t'}$ proving~$a$ with~$\varphi_{t'}$.
Therefore, since by construction of~$\varphi$, there cannot be a node~$t$ of~$T$ with~$x,x'\in\chi(t)$, $\varphi_t(x) < \varphi_t(x')$, but~$\varphi(x) \geq \varphi(x')$, 
$r$ is proving~$a$ with~$\varphi$. 
\end{proof}%

\section{
Strengthening the Reduction of Section~\ref{sec:tdguided}}\label{sec:bijective}
Next, we strengthen the reduction of Section~\ref{sec:tdguided} to remove %
some duplicate $\mathcal{T}$-local orderings
for a particular answer set of~$\Pi$.

\begin{figure}[t]%
  \centering{
  \shortversion{ %
    \input{graph0/graph_cycle}%
    \includegraphics{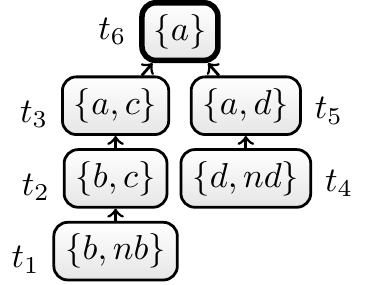}
    \vspace{-.4em}
    \caption{\FIX{Graph~$\mathcal{G}_{\Pi'}$ (left) and a tree decomposition~$\mathcal{T}'$ of~$\mathcal{G}_{\Pi'}$ (right), where program~$\prog'$ is given in Example~\ref{ex:separatedcycles}.}}
  }%
  \label{fig:tdtwo}}%
\end{figure}

In Formulas~(\ref{red:cnt:ineq}), we ensure that if a variable~$x\in\at(\Pi)$ is set to false, then its ordering position is zero.
Formulas~(\ref{red:cnt:ineq2}) make sure that if the position of~$x$ is set to~$i\geq 1$ in node~$t$, there has to be a bag atom~$y$ having position~$i-1$. 
Intuitively, if this is not the case we could shift the position of~$x$ from~$i$ to~$i-1$.
Finally, Formulas~(\ref{red:cnt:checkfirst}) ensure that whenever in a node~$t$ there is a rule~$r\in\Pi_t$ with~$x\in H_r$ and~$x$ has position~$i\geq 1$, 
either %
there is at least one atom~$y\in B_r^+$ having position~$i-1$, or $r$ is not proving~$x$.
{
\begin{flalign}
	\label{red:cnt:ineq}&\neg x \longrightarrow \bigwedge_{1 \leq j \leq \ceil{\log(\Card{\chi(t)})}}\hspace{-.75em}\neg b_{x_t}^j&&{\text{for each } x\in \chi(t)}\raisetag{2.5em}\\
	\label{red:cnt:ineq2}&\bvali{x}{t}{i} \longrightarrow\hspace{-1em}\bigvee_{y \in \chi(t)\setminus\{x\}}\hspace{-1em}\bvali{y}{t}{i-1}&&{\text{for each } x\in \chi(t),%
	1 \leq i < {\Card{\chi(t)}}}\\
	&\hspace{-1em}\bigwedge_{r\in\Pi_t, x\in H_r, 1 \leq i <\Card{\chi(t)}}\hspace{-2em}(\bvali{x}{t}{i}\longrightarrow %
	\label{red:cnt:checkfirst}\hspace{-.5em}\bigvee_{b\in B_r^+}\hspace{-.25em}\neg b \vee {(b \not\prec_t x)} \vee\hspace{-2.5em}\notag\\ 
	&%
	\hspace{5em}\bigvee_{a\in B_r^- \cup (H_r \setminus \{x\})}\hspace{-1.25em}a\, \vee\bigvee_{y\in B_r^+}\hspace{-.25em}\bvali{y}{t}{i-1})&&{\text{for each } x\in \chi(t)}\raisetag{2.5em} %
\end{flalign}%
}

In general, we do not expect to get rid of
all redundant $\mathcal{T}$-local orderings for an answer set, though.
The reason for this expectation lies in the fact that
the different (chains of) rules required for setting the position for an atom~$a$ that is part of cycles of~$D_\prog$ might be spread across the whole tree decomposition.
\FIX{Therefore, these local orderings might not provide the same information that we get from global orderings~\cite{Janhunen06}, where we have absolute values. Instead, these local orderings are insufficient to conclude absolute positions without further information. This is clarified in the following example.}
\begin{example}\label{ex:separatedcycles}
\FIX{Consider the program~$\Pi'\eqdef \{b \vee nb \leftarrow; c \leftarrow b; a \leftarrow c; d \vee nd  \leftarrow; a \leftarrow d\}$.
Observe that program~$\prog'$ has four answer sets~$\{nb, nd\}$, $\{a,d,nb\}$, $\{a,b,c,nd\}$, as well as~$\{a,b,c,d\}$.
Assume the TD~$\mathcal{T}'=(T',\chi')$ of Figure~\ref{fig:tdtwo}, whose width is~$1$ and equals the treewidth of~$\mathcal{G}_\Pi$.
This TD~$\mathcal{T}'$ is such that~$\Pi_{t_3}=\{a\leftarrow c\}$ and~$\Pi_{t_5}=\{a\leftarrow d\}$. 
The particular issue is that node~$t_3$ only considers atoms~$a,c$ and node~$t_5$ only considers atoms~$a,d$. 
Now, assume answer set~$M=\{a,b,c,d\}$.
Then, given only~$t_3$-local orderings and~$t_5$-local orderings, we cannot conclude a unique, canonical global ordering for~$\{a,c,d\}\subseteq M$.
In particular, one could prove~$a$ with either~$a\leftarrow c$ or with~$a\leftarrow d$ (or both).
From a global perspective the latter rule would be preferred to prove~$a$, since it allows an ordering with a smaller position for~$a$.
This is witnessed by the corresponding ordering~$\varphi\eqdef \{b\mapsto 0, c\mapsto 1, d\mapsto 0, a \mapsto 1\}$ for~$M$.
If instead we use the rule~$a\leftarrow c$ for proving~$a$, this would require ordering~$\varphi'\eqdef \{b\mapsto 0, c\mapsto 1, a\mapsto 2, d \mapsto 0\}$, i.e., $\varphi$ is preferred since~$\varphi(a) < \varphi'(a)$.
However, this information is ``lost'' due to the usage of local orderings, which makes it hard to define canonical orderings. 
Therefore our constructed \SAT formula yields two satisfying assignments for~$M$ in this case, corresponding to proving~$a$ either with~$a\leftarrow d$ or~$a\leftarrow c$.
In general, a TD similar to $\mathcal{T}'$ cannot be avoided.
In particular, one can construct programs, where similar situations have to occur in every TD of smallest width.
}
\end{example}

\FIX{
One can even devise further %
cases, where
without absolute orders it is hard to verify whether 
it is indeed required that an atom precedes an other atom. %
This is still the case, if %
for 
each answer set~$M$ of~$\Pi$, and every~$a\in M$, 
there can be only one rule~$r\in \Pi$ suitable for proving~$a$. %
From now on, we refer to such HCF programs~$\prog$ by \emph{uniquely provable}.
Note that even for uniquely provable programs,
there might be several cycles in its positive dependency graph. 
In fact, the program that will be used for the hardness result of normal \ASP and treewidth in Section~\ref{sec:hardness} is uniquely  provable.
However, even for uniquely provable programs and any TD of~$\mathcal{G}_\Pi$, 
there is in general no bijective correspondence 
between answer sets of~$\Pi$ and models of 
Formulas~(\ref{red:checkrules})--(\ref{red:cnt:checkfirst}). 
Consequently, one could compare different,
\emph{absolute} positions of orderings, cf.,~\cite{Janhunen06}, instead of the 
ordering positions relative to one TD node as presented here, 
which requires to store for each atom in the worst-case numbers up to~$\Card{\at(\Pi)}-1$.
Obviously, this number is then not bounded by the treewidth,
and one cannot encode it without increasing the treewidth in general.
Observe that even if one uses orderings on a component-by-component basis, similar to related work~\cite{Janhunen06}, this issue still persists in general since the whole program could be one large component.}
\section{Correctness of the Reduction of Section~\ref{sec:ext}}\label{sec:appendix}
\noindent In order to discuss correctness, we rely on the following lemma,
which establishes that the reduction actually ensures and preserves the transitive closure over~$\prec$
for each node of any tree decomposition.
\begin{lemma}[Transitive Closure]\label{lem:closure}
Let~$\Pi$ be an HCF program, $\mathcal{T}=(T,\chi)$ be a TD of~$\mathcal{G}_\Pi$, and let~$t$ be any node of~$T$.
Further, let $F_{\leq t}$ consist of all Formulas~(\ref{red2:proptrans}) constructed for node~$t$ and every node below~$t$ in~$T$,
and let $\chi_{\leq t}$ be the union of bag~$\chi(t)$ and all bag contents for every node below~$t$ in~$T$.
Then, for any model~$M$ of~$F_{\leq t}$ the following invariant holds:
	All \emph{transitive consequences} of~$\prec$ in~$F_{\leq t}$ are preserved for~$t$, i.e., we have $M\supseteq\{(x\prec y) \mid \{x,y\}\subseteq\chi(t),\text{ there is a path }x,p_1,\ldots,p_o,y\text{ from }x\text{ to }y\text{ in }\mathcal{G}_\prog\text{ with }\{p_1,\ldots,p_o\}\subseteq\chi_{\leq t},\allowbreak \{(x\prec p_1), (p_1 \prec p_2), \ldots, (p_o\prec y)\}\subseteq M\}$.  %
\end{lemma}
\begin{proof}
The proof proceeds by induction. Thereby we assume the invariant holds for every child node of~$t$ and show that then it also is ensured for~$t$.
For simplicity, we only show the ideas for nice TDs, as for non-nice TDs the following cases just overlap.
Recall that a nice TD of~$\mathcal{G}_\Pi$ of width~$k=\tw{\mathcal{G}_\Pi}$, 
having only~$h=\mathcal{O}(\Card{\at(\Pi)})$ many nodes~\cite{Kloks94a}[Lem.\ 13.1.2] always exists. 
We distinguish the following cases.

Case ($\type(t)=\leaf$): The invariant vacuously holds since~$\chi(t)=\emptyset$.

Case ($\type(t)=\rem$ or~$\type(t)=\join$): In these cases, we do not encounter any new auxiliary variable, nor do we add a new formula (that has not been constructed for a node below~$t$) of the form~(\ref{red2:proptrans}). Therefore, since the invariant holds for the child node of~$t$, it is also valid for~$t$.

Case ($\type(t)=\intr$): Let~$\{a\}=\chi(t)\setminus\chi(t')$ with~$\children(t)=\{t'\}$. Assume towards a contradiction that a transitive consequence of~$F_{\leq t}$ is missing, i.e., we have a path $x,p_1,\ldots,p_o,y$ from $x$ to $y$ in $\mathcal{G}_\prog$ with $\{p_1,\ldots,p_o\}\subseteq\chi_{\leq t}$, but~$(x\prec y)\notin M$.
Then, we have either $x= a$ or~$y= a$ since otherwise the consequence would have been already missing for~$t'$. We continue with the case~$x=a$ since the other case works analogously.
Consequently, $p_1\in\chi(t)$, since otherwise~$\mathcal{T}$ would not be a TD of~$\mathcal{G}_\prog$ due to the fact that the edge~$\{x,p_1\}$ would not occur in any TD node of~$\mathcal{T}$.
Then, by the induction hypothesis, we have~$(p_1\prec y)\in M$. 
As a result, we obtain~$(x\prec y)\in M$ since~$x,p_1,y\in\chi(t)$, which contradicts our assumption.
\end{proof}%
\FIX{\begin{corollary}[Acyclicity]\label{cor:acyclicity}
Let~$\Pi$ be an HCF program, $\mathcal{T}=(T,\chi)$ be a TD of~$\mathcal{G}_\Pi$, and let~$t$ be any node of~$T$.
Then, the formula~$F'_{\leq t}$ consisting of all Formulas~(\ref{red2:proptrans}) and~(\ref{red2:exclusion}) for~$t$ and every node below~$t$,
ensures that there cannot exist a model~$M$ of~$F'_{\leq t}$, where the transitive consequences of~$\prec$ in~$F'_{\leq t}$ as in Lemma~\ref{lem:closure} form a cycle.
\end{corollary}}
\begin{proof}
The proof is a direct consequence of Lemma~\ref{lem:closure} and the observation that if there is a path~$x,p_1,\ldots,p_o,x$
from~$x$ to~$x$ in~$\mathcal{G}_\Pi$ with~$(x\prec p_1), (p_1 \prec p_2), \ldots, (p_o \prec x)\in M$, one can reverse any~$(y\prec y')\in M$.
More concretely, for any~$(y\prec y')\in M$ with~$y,y'\in \chi(t)$, by Formulas~(\ref{red2:proptrans}), one can derive~$(y'\prec y)\in M$, which is prohibited
by Formulas~(\ref{red2:exclusion}).
Consequently, cyclic, transitive consequences over variables in~$\chi_{\leq t}$ cannot occur in any model of~$F'_{\leq t}$.
The claim follows since~$\chi_{\leq n}=\at(\prog)$ for root~$n=\rootOf(T)$.
\end{proof}

\begin{theorem}[Correctness and Bijectivity]\label{thm:bijective}
The reduction from a uniquely provable program~$\Pi$ and a TD~$\mathcal{T}=(T,\chi)$ of~$\mathcal{G}_\Pi$ to \SAT formula~$F'$ consisting of Formulas~(\ref{red2:checkrules}), (\ref{red2:checkremove})--(\ref{red2:check}), and (\ref{red2:prove})--(\ref{red2:check2}) is correct.
Concretely, for each answer set of~$\Pi$ there is \emph{exactly one} model of~$F'$ and vice versa.
\end{theorem}

\begin{proof}
``$\Rightarrow$'':
Assume an answer set~$M$ of~$\Pi$. 
Then, there is a minimal\footnote{\FIX{An ordering~$\varphi$ over~$\at(\prog)$ is \emph{minimal (for~$M$)} if there is no atom~$a\in\at(\prog)$ such that when decreasing~$\varphi(a)$, interpretation~$M$ can still be proved with the resulting (modified) ordering.}} ordering~$\varphi$ over~$\at(\Pi)$,
where every atom of~$M$ is proven. %
Next, we construct a model~$I$ of~$F'$ as follows.
For each~$x\in\at(\Pi)$, we let (c1) $x\in I$ if $x\in M$.
For each node~$t$ of~$T$, %
and~$x,y\in\chi(t)$: %
(c2) We set~$(x\prec y)\in I$ if and only if~$\varphi(x)<\varphi(y)$. %
(c3) If there is a rule~$r\in\Pi_t$ proving~$x$ using~$\varphi$, %
we let~$p^x_{t,r},p^x_{<t}, p^x_{t}\in I$ as well as~$p^{y\prec x}_{t},p^{y\prec x}_{<t}\in I$ for~$y\in B_r^+$.
Further (c4) if there is~$z\in\chi(t)$ with~$\Card{\{x,y,z\}}=3$ and~$(x\prec y), (y\prec z)\in I$, we set~$p^{x\prec z}_{t},p^{x\prec z}_{<{t}} \in I$.
Finally, (c5) for~$t'\in\children(t)$, we set $p^x_{<t} \in I$ if~$p^x_{<{t'}}\in I$, as well as~$p^{y\prec x}_{<t} \in I$ if~$p^{y\prec x}_{<{t'}}\in I$. %

Observe that~$I$ is indeed a model of~$F'$.
By (c1), Formulas~(\ref{red2:checkrules}) are satisfied by~$I$.
Further, by (c2) of~$I$, the order of~$\varphi$ is preserved
among~$\chi(t)$ for each node~$t$ of~$T$, therefore Formulas~(\ref{red2:proptrans}) and~(\ref{red2:exclusion}) are satisfied by~$I$.
Then, due to (c3), Formulas~(\ref{red2:prove}) and Formulas~(\ref{red2:propsmaller}) are satisfied and so are Formulas~(\ref{red2:checkfirst}).
Further, due to both (c3) and (c4), Formulas~(\ref{red2:checkfirst2}) are satisfied.
Further, by definition of TDs, for each rule~$r\in \Pi$ there is a node~$t$ with~$r\in\Pi_t$.
Consequently, $M$ is proven with ordering~$\varphi$, %
for each~$x\in M$ there is a node~$t$ and a rule~$r\in\Pi_t$ proving~$x$.
Then, Formulas~(\ref{red2:check}) and~(\ref{red2:check2}) satisfied by~$I$ due to (c5). 
Finally, by connectedness of TDs, also Formulas~(\ref{red2:checkremove}) and (\ref{red2:checkremove2}) as well as~(\ref{red2:checkremove3}) and (\ref{red2:checkremove5}) are satisfied.

It remains to show that there cannot be a model~$I'\neq I$ of~$F'$ with~$I'\cap \at(\Pi)=M$.
Assume towards a contradiction that such a model~$I'$ of~$F'$ indeed exists.
Then, since~$\prog$ is uniquely provable, for each atom~$x\in\at(\prog)$, there is only one rule~$r$
suitable for proving~$x$.
Observe that therefore~$I'$ agrees with~$I$ on variables of the form~$p^x_{t,r}$ by Formulas~(\ref{red2:prove}).
Consequently, $I'$ agrees with~$I$ on provability variables of the form~$p^x_t, p^x_{<t}$ by Formulas~(\ref{red2:checkfirst})
and~(\ref{red2:checkremove})--(\ref{red2:check}).
Analogously~$I'$ also sets provability variables of the form~$p^{x\prec y}_t, p^{x\prec y}_{<t}$ as~$I$
by Formulas~(\ref{red2:checkfirst2}) and~(\ref{red2:checkremove3})--(\ref{red2:check2}).
Then, since relation~$\prec$ for~$I'$ is transitively closed by Lemma~\ref{lem:closure}
and due to acyclicity of~$\prec$ by Corollary~\ref{cor:acyclicity}, we have that~$I'$ precisely gives rise to ordering~$\varphi$,
which contradicts the assumption.

``$\Leftarrow$'':
Assume any model~$I$ of~$F'$.
Then, we construct an answer set~$M$ of~$\Pi$ as follows.
We set~$a\in M$ if~$a\in I$ for any~$a\in\at(\Pi)$.
Then, we define an ordering~$\varphi$ iteratively as follows.
We set~$\varphi(a)\eqdef 0$ for each~$a\in \at(\Pi)$,
where there is no atom~$b\in\at(\Pi)$ with~$(b\prec a)\in I$. %
Then, we set~$\varphi(a)\eqdef 1$ for each~$a\in \at(\Pi)$,
where there is no atom~$b\in\at(\Pi)$ with~$(b\prec a)\in I$ such that~$b$ has not  been already assigned in the previous notation.
In turn, we construct~$\varphi$ iteratively by assigning increasing values to~$\varphi$.
Observe that $\varphi$ is well-defined, i.e., each atom~$a\in\at(\Pi)$ gets a unique value, since~$\prec$ is transitively closed by Lemma~\ref{lem:closure}
and by Corollary~\ref{cor:acyclicity} there cannot be a cycle over relation~$\prec$ for~$I$.

Obviously, $M$ is a model of~$\prog$ by Formulas~(\ref{red2:checkrules}).
It remains to show that~$\varphi$ is an ordering for~$\Pi$ proving~$M$.
Assume towards a contradiction that there is an atom~$a\in M$
that is not proven.
Observe that either~$a$ is in the bag~$\chi(n)$ of the root node~$n$ of~$T$, or it is forgotten below~$n$.
In both cases we require a node~$t$ such that~$p^x_{<t} \notin I$
by Formulas~(\ref{red:checkremove2}) and~(\ref{red:checkremove}), respectively.
Consequently, by connectedness of $\mathcal{T}$ and Formulas~(\ref{red:check}) there is a node~$t'$,
where~$p_{t'}^x \in I$.
But then, since Formulas~(\ref{red2:checkfirst2}) are satisfied by~$I$,
there is a rule~$r\in\Pi_{t'}$ proving~$a$ with~$\varphi$.
\end{proof}%

\FIX{Note that the bijectivity result above can be %
generalized.
Let therefore~$\Pi$ be an HCF program and~$M$ be an answer set of~$\Pi$. 
Then, we define the \emph{dependency graph~$D_\prog^M$ for~$M$} by~$D_\prog^M\eqdef (M, E)$ where we have an edge~$(x,y)\in E$ if and only
if there is a rule~$r\in\prog$ suitable for proving~$y$ with~$x\in B_r^+$.
Observe that Theorem~\ref{thm:bijective} also holds for HCF programs~$\Pi$, as long as for each answer set~$M$ of~$\Pi$
the graph~$D_\prog^M$ is acyclic. %
We call such programs~$\prog$ \emph{deterministically provable} and it is easy to see that every uniquely provable program is also deterministically provable.
While for such deterministically provable programs there might be several rules suitable for proving an atom~$a$ of such an answer set~$M$,
a model for~$M$ of the \SAT formula constructed above is still unique as it greedily aims to prove~$a$ with every applicable rule, cf., Formulas~(\ref{red2:propsmaller}).
\begin{corollary}\label{thm:bijective2}
The reduction from a deterministically provable program~$\Pi$ and a TD~$\mathcal{T}=(T,\chi)$ of~$\mathcal{G}_\Pi$ to \SAT formula~$F'$ consisting of Formulas~(\ref{red2:checkrules}), (\ref{red2:checkremove})--(\ref{red2:check}), and (\ref{red2:prove})--(\ref{red2:check2}) is correct.
Concretely, for each answer set of~$\Pi$ there is \emph{exactly one} model of~$F'$ and vice versa.
\end{corollary}

\begin{proof}
``$\Rightarrow$'':
Assume an answer set~$M$ of~$\Pi$. 
Then, there is a greedy-minimal\footnote{\FIX{A \emph{greedy-minimal} ordering~$\varphi$ over~$\at(\prog)$ for~$M$ is an ordering that serves in proving every atom~$a\in M$ with every rule~$r\in\prog$ suitable for proving~$a$, which is not the case anymore when decreasing~$\varphi(b)$ for some~$b\in\at(\prog)$.}} ordering~$\varphi$ over~$\at(\Pi)$,
where every atom of~$M$ is proven. %
Next, we construct a model~$I$ of~$F'$ as in Theorem~\ref{thm:bijective}.
Observe that~$I$ is a model of~$F'$ by the same arguments as in Theorem~\ref{thm:bijective}.

It remains to show that there cannot be a model~$I'\neq I$ of~$F'$ with~$I'\cap \at(\Pi)=M$.
Assume towards a contradiction that such a model~$I'$ of~$F'$ exists.
Then, since~$\prog$ is deterministically provable, i.e., there is no cycle in~$D_\prog^M$ for each atom~$x\in M$, every rule suitable for proving~$x$ proves~$x$, cf., Formulas~(\ref{red2:propsmaller}). %
Observe that therefore~$I'$ agrees with~$I$ on variables of the form~$p^x_{t,r}$ by Formulas~(\ref{red2:prove}).
Consequently, $I'$ agrees with~$I$ on provability variables of the form~$p^x_t, p^x_{<t}$ by Formulas~(\ref{red2:checkfirst})
and~(\ref{red2:checkremove})--(\ref{red2:check}).
Analogously~$I'$ also sets provability variables of the form~$p^{x\prec y}_t, p^{x\prec y}_{<t}$ as~$I$
by Formulas~(\ref{red2:checkfirst2}) and~(\ref{red2:checkremove3})--(\ref{red2:check2}).
Then, since relation~$\prec$ for~$I'$ is transitively closed by Lemma~\ref{lem:closure}
and due to acyclicity of~$\prec$ by Corollary~\ref{cor:acyclicity}, we have that~$I'$ precisely gives rise to ordering~$\varphi$,
which contradicts the assumption.

``$\Leftarrow$'': The other direction holds by the same argument as in  Theorem~\ref{thm:bijective}.
\end{proof}%

\begin{example}
Recall program~$\prog$ from Example~\ref{ex:running1} and answer sets~$M_1=\{b,c,d\}$, $M_2=\{b,e\}$, $M_3=\{a,c,d\}$ as well as~$M_4=\{a,d,e\}$ of~$\prog$. 
Observe that the graph~$D_\prog^M$ is acyclic for any answer set~$M\in\{M_1,M_3,M_4\}$, whereas~$D_\prog^{M_2}$ contains a cycle.
Consequently, the formula~$F'$ constructed by the reduction above is guaranteed to have exactly one model for each of the answer sets~$M_1, M_3, M_4$, which is a priori not guaranteed for~$M_2$. However, in this case it is easy to see that also for~$M_2$ formula~$F'$ has only one corresponding model since, although~$D_\prog^{M_2}$ is not acyclic, there is no rule suitable for proving~$e$ without having~$b$ first.
\end{example}}

\section{Correctness of the Reduction of Section~\ref{sec:hardness}}\label{sec:corrhard}

\noindent For proving correctness of Theorem~\ref{thm:corr}, we rely on the following key lemma.

\begin{lemma}[$\leq 1$ Outgoing Edge]\label{lem:degree}
Let us consider any instance~$I=(G,P)$ of the \problemFont{Disjoint Paths Problem},
and any answer set~$M$ of~$R(I, \mathcal{T})$ using any pair-connected TD~$\mathcal{T}$ of~$(G,P)$.
Then, there cannot be two edges of the form $e_{u,v},\allowbreak e_{u,w}\in M$.
\end{lemma}
\begin{proof}
Assume towards a contradiction that there are three different vertices~$u,v,\allowbreak w\in V$ with~$e_{u,v}, e_{u,w}\in M$.
Then, by Rules~(\ref{red:localdegree}) there cannot be a node~$t$ with~$(u,v),(u,w)\in E_t^{\text{re}}$.
However, by the definition of TDs, there are nodes~$t', t''$ with~$(u,v)\in E_{t'}^{\text{re}}$ and~$(u,w)\in E_{t''}^{\text{re}}$.
By connectedness of TDs, $u$ appears in each bag of any node of the path~$X$ between~$t'$ and~$t''$.
Then, either~$t'$ is an ancestor of~$t''$ (or vice versa, symmetrical) or there is a common ancestor~$t$.
In the former case, $f^u_{t''}$ is justified by Rules~(\ref{red:setf}) and so is~$f^u_{\hat t}$ on each node~$\hat t$ of~$X$
by Rules~(\ref{red:propf}) and therefore ultimately Rules~(\ref{red:tddegree}) fail due to~$f^u_{t'}, e_{u,w}\in M$.
In the latter case, $f^u_{t''}, f^u_{t'}$ is justified by Rules~(\ref{red:setf}) and so is~$f^u_{\hat t}$ on each node~$\hat t$ of~$X$
by Rules~(\ref{red:propf}). Then, %
Rules~(\ref{red:prohibitf}) fail due to~$f^u_{t'}, f^u_{t''}\in M$.
\end{proof}

\begin{restatetheorem}[thm:corr]
\begin{theorem}[Correctness]%
Reduction~$R$ as proposed in this section is correct.
Let us consider an instance~$I=(G,P)$ of the \problemFont{Disjoint Paths Problem},
and a pair-connected TD~$\mathcal{T}=(T,\chi)$ of~$G$.
Then, $I$ has a solution if and only if the program~$R(I,\mathcal{T})$ admits an answer set.
\end{theorem}
\end{restatetheorem}
\begin{proof}
``$\Rightarrow$'': Assume any positive instance~$I$ of~\problemFont{Disjoint Paths Problem}.
Then, there are disjoint paths $P_1, \ldots, P_i, \ldots P_{\Card{P}}$ 
from~$s_1$ to~$d_1$, \ldots, $s_i$ to~$d_i$, \ldots, $s_{\Card{P}}$ to~$d_{\Card{P}}$ for each pair~$(s_i,d_i)\in P$.
Assuming further pair-connected TD~$\mathcal{T}$ of~$I$, we construct in the following
an answer set~$M$ of~$\Pi=R(I,\mathcal{T})$.
To this end, we collect reachable atoms~$A\eqdef \{u \mid u\text{ appears in some }P_i, 1\leq i \leq \Card{P}\}$
and used edges~$U\eqdef\{(u,v) \mid v\text{ appears immediately after }u\text{ in some }P_i, 1\leq i \leq \Card{P}\}$.
Then, we construct answer set candidate~$M\eqdef %
\{r_u \mid u\in A\}\cup\{e_{u,v}\mid (u,v)\in U\} \cup\{ne_{u,v}\mid (u,v)\in E\setminus U\}\cup \{f^u_t \mid (u,v) \in U \cap E^{\text{re}}_t\} \cup \{f^u_{t} \mid (u,v) \in U\cap E^{\text{re}}_{t'}, u\in\chi(t), t'\text{ is a descendant of }t\text{ in }T\}$.
It remains to show that~$M$ is an answer set of~$\Pi$.
Observe that~$M$ satisfies all the rules of~$\Pi_{\mathcal{R}}$.
In particular, by construction, we have reachability~$r_v$ for every vertex~$v$ of every pair in~$P$, 
and the partition in used edges~$e_{u,v}$ and unused edges~$ne_{u,v}$ is ensured.
Further, $\Pi_{\mathcal{L}}$ is satisfied, as, again by construction, for each vertex~$v$ of every pair in~$P$, we have~$r_v\in M$.
Finally, $\Pi_{\mathcal{C}}$ is satisfied as by construction~$f^u_t\in M$ iff $e_{u,v}\in M\cap E^{\text{re}}_t$ or
$e_{u,v}\in M\cap E^{\text{re}}_{t'}$ for any descendant node~$t'$ of~$t$ with~$u\in\chi(t)$.
It is easy to see that~$M$ is indeed a $\subseteq$-smallest model of the reduct~$\Pi^M$,
since, atoms for used and unused edges form a partition of~$E$.

``$\Leftarrow$'': Assume any answer set~$M$ of~$\Pi$.
First, we observe that we can only build paths from sources towards destinations,
as sources have only outgoing edges and destinations allow only incoming edges.
Further, by construction, vertices can only have one used, outgoing edge, cf., Lemma~\ref{lem:degree}.
Consequently, if a vertex had more than one used, incoming edge, one cannot match at least one pair of~$P$ (by combinatorial pigeon hole principle).
Hence, in an answer set~$M$ of $\Pi$, there is at most one incoming edge per vertex.
By construction of~$\Pi$, in order to reach each~$d_i$ with~$(s_i,d_i)\in_i\sigma$, 
$s_i$ cannot reach some~$d_{j'}$ with~$j'< i$.
Towards a contradiction assume otherwise, i.e., $s_i$ reaches~$d_{j'}$.
But then, by construction of the reduction, we also have a reachable path from~$d_{j'}$ to~$s_i$, consisting of~$d_{j'}, d_{j'+1}, \ldots, d_{i-1}, s_i$.
Since every vertex has at most one incoming edge, $d_{j'}$ cannot have any other justification for being reachable, nor does any source on this path.
Hence, this forms a cycle %
such that no atom of the cycle is proven, which cannot be present in an answer set.
Therefore, $s_i$ only reaches~$d_i$, since otherwise there would be at least one vertex~$s_j$ required to reach~$s_{i'}$ with~$(s_{i'}, d_{i'})\in_{i'}\sigma$, $i'<j$.
Consequently, we construct a witnessing path~$P_i$ for each pair~$(s,d)\in_i \sigma$
as follows: $P_i\eqdef s, p_1, \ldots, p_m, d$ where~$\{e_{s, p_1}, e_{p_1, p_2}, \ldots, e_{p_{m-1}, p_m}, e_{p_m, d}\}\subseteq M$.
Thus, $P_i$ starts with~$s$, follows used edges in~$M$ and reaches~$d$.
\end{proof}

\section{\FIX{Properties and Consequences of Section~\ref{sec:hardness}}}\label{sec:cons}

\FIX{
The resulting program of the reduction consisting of Rules~(\ref{red:edgeguess1})--(\ref{red:localdegree}) is not unary.
However, only Rules~(\ref{red:paircycles}) as well as~(\ref{red:prohibitf})--(\ref{red:localdegree}) are not unary.
Still, Rules~(\ref{red:tddegree}) and~(\ref{red:localdegree}) can be turned unary by replacing the occurrence of~$e_{u,v}$ in these two rules by~$\neg ne_{u,v}$. Further, Rules~(\ref{red:prohibitf}) can be replaced by the following rules, which use an additional auxiliary atom~``$\text{bad}$''.
}\FIX{
\vspace{-.3em}
\begin{flalign}
	& \text{bad} \leftarrow f^u_{t}, f^u_{t'}&&\text{for each }u\in\chi(t')\cap\chi(t''), t'\neq t''\\
	& \leftarrow \text{bad}
\end{flalign}%
\vspace{-1em}
}

\FIX{On the other hand, for Rules~(\ref{red:paircycles}) %
the resulting (positive) cycles of the dependency graph are required for the whole construction, cf., Figure~\ref{fig:cycles}.
More precisely, it is indeed essential for the whole construction that reachability  of a source~$s_i$ requires both reachability of the preceding source~$s_{i-1}$ and destination~$d_{i-1}$. %
Otherwise we cannot prevent a source from reaching a preceding destination via cyclic reachability without provability and still linearly preserve the treewidth. %
Consequently, Rules~(\ref{red:paircycles}) are \emph{not unary} and we expect that this is crucial.
Nevertheless, it was shown that non-unary programs are more expressive than unary programs~\cite{Janhunen06}. Still, we are convinced that exploiting cyclic, unproven reachability %
such that the treewidth is not increased more than linearly, actually requires the usage of non-unary rules.
}
\begin{example}
\FIX{Consider again Figure~\ref{fig:cycles}, depicting the positive dependency graph $D_{R_{\mathcal{L}}}$ of Rules~(\ref{red:paircycles}), as well as Example~\ref{ex:cycle}.
More concretely, consider the same situation of Example~\ref{ex:cycle}, where a source~$s_i$ reaches some destination~$d_j$ with~$j<i$, which causes a cycle~$C{=}r_{s_i},\ldots, r_{d_j}, r_{s_{j+1}}, \ldots, r_{s_i}$ over reachability atoms.
Then, it is crucial for the construction that Rules~(\ref{red:paircycles}) are not unary.
To be more concrete, for the instantiated rule~$r$ with~$r_{s_{j+1}}\in H_r$, we require that both~$r_{s_j}, r_{d_{j}}\in B_r^+$. 
If instead of~$r$ we constructed two rules~$r_{s_{j+1}}\leftarrow r_{s_j}$ and~$r_{s_{j+1}}\leftarrow r_{d_{j}}$, every atom of the cycle~$C$ could be provable since~$r_{s_{j+1}}$ can already be proven by the former rule. 
Further, also for the instantiated rule~$r'$ of Rules~(\ref{red:paircycles}) with~$r_{s_o}\in H_{r'}$ for every~$j+1<o\leq i$, we require that the body is not unary. 
If instead of such a rule~$r'$, we constructed two rules~$r_{d_o}\leftarrow r_{s_{o-1}}$ and~$r_{d_o}\leftarrow r_{d_{o-1}}$, every atom of the cycle~$C$ could be provable since~$r_{d_o}$ is already proven by the latter rule. 
Since, in particular the result should hold for any such cycle~$C$, we rely on non-unary rules for our reduction to work.
}
\end{example}

\end{document}

\section{Treewidth-aware reduction from \SAT to normal \ASP}

\begin{flalign}
	\label{red3:guess}&\{a_{\varphi_t} \} \longleftarrow &&{\text{if }\chi(t)\neq\emptyset}\\
	\label{red3:checkrules}&\longleftarrow a_{\varphi_t} &&{\text{if }\chi(t)\neq\emptyset,\varphi_t\not\models \Pi_t}\\
	\label{red3:prop}&p_{\chi(t)} \longleftarrow a_{\varphi_t}  &&{\text{if }\chi(t)\neq\emptyset}\\
	\label{red3:checkprop}&\longleftarrow \neg p_{\chi(t')}  && {\text{for each }t'\in\children(t), \chi(t)\neq\chi(t'), \chi(t')\neq\emptyset}\\
	\label{red3:checkprop2}&\longleftarrow \neg p_{\chi(n)}  && {\text{if }n=\rootOf(T), \chi(n)\neq\emptyset}\\
\label{red3:prove}&y_t \longleftarrow a_{\varphi_t}, y_{t_1}, \ldots, y_{t_\ell}, x  &&{\text{for each } (x\prec y) \in \prec_{\varphi_t}, \{t_1,\ldots,t_\ell\}=\children(t)}\\ %
\label{red3:prove2}&y_t \longleftarrow a_{\varphi_t}, y_{t_1}, \ldots, y_{t_\ell}  &&{\text{with } \{t_1,\ldots,t_\ell\}=\children(t),\text{ there is no }x\text{ with }(x\prec y) \in \prec_{\varphi_t}}\\ %
	\label{red3:checkremove}&x\longleftarrow x_{t'}&&{\text{for each }{t'\in\children(t)},x\in\chi(t')\setminus\chi(t)}\\
	\label{red3:checkremove2}&x \longleftarrow x_{n}&&{\text{for each }x\in\chi(n)\text{ with }n=\rootOf(T)}%
\end{flalign}

\end{document}

\section{The Elsevier article class}

\paragraph{Installation} If the document class \emph{elsarticle} is not available on your computer, you can download and install the system package \emph{texlive-publishers} (Linux) or install the \LaTeX\ package \emph{elsarticle} using the package manager of your \TeX\ installation, which is typically \TeX\ Live or Mik\TeX.

\paragraph{Usage} Once the package is properly installed, you can use the document class \emph{elsarticle} to create a manuscript. Please make sure that your manuscript follows the guidelines in the Guide for Authors of the relevant journal. It is not necessary to typeset your manuscript in exactly the same way as an article, unless you are submitting to a camera-ready copy (CRC) journal.

\paragraph{Functionality} The Elsevier article class is based on the standard article class and supports almost all of the functionality of that class. In addition, it features commands and options to format the
\begin{itemize}
\item document style
\item baselineskip
\item front matter
\item keywords and MSC codes
\item theorems, definitions and proofs
\item lables of enumerations
\item citation style and labeling.
\end{itemize}

\section{Front matter}

The author names and affiliations could be formatted in two ways:
\begin{enumerate}[(1)]
\item Group the authors per affiliation.
\item Use footnotes to indicate the affiliations.
\end{enumerate}
See the front matter of this document for examples. You are recommended to conform your choice to the journal you are submitting to.

\section{Bibliography styles}

There are various bibliography styles available. You can select the style of your choice in the preamble of this document. These styles are Elsevier styles based on standard styles like Harvard and Vancouver. Please use Bib\TeX\ to generate your bibliography and include DOIs whenever available.

Here are two sample references: \cite{Feynman1963118,Dirac1953888}.

\section*{References}

\bibliography{mybibfile}

\end{document}

%% file: graph0/depgraph.tex
\begin{tikzpicture}[node distance=7mm,every node/.style={fill,circle,inner sep=2pt}]
\node (e) [label={[text height=1.5ex,yshift=0.0cm,xshift=0.05cm]left:$e$}] {};
\node (a) [right of=e,label={[text height=.85ex,xshift=0.25em]left:$a$}] {};
\node (d) [below of=e, label={[text height=1.5ex,xshift=-.34em,yshift=.52em]right:$d$}] {};
\node (b) [below of=a,label={[text height=1.5ex,yshift=0.09cm,xshift=-0.07cm]right:$b$}] {};
\node (c) [left of=d,label={[text height=1.5ex,yshift=0.09cm,xshift=0.05cm]left:$c$}] {};
\draw[->] (d) to  (c);
\draw[->] (d) to  (e);
\draw[->] (b) to  (d);
\draw[->] (b) to  (e);
\draw[->] (e) to  (b);
\end{tikzpicture}

%% file: graph0/graph.tex
\begin{tikzpicture}[node distance=7mm,every node/.style={fill,circle,inner sep=2pt}]
\node (a) [label={[text height=1.5ex,yshift=0.0cm,xshift=0.05cm]left:$e$}] {};
\node (b) [right of=a,label={[text height=.85ex,xshift=0.25em]left:$a$}] {};
\node (e) [below of=a, label={[text height=1.5ex,xshift=-.34em,yshift=.52em]right:$d$}] {};
\node (d) [below of=b,label={[text height=1.5ex,yshift=0.09cm,xshift=-0.07cm]right:$b$}] {};
\node (c) [left of=e,label={[text height=1.5ex,yshift=0.09cm,xshift=0.05cm]left:$c$}] {};
\draw (a) to (c);
\draw (b) to (d);
\draw (c) to (e);
\draw (d) to (e);
\draw (e) to (a);
\draw (d) to (a);
\end{tikzpicture}

%% file: graph0/graph_cycle.tex
\begin{tikzpicture}[node distance=7mm,every node/.style={fill,circle,inner sep=2pt}]
\node (a) [label={[text height=1.5ex,yshift=0.0cm,xshift=-0.05cm]right:$d$}] {};
%
\node (d) [below of=a,label={[text height=1.5ex,yshift=0.09cm,xshift=-0.07cm]right:$a$}] {};
\node (e) [left of=d, label={[text height=1.5ex,xshift=-.34em,yshift=.42em]right:$c$}] {};
\node (c) [left of=e,label={[text height=1.5ex,yshift=0.09cm,xshift=0.05cm]left:$b$}] {};
\node (nb) [above of=c,label={[text height=1.5ex,yshift=0.09cm,xshift=0.05cm]left:$nb$}] {};
\node (nd) [above of=e,label={[text height=1.5ex,yshift=0.09cm,xshift=0.1cm]left:$nd$}] {};
\draw (c) to (e);
\draw (d) to (e);
\draw (nb) to (c);
\draw (nd) to (a);
\draw (d) to (a);
\end{tikzpicture}